\documentclass[sigconf]{acmart}

\pdfoutput=1

\usepackage{booktabs} % For formal tables

\usepackage{enumitem}
\usepackage{multirow}
\usepackage{hhline}
\usepackage{pifont}
\usepackage{amsmath}
\usepackage{bm}
\usepackage{array}
\usepackage{algorithm}
\usepackage{algpseudocode}
\usepackage{xcolor}
\usepackage{wasysym}
\usepackage{graphicx}
\usepackage{balance}
\usepackage{subfig}

\algnewcommand{\LeftComment}[1]{\State \(\triangleright\) #1}

\newtheorem{theorem}{Theorem}

\newcommand{\sij}{S_{ij}}
\newcommand{\zij}{Z_{ij}}
\newcommand{\bmx}{\bm{x}}
\newcommand{\bme}{{\bm e}}
\newcommand{\bmv}{{\bm v}}
\newcommand{\bmz}{{\bm z}}
\newcommand{\bmw}{{\bm w}}
\newcommand{\bmh}{{\bm h}}
\newcommand{\bmy}{{\bm y}}

\newcommand{\ev}{eigenvector}
\newcommand{\pev}{pseudo-eigenvector}
\newcommand{\xarrow}[2]{x_{#1} \rightarrow x_{#2}}

\newcommand{\tZ}{\tilde{Z}}
\newcommand{\algo}{CAST}
\newcommand{\synone}{\textsc{Syn1}}
\newcommand{\syntwo}{\textsc{Syn2}}

\newcommand{\mnist}{\emph{MNIST0127}}
\newcommand{\yale}{\emph{Yale\_5\textsc{class}}}
\newcommand{\coil}{\emph{COIL20}}
\newcommand{\isolet}{\emph{isolet\_5\textsc{class}}}

\newcommand{\glass}{\emph{glass}}

\long\def\comment#1{}

\AtBeginDocument{%
  \providecommand\BibTeX{{%
    \normalfont B\kern-0.5em{\scshape i\kern-0.25em b}\kern-0.8em\TeX}}}

\copyrightyear{2020} 
\acmYear{2020} 
\setcopyright{acmcopyright}\acmConference[KDD '20]{Proceedings of the 26th ACM SIGKDD Conference on Knowledge Discovery and Data Mining}{August 23--27, 2020}{Virtual Event, CA, USA}
\acmBooktitle{Proceedings of the 26th ACM SIGKDD Conference on Knowledge Discovery and Data Mining (KDD '20), August 23--27, 2020, Virtual Event, CA, USA}
\acmPrice{15.00}
\acmDOI{10.1145/3394486.3403086}
\acmISBN{978-1-4503-7998-4/20/08} 
    
\begin{document}
\fancyhead{}

\title{\algo: A Correlation-based Adaptive Spectral Clustering Algorithm on Multi-scale Data}

%\numberofauthors{1} 
%  in this sample file, there are a *total*

\author{Xiang Li$^{\mathsection}$, Ben Kao$^{\mathsection}$, Caihua Shan$^{\mathsection}$, Dawei Yin$^{\dagger}$, Martin Ester$^{\ddagger}$}
\affiliation{
  \institution{\textsuperscript{$\mathsection$}The University of Hong Kong, Pokfulam Road, Hong Kong \\ \textsuperscript{$\dagger$}JD.com, Beijing, China \\ \textsuperscript{$\ddagger$}Simon Fraser University, Burnaby, BC, Canada}
  %\streetaddress{\textsuperscript{$\dagger$}\{xli2, kao, sqluo\}@cs.hku.hk \hspace{2mm} \textsuperscript{$\ddagger$}\{ester\}@sfu.ca}
  \city{\textsuperscript{$\mathsection$}\{xli2, kao, chshan\}@cs.hku.hk \hspace{2mm} \textsuperscript{$\dagger$}yindawei@acm.org \hspace{2mm} \textsuperscript{$\ddagger$}ester@sfu.ca} 
  %\state{\textsuperscript{$\ddagger$} Simon Fraser University, Burnaby, BC,Canada} 
}
%\email{{$\dagger$}\{xli2, kao, sqluo\}@cs.hku.hk \hspace{2mm} {$\ddagger$}\{ester\}@sfu.ca}

%\author{Xiang Li}
%\affiliation{%
%  \institution{Institute for Clarity in Documentation}
%  \streetaddress{P.O. Box 1212}
%\email{webmaster@marysville-ohio.com}
%
%\author{Lars Th{\o}rv{\"a}ld}
%\authornote{This author is the
%  one who did all the really hard work.}
%\affiliation{%
%  \institution{The Th{\o}rv{\"a}ld Group}
%  \streetaddress{1 Th{\o}rv{\"a}ld Circle}
%\email{larst@affiliation.org}
%\affiliation{%
%  \institution{Inria Paris-Rocquencourt}
%\affiliation{%
% \institution{Rajiv Gandhi University}
%\affiliation{%
%  \institution{Tsinghua University}
%  \streetaddress{30 Shuangqing Rd}
%  \country{China}
%\affiliation{%
%  \institution{Palmer Research Laboratories}
%  \streetaddress{8600 Datapoint Drive}
%\affiliation{\institution{The Th{\o}rv{\"a}ld Group}}
%\email{jsmith@affiliation.org}
%
%\author{Julius P.~Kumquat}
%\affiliation{\institution{The Kumquat Consortium}}
%\email{jpkumquat@consortium.net}

\renewcommand{\shortauthors}{X. Li et al.}

\begin{abstract}
We study the problem of applying spectral clustering to cluster multi-scale data,
which is data whose clusters are of various sizes and densities.
Traditional spectral clustering techniques discover clusters by processing a similarity matrix that reflects the proximity of objects.
For multi-scale data,  distance-based similarity is not effective because objects of a sparse cluster could be far apart while those of a dense cluster
have to be sufficiently close.
Following~\cite{li2018rosc}, we solve the problem of spectral clustering on multi-scale data by integrating the concept of 
objects' ``reachability similarity'' with a given distance-based similarity to derive an objects' coefficient matrix. 
We propose the algorithm \algo\ that applies {\it trace Lasso} to regularize the coefficient matrix. 
We prove that the resulting coefficient matrix has the ``grouping effect'' and that it exhibits ``sparsity''. 
We show that these two characteristics imply very effective spectral clustering. 
We evaluate \algo\ and 10 other clustering methods on a wide range of datasets w.r.t. various measures.
Experimental results show that \algo\ provides excellent performance and is highly robust across test cases of multi-scale data.
\end{abstract}

%\begin{CCSXML}
%<ccs2012>
%  <concept_id>10010520.10010553.10010562</concept_id>
%  <concept_desc>Computer systems organization~Embedded systems</concept_desc>
%  <concept_significance>500</concept_significance>
%  <concept_id>10010520.10010575.10010755</concept_id>
%  <concept_desc>Computer systems organization~Redundancy</concept_desc>
%  <concept_significance>300</concept_significance>
%  <concept_id>10010520.10010553.10010554</concept_id>
%  <concept_desc>Computer systems organization~Robotics</concept_desc>
%  <concept_significance>100</concept_significance>
%  <concept_id>10003033.10003083.10003095</concept_id>
%  <concept_desc>Networks~Network reliability</concept_desc>
%  <concept_significance>100</concept_significance>
%</ccs2012>  
%\end{CCSXML}
%
%\ccsdesc[500]{Computer systems organization~Embedded systems}

\maketitle

\section{Introduction}
\label{sec:intro}
Cluster analysis is a fundamental task in machine learning and data mining,
which seeks to group similar objects into same clusters and separate dissimilar objects into different clusters.
Spectral clustering,
which transforms clustering into a graph partitioning problem,
has been shown to be effective in image segmentation~\cite{stella2003multiclass},
text mining~\cite{dhillon2001co} and information network analysis~\cite{li2019spectral}.
These are fundamental tasks that are at the cores of many applications and services, such as text/media information retrieval systems, 
recommender systems, and viral marketing. 

Given a set of objects $\mathcal{X} = \{x_1, x_2 ,... x_n\}$
and a similarity matrix $S$ such that each entry $S_{ij}$ represents the affinity between objects $x_i$ and $x_j$,
standard spectral clustering methods first construct a graph $G = (\mathcal{X}, S)$,
where $\mathcal{X}$ denotes the set of vertices 
and $\sij$ gives the weight of the edge that connects $x_i$ and $x_j$.
Then, the 
%(normalized) 
graph Laplacian $L$ of $G$ is computed and eigen-decomposition is performed on 
% the $n \times n$ 
matrix $L$
to derive the $k$ smallest eigenvectors $\{\bme_1, \bme_2,..., \bme_k\}$\footnote{We say that an eigenvector 
$\bme_i$ is smaller than another eigenvector $\bme_j$ if $\bme_i$'s eigenvalue is smaller than 
that of $\bme_j$'s.},
where $k$ is the desired number of clusters and
$\bm{e}_i$ is the $i$-th smallest eigenvector.
These eigenvectors form a $k\times n$ matrix,
whose $j$-th column is taken as the feature vector of object $x_j$.
(Essentially, objects are mapped into low-dimensional embeddings using the eigenvectors.) 
Finally, a post-processing step, e.g., $k$-means, is applied on the objects with their feature vectors to return clusters.
Figure~\ref{figure:flow_graph} illustrates the general pipeline of spectral clustering. 
%The general pipeline of standard spectral clustering methods [xxx] is summarized in Figure~\ref{figure:flow_graph}(a).

\begin{figure}
    \centering
        \includegraphics[width = \linewidth]{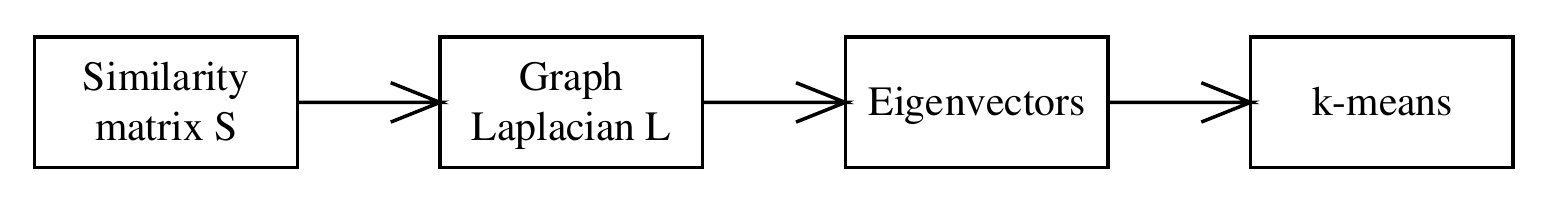}
        \caption{Spectral clustering pipeline}
        \vspace{-5mm}
        \label{figure:flow_graph}
\end{figure}

Spectral clustering aims to optimize certain criterion that measures the 
quality of graph partitions.
For example, the NCuts~\cite{shi2000normalized} method minimizes the \emph{normalized cut} between clusters,
which measures the weights of inter-cluster edges. 
Conventionally, objects' affinity is given by some distance-based similarity. 
For multi-scale data, which consists of object clusters of different sizes and densities, distance-based similarity
is often ineffective in capturing the correlations between objects~\cite{zelnik2004self,nadler2006fundamental}.
This leads to poor performance of spectral methods. 
For example,
Fig.~\ref{figure:example}(a) shows a dense rectangular cluster located on top of a very sparse strip-shaped cluster.
Objects at different ends of the strip-shaped cluster are far apart and hence their distance-based similarity is small.
Fig.~\ref{figure:example}(b) shows the clustering given by NCuts, from which we see that the strip-shaped cluster is 
incorrectly segmented. 
%Due to the long distance between objects in two ends of the strip cluster,
%similarities between these objects are small,
%which leads to the partitioning of the strip cluster into two segments in the clustering given by NCuts, as is shown in Fig.~\ref{figure:syn1}(b).
%In this paper, we study spectral clustering on multi-scale data, with the goal of putting forward a robust spectral clustering method.

In~\cite{li2018rosc}, the ROSC algorithm was proposed to address the multi-scale data issue in spectral clustering. 
%a number of methods have been proposed [xxx].
%Recently, 
The idea is to rectify a given distance-based similarity matrix $S$ by deriving a coefficient matrix $Z$ that can better
express the correlation among objects. 
Intuitively, each entry $Z_{ij}$ in $Z$ represents how well an object $x_i$ characterizes another object $x_j$, and two objects
are considered highly correlated (and thus should be put into the same cluster) if they give similar characterization to other objects. 
The coefficient matrix $Z$ is constructed based on the similarity matrix $S$ as well as a transitive $K$-nearest-neighbor (TKNN) graph.
Specifically, two 
%To capture the high correlation between distant objects in an elongated cluster,
%$Z$ is regularized by 
%a transitive $K$-nearest neighbor (TKNN) graph,
objects $x_i$ and $x_j$ are connected in the TKNN graph if there exists an object sequence $<x_i,...x_j>$
such that adjacent objects in the sequence are $K$-nearest-neighbors of each other.
For example, objects that are located at far ends of the strip-shaped cluster (Fig.~\ref{figure:example}) are connected by a chain of K-NN
relations. 
An important property that was proven in~\cite{li2018rosc} is that the matrix $Z$ has 
the \emph{grouping effect}~\cite{lu2012robust,li2018rosc}, which states that
if two objects are similar in terms of both $S$ and TKNN graph connectivity, 
their corresponding coefficient vectors in $Z$ are also similar.
Based on $Z$, ROSC constructs a new correlation matrix $\tZ$. 
The grouping effect of $Z$ ensures that 
highly correlated objects are grouped together by applying spectral clustering on $\tZ$.

Besides expressing the correlation between objects of the same cluster, another important factor for
correct clustering is to suppress the correlation between objects of different clusters. 
ROSC, however, focuses on enhancing the former by deriving a coefficient matrix $Z$ that amplifies intra-cluster correlation;
it does not promote the latter. 
Our objective is to study methods that deal with both factors.
Specifically, our proposed algorithm \algo\ regularizes matrix $Z$ so that it has grouping effect and it exhibits inter-cluster {\it sparsity}. 
By sparsity, we refer to the desired property that entries in the matrix that correspond to inter-cluster object pairs should be 0 or very small, 
hence the matrix is sparse. 

%
%%and thus explains the superior performance of ROSC.
%However,
%the grouping effect only enhances correlations between objects in the same cluster.
%For objects from different clusters,
%there should be very few connections between them,
%%they should have very few connections,
%which is referred to as \emph{sparsity}.
%ROSC fails to
%%lacks an effective mechanism to
%enforce the sparsity between clusters, 
%which may adversely affect its performance.
%The objective of this paper is to study the spectral clustering issue on multi-scale data.
%In particular,
%given an affinity matrix $S$,
%we aim to construct a new one that exhibits
%grouping effect for highly correlated objects and sparsity for uncorrelated objects.
%The matrix can thus enhance the performance of spectral clustering on multi-scale data.
%%then be fed into the general pipeline of standard spectral clustering.
%

\begin{figure}
    \centering
        \includegraphics[width = \linewidth]{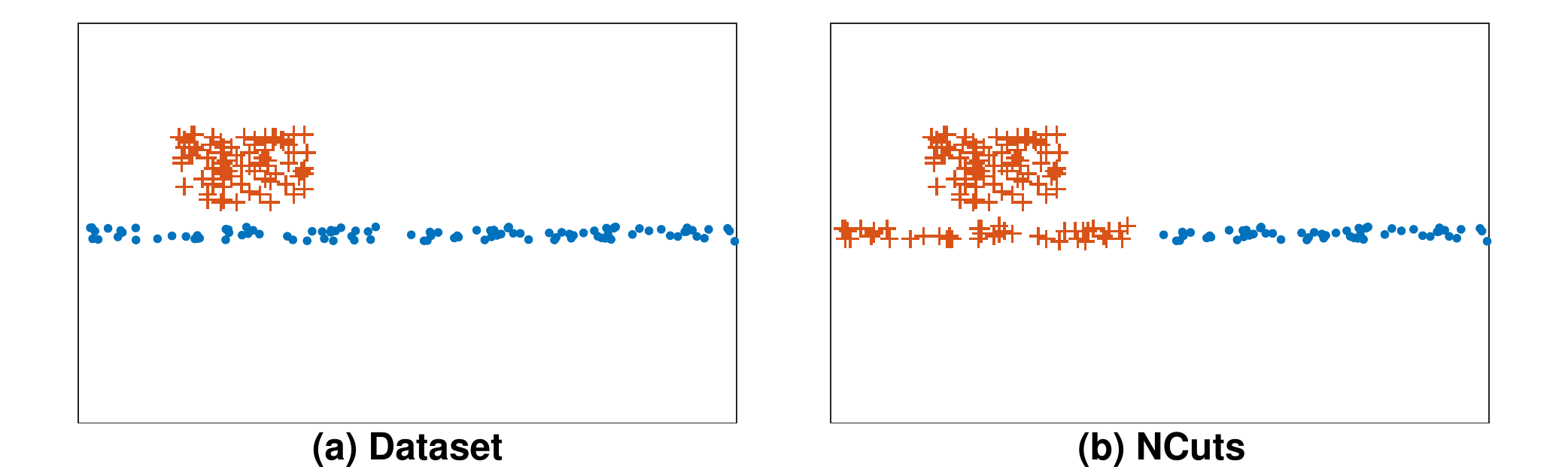}
        \caption{(a) A multi-scale dataset, (b) clustering by NCuts}
        \vspace{-5mm}
        \label{figure:example}
\end{figure}

One common approach to enforce sparsity is to apply $\ell_1$ regularization on a solution matrix 
(i.e., by including the $\ell_1$-norm as a penalty term in an optimization problem).
While using the $\ell_1$-norm helps sparsify inter-cluster correlation, 
it weakens intra-cluster correlation and hence 
it goes against establishing grouping effect.
In contrast, ROSC
uses the Frobenius norm
%\footnote{The Frobenius norm of a matrix can be decomposed into a set of $\ell_2$-norm of vectors.}
to regularize the coefficient matrix $Z$,
which is equivalent to regularizing each column vector of $Z$ by the $\ell_2$-norm.
%adding a $\ell_2$-norm regularization on each column vector of $Z$.
The Frobenius norm has been shown to enhance grouping effect but not sparsity.
%We can also say that the $\ell_2$-norm enhances the grouping effect.
%The Frobenius norm is similar to the $\ell_2$-norm, 
%which has been shown to enhance grouping effect but not sparsity.
%In spectral clustering,
Generally,
a desired coefficient matrix should possess both
grouping effect for objects in the same cluster
and sparsity for objects from different clusters.
To construct such a matrix,
we introduce trace Lasso~\cite{grave2011trace},
which is a regularizer that falls in between the $\ell_1$-norm and  the $\ell_2$-norm.
The trace Lasso is adaptive depending on the correlation between objects.
Given a set of objects $\mathcal{X}$,
let $X$ denote the feature matrix of objects,
%We normalize the feature vector $\bmx_i$ of an object $x_i$ by $\bmx_i^T\bmx_i = 1$.
whose $j$-th column is the feature vector $\bmx_j$ of an object $x_j$.
(We assume the vectors are normalized, i.e., $\bmx_j^T\bmx_j = 1$ for all $j$.)
If objects are highly correlated,
i.e., $X^TX = \textbf{1}\textbf{1}^T$ ($\textbf{1}$ is the all-one vector), 
the trace Lasso is equivalent to the $\ell_2$-norm;
If objects are independent,
i.e., $X^TX = I$ ($I$ is the identity matrix),
the trace Lasso will behave like the $\ell_1$-norm.

%Details of trace Lasso will be introduced in Sec.~\ref{subsection:tracelasso}.
%The correlation adaptivity of trace Lasso enables it to 
%automatically balance the effect of $\ell_1$-norm and $\ell_2$-norm.
%We thus apply it to
%construct a similarity matrix that can be
%used in spectral clustering.
%%that contains both grouping effect within clusters and sparsity between clusters.

In this paper we study spectral clustering over multi-scale data. We propose the
\textbf{C}orrelation-based \textbf{A}daptive \textbf{S}pectral clustering method  using \textbf{T}race lasso, or \algo.
We discuss how \algo\ takes advantage of the trace Lasso to achieve robust spectral clustering.  
%
%Similar to ROSC,
%we construct a coefficient matrix $Z$ that characterizes objects' correlations.
%%with the TKNN graph.
%%to capture the correlations between objects.
%%Moreover, 
%We further regularize $Z$ by $\ell_1$-norm and trace Lasso,
%respectively,
%to reduce connections between clusters.
%In particular, 
%we visually compare the matrices constructed by different methods
%and show the advantage of trace Lasso over other regularizers.
%%trace Lasso can be used to construct a matrix that exhibits strong intra-cluster relations and weak inter-cluster relations.
%%The derived matrix $Z$ in CASC also has the grouping effect.
We summarize 
our main contributions as follows.

\noindent{\small$\bullet$}
We study the problem of applying spectral clustering on multi-scale data.
We propose the \algo\ algorithm, which uses trace Lasso to construct and regularize a coefficient matrix $Z$.
A correlation matrix that exhibits grouping effect and inter-cluster sparsity is subsequently derived for effective and robust spectral clustering.

\noindent{\small$\bullet$}
We mathematically prove that the derived matrix by \algo\ has grouping effect. 
This ensures high intra-cluster object correlation. 
%correlations between objects in the same cluster.
%It thus enhances the performance of spectral clustering when applied to multi-scale data.

\noindent{\small$\bullet$}
We conduct extensive experiments to show the effectiveness of \algo. 
We compare \algo\ with 10 other methods w.r.t. various clustering quality measures over a wide range of datasets.
Our results show that \algo\ consistently provides very good performance over the range of datasets. It is thus a very robust algorithm especially in handling multi-scale data.
%Specifically,
%we visually compare matrices derived by different methods and 
%illustrate the advantage of trace Lasso over other regularizers.

The rest of the paper is organized as follows.
Section~\ref{sec:related}
introduces related works.
In Section~\ref{sec:algorithm}
we describe the ROSC algorithm, give formal definitions of some important concepts based on which our algorithm is designed, and then
present \algo.
%and formally give definitions on some related concepts.
%We then 
%present the CASC model
%and optimization techniques.
Section~\ref{sec:exp}
presents experimental results. Finally, Section~\ref{sec:conclusion} concludes the paper.

\comment{
Cluster analysis is a core technique in data mining and machine learning.
Given a set of objects, the general idea of clustering is to group objects that are {\it similar} into the same clusters
and to separate dissimilar objects into different clusters.  
Among existing techniques, {\it spectral clustering} has been shown to be very effective, particularly in the
fields of image segmentation~\cite{shi2000normalized,stella2003multiclass} and
 text mining~\cite{dhillon2001co}.
 %In these contexts,  clusters of objects form relatively dense regions that are separated by low-density regions. 
 
Based on spectral graph theory,
spectral clustering methods transform the clustering problem into a graph partitioning problem.
Given a set of  $n$ objects $\mathcal{X} = \{x_1, \ldots, x_n\}$,
and a similarity matrix $S$, such that the matrix entry $\sij$ captures the affinity of objects $x_i$ and $x_j$,
spectral clustering first constructs a weighted graph $G=(\mathcal{X},S)$, where $\mathcal{X}$ gives the set of vertices and 
$\sij$ gives the weight of the edge connecting $x_i$ to $x_j$.
The graph $G$ is then partitioned with the objective of optimizing 
a criterion that
measures the quality of the partitioning such as the \emph{normalized cut}~\cite{shi2000normalized}.

Figure~\ref{figure:flow_graph}(a) illustrates the key steps of basic spectral clustering\footnote{There are a number 
of variants of spectral clustering methods. Our description here is based on the
NJW method~\cite{ng2001spectral}.
More details will be presented in Section~\ref{sec:related}.}.
Given a similarity matrix $S$, we first compute a normalized Laplacian matrix  $L$ from $S$. 
Then, we apply eigen-decomposition on $L$ to obtain the $k$ smallest eigenvectors\footnote{We say that an eigenvector 
$\bme_i$ is smaller than another eigenvector $\bme_j$ if $\bme_i$'s eigenvalue is smaller than 
that of $\bme_j$'s.},
$\bme_1, \ldots, \bme_k$.
Let $M$ be an $k \times n$ matrix whose $i$-th row is $\bme_i$.
We take the $j$-th column of $M$ as the feature vector of object $x_j$ and perform $k$-means clustering 
on the objects. 
In a nutshell, spectral clustering methods
map objects into low dimensional embeddings
using the $k$ smallest eigenvectors.
%They have been shown to have interesting properties.
%For example, it is shown that for the special case of $k=2$, 
%the partitioning produced provides a guaranteed approximation 
%to the optimal cuts~\cite{chung1997spectral}.
%\ben{(I don't understand this statement:) Also, It has been experimentally shown that more eigenvectors will be better~\cite{alpert1995spectral,malik2001contour}}.

\begin{figure}
    \centering
        \includegraphics[width = 1.09\linewidth]{flow_graph3}
        \caption{The key steps of (a) basic spectral clustering; (b) with local scaling and PI; (c) ROSC}
        \label{figure:flow_graph}
\end{figure}

\comment{
For example, for a 2-cluster partition $V=V_1 \cup (V \backslash V_1)$, 
the \emph{normalized cut} is defined as
\begin{equation}
\label{eq:ncut}
Ncut(V_1, V\backslash V_1) = \sum_{i\in V_1,j\in V\backslash V_1} S_{ij}[\frac{1}{a(V_1)} + \frac{1}{a(V\backslash V_1)}]
\end{equation}
where $a(V_1) = \sum_{i\in V_1, j\in V}S_{ij}$.
Minimizing the normalized cut was proposed in~\cite{shi2000normalized} and
the optimization problem is NP-hard.
It has been shown that cuts based on the eigenvector corresponding to the second largest eigenvalue of the normalized graph Laplacian
$D^{-1}(D-S)$
give a guaranteed approximation to the optimal cuts~\cite{chung1997spectral,shi2000normalized},
where $D$ is a diagonal matrix with $D_{ii} = \sum_jS_{ij}$.
It is easy to further extend the normalized cut criterion to the case of $k$ clusters,
and cuts based on the largest $k$ eigenvectors will give a guaranteed approximation.
From this viewpoint, spectral clustering is closed related with the spectral analysis to a matrix.
However, since the object cluster membership is hard,
mapping from these eigenvectors to the discrete cluster membership is required.
}

\comment{
Clustering fundamentally serves as an analysis tool in data mining and machine learning,
and spectral clustering is one important type.
Some basic clustering methods, such as $k$-means and EM clustering~\cite{dempster1977maximum}, 
explicitly or implicitly pre-assume that data should fit specific distributions.
%which assumes data follows Gaussian distribution.
Obviously, when data is complex, these methods tend to fail.
%especially those which contain non-Gaussian distributed clusters.
In contrast, spectral clustering may work well on such data.
Instead of making assumptions on the data distribution,
it translates clustering into a graph partition problem,
which aims to achieve strong intra-cluster relations and weak inter-cluster relations between objects.
Objects are clustered by a spectral analysis on the normalized graph Laplacian,
which has been theoretically proved~\cite{chung1997spectral,spielmat1996spectral}.
}

\comment{
\begin{figure}
    \centering
        \includegraphics[width = \linewidth]{figure/syn1_intro.eps}
        \caption{An toy example}
        \label{figure:SYN1_intro}
\end{figure}
}

%%%%%%%%%% Ben: I masked the results of SYN1 because we are thinking of removing it.
\begin{figure}
    \centering
        \includegraphics[width = \linewidth]{figure/example_intro.eps}
        \caption{(a) A multi-scale dataset, (b) clustering by NJW}
        \label{figure:syn1}
\end{figure}

Despite the successes of spectral clustering, previous works have
pointed out that spectral methods
can be adversely affected by the presence of 
\emph{multi-scale data}~\cite{zelnik2004self,nadler2006fundamental},
which is defined as data whose object clusters are of various sizes and densities.
As an illustrative example, 
Figure~\ref{figure:syn1}(a) shows a dataset of three clusters: 
two dense rectangular clusters on top of a narrow sparse stripe cluster.
Figure~\ref{figure:syn1}(b) shows the result of applying the standard spectral clustering method NJW
to the dataset. We see that the stripe cluster is segmented into three parts, two of which are
incorrectly merged with the rectangular clusters. 
The objective of this paper is to address the multi-scale data issue in spectral clustering. 
In particular, we review existing methods for handling multi-scale data, provide insight into how
the issue can be addressed, and put forward our algorithm called ROSC which outperforms existing 
methods in clustering multi-scale data.

There are two general approaches to address the multi-scale data problem: one on scaling the similarity matrix 
$S$ and another on applying the power iteration technique to obtain pseudo-eigenvectors that contain rich cluster separation
information.

Recall that spectral clustering methods model data objects as a graph and
perform clustering by graph partitioning.
The similarity matrix $S$ should therefore capture objects' local neighborhood information. 
A common choice of such a similarity function is the
Gaussian kernel:
$\sij = \exp \left(-\frac{||\bmx_i-\bmx_j||^2}{2\sigma^2}\right)$,
where $\bm x$ (boldface) denotes a feature vector of an object $x$, 
%$||\cdot||$ denotes the standard Euclidean distance
and $\sigma$ is a global scaling parameter.
A major difficulty in using the Gaussian function to cluster multi-scale data lies in the choice of
$\sigma$. 
If $\sigma$ is set small, then $\sij$ will become small.  Objects in a sparse cluster (which are 
relatively distant among themselves compared with objects in a dense cluster) will likely be
judged as dissimilar leading to partitioning of the cluster. On the other hand, if $\sigma$ is set large,
$\sij$ will be large. Hence,
nearby dense clusters could be judged similar to each other and are inadvertently merged. 

To tackle this problem, the ZP method~\cite{zelnik2004self} applies {\it local scaling} and modifies the Gaussian similarity
to $\sij = \exp\left (-\frac{||\bmx_i-\bmx_j||^2}{\sigma_i\sigma_j}\right)$.
Here, $\sigma_i$ (and likewise for $\sigma_j$) is a local scaling parameter for object $x_i$.
It is defined as the distance between $x_i$ and its $l$-th nearest neighbor for some 
empirically determined value $l$.
If $x_i$ is located in a sparse cluster, then $\sigma_i$ will be large.
This boosts the similarity of $x_i$ and its neighboring objects and thus avoids the splitting of 
sparse clusters. 
Also, if $x_i$ is located in a dense cluster, $\sigma_i$ will be small.
Objects will have to be very close to be considered neighbors.
This avoids the merging of nearby dense clusters. 

Previous studies~\cite{alpert1995spectral,ye2016fuse} have suggested that employing more eigenvectors beyond the $k$ smallest ones
can help capture more cluster separation information and thus improve spectral clustering in handling multi-scale data.
A {\it power iteration} (PI) method~\cite{saad2011numerical} was put forward as an efficient method for computing the
dominant eigenvector of a matrix. 
It is observed in~\cite{lin2010power} that one can
{\it truncate} the iteration process to obtain an 
intermediate \pev\ $\bmv_t$.
It is shown that $\bmv_t$ represents a weighted linear combination of all the eigenvectors and is thus
a very effective replacement of the $k$ smallest eigenvectors typically used in a standard spectral
clustering process. 
Figure ~\ref{figure:flow_graph}(b) shows how the local-scaling method (green box) and the power iteration method (yellow box)
are integrated into the basic spectral clustering method. 

In this paper we take a different approach to tackle the multi-scale data problem. 
The core idea is to construct an $n \times n$  coefficient matrix $Z$ such that the entry $\zij$\footnote{Given a matrix $M$, 
%unless otherwise specified, 
we use a pair of subscripts to specify an entry of $M$.} reflects how well 
an object $x_i$  characterizes another object $x_j$. 
Our objective is to derive such a $Z$ with  ``grouping effect'': 
if two objects $x_i$ and $x_j$ are highly correlated (and thus should be put into the same cluster), 
then their corresponding coefficient vectors 
$\bmz_i$ and $\bmz_j$ given in $Z$ are similar.
The interesting question we address is how to find such a $Z$.

The main feature of our algorithm ROSC is illustrated by the red box shown in Figure ~\ref{figure:flow_graph}(c). 
Instead of using PI to obtain low dimensional embeddings of the objects as input to $k$-means
(yellow box in Figure ~\ref{figure:flow_graph}(b)),
ROSC uses the embeddings to construct the matrix $Z$ (upper path in the red box).
We note that two objects that belong to the same cluster could be located at distant far ends of a cluster,
their high correlation is therefore not expressed properly by a distance-based similarity matrix $S$.
To capture the high correlation between distant objects in the same cluster,
we propose to use a transitive $K$ nearest neighbor (TKNN) graph (lower path in the red box). 
Two objects $x_i$ and $x_j$ are connected by an edge in the TKNN graph
if there is a sequence of objects $<x_i, \ldots, x_j>$ such that adjacent objects in the sequence 
are mutual $K$ nearest neighbors of each other. 
We use the TKNN graph to regularize the matrix $Z$ so that it possesses the desired grouping effect.
The matrix $Z$ is then fed to the pipeline of spectral clustering, taking the role of $S$.

%To tackle the problem,
%researchers mainly focus on two aspects.
%First, construct a more effective similarity matrix.
%The Gaussian kernel is based on the feature distance between objects, 
%however, in the case of clusters with multiple densities,
%the average distance between objects in the sparse cluster is larger than in the dense cluster.
%Therefore, 
%the average similarity between objects in the sparse cluster is smaller than in the dense cluster,
%leading to the sparse cluster being more likely to be partitioned.
%A representative method in this kind is 
%the self-tuning spectral clustering method ZP.
%Instead of using the global scaling parameter, 
%ZP considers the local statistics information for each object
%and calculates $S_{ij} = exp(-\frac{||\bmx_i-\bmx_j||^2}{\sigma_i\sigma_j})$,
%where $\sigma_i,\sigma_j$ denote the local scaling parameters for objects $\bmx_i,\bmx_j$ respectively.
%$\sigma_i$ is set to be the distance between $\bmx_i$ and its $l$-th neighbor,
%where $l$ is set empirically.
%By considering the local density, ZP can increase the similarity between objects in the sparse cluster.
%However, when clusters are of uniform density but different sizes,
%objects will have similar local scales and ZP will fail. 
%Further, the miscalculation on $\sigma_i$ can also lead to the poor clustering performance.

\comment{
Second, employ more eigenvectors.
%The standard spectral clustering methods use only $k$ eigenvectors.
%(or the $k$ largest eigenvectors of the graph similarity matrix).
%However, it is not unlikely that the $k$-th eigenvector corresponds to some particularly salient noise in the data,
%while the $k$+1-th eigenvector contains good cluster indicators and then it will be missed totally~\cite{lin2010power}.
It has been pointed out that when dealing with clusters of different scales, 
standard spectral clustering using only $k$ eigenvectors may fail~\cite{nadler2006fundamental,ye2016fuse}.
Therefore, some works propose to use more eigenvectors to derive more cluster-separation information
based on power iteration, which provides a way to fuse information in all eigenvectors (will be introduced in the next section).
However, more eigenvectors also bring the redundancy and noise problem, 
which further bring negative effects.
%Recently, power-iteration based methods have attracted great attention.
%the full spectral clustering method FUSE was proposed to
%fuse the useful cluster-separation information in all the eigenvectors, but 
%it neglects the noise reduction.
Finally,
we emphasize that methods of both kinds 
are not mutually independent.
Some approaches use more eigenvectors based on a locally scaled similarity matrix.
The difference is that 
the former focuses more on the similarity matrix construction 
while the latter highlights more on the usage of eigenvectors based on a given similarity matrix.
}

\comment{
On the other hand,
more eigenvectors can indeed bring more cluster-separation information, but in the meantime, 
more redundancy and more noise,
which further arouses the redundancy and noise reduction problem.
Existing methods attempt to address these problems, 
but no one can be widely applied.
For example, the self-tuning spectral clustering method ZP~\cite{zelnik2004self},
which introduces local scaling, may fail 
when large clusters have comparable densities with small clusters~\cite{nadler2006fundamental};
the full spectral clustering method FUSE~\cite{ye2016fuse}, 
which adopts \emph{independent component analysis} (ICA) to 
reduce redundancy, neglects noise reduction in their case.
To investigate the robustness of these methods, 
we first conduct a group of experiments on two multi-scale datasets as shown in Fig.~\ref{figure:syn1}(a) and~\ref{figure:syn2}(a) respectively. 
}

\comment{
The first dataset consists of 
three uniformly distributed clusters with 500, 80 and 120 objects respectively.
Since the largest rectangular cluster is of large length,
the two-end objects in the cluster are far away from each other,
i.e., they are less similar. 
Further, the closeness between the two small clusters and the large cluster increases more difficulty in clustering.
The second dataset
is composed of five clusters:
two Gaussian distributed with 100 objects each,
two uniformly distributed with 150 and 200 objects respectively 
and an annular cluster with 100 objects.
The annular cluster is close to the other clusters and it is hard to be identified.
Fig.~\ref{figure:syn1} and~\ref{figure:syn2} also show the clustering results for ZP and FUSE.
We observe that ZP performs well on SYN2 but poorly on SYN1 
while FUSE works better on SYN1 but fails on SYN2.
We further notice that ROSC, our proposed method in this paper,
achieves favorable results on both datasets.
From Fig.~\ref{figure:syn1},
ROSC can identify the three clusters with some errors in the small clusters
while FUSE separates the large cluster into two clusters.
From Fig.~\ref{figure:syn2}, ROSC performs well in identifying the annular cluster
while ZP correctly finds Gaussian distributed and uniformly distributed clusters
with misclassification on some objects in the annular cluster.
Obviously, ROSC is more robust than ZP and FUSE.
}

\comment{
Although existing methods attempt to improve spectral clustering on multi-scale data,
we notice that they are unstable.
The instability originates from the fact that none of them settles the problem from the origin.
The effect of spectral clustering depends on the quality of the similarity matrix.
Suppose a block diagonal matrix is given, in which each block corresponds to one cluster,
spectral clustering can definitely perform well because there exist only intra-cluster similarities but no inter-cluster similarities.
Some approaches aim to construct a locally scaled similarity matrix,
but the difficulty in estimating the local scales may 
result in inaccurate similarities which fail in reflecting the true similarities between objects.
Further, they lack a scheme to rectify these inaccurate similarities. 
Different from all the existing methods,
this paper attempts to improve spectral clustering from the origin of the problem. 
Given a similarity matrix, we aim to derive a new matrix which can better reflect the true similarities between objects.
To achieve such goal,
the new matrix should inherit the accurate similarities in the original matrix
and further rectify the inaccurate ones.
The rectified matrix will lead to more robust spectral clustering, as shown in Fig.~\ref{figure:syn1}(b) and~\ref{figure:syn2}(b).
}

\comment{
While these two bottlenecks have attracted great attention,
most of the proposed methods are dedicated to only one aspect~\cite{yan2009fast,chen2011large,zelnik2004self,li2007noise}.
Recently, Ye et al.~\cite{ye2016fuse} put forward a novel method FUSE 
to address both the two problems simultaneously.
Since eigen-decomposition requires high time complexity, it employs \emph{power iteration},
in which eigenvector calculation is replaced by a small number of matrix-vector multiplications,
to derive a set of pseudo-eigenvectors.
The derived pseudo-eigenvectors are fused by all the eigenvectors, 
which contain all the cluster-separation information.
Then it adopts \emph{Independent Component Analysis} (ICA)~\cite{learned2003ica, bohm2008outlier} to rotate these vectors 
and get statistically independent ones.
These pseudo-eigenvectors are then regarded as feature vectors, 
to which k-means would be applied to derive the final clustering result.
FUSE achieves both effectiveness and efficiency, 
however, there still remain some problems. 
First, the derived pseudo-eigenvectors may be redundant and noise corrupted. 
Despite rotation by ICA reducing redundancy, the noise still remains, which weakens the effectiveness.
Second, a self-adaptive greedy strategy is developed to search for directions towards the optimal solution, however,
it might still be trapped in the local optimum.
Third, the time complexity of FUSE relies on
the maximal rank considered by the low-rank decomposition algorithms for pair-wise mutual information estimation in ICA. 
When the rank is as close as $n$, FUSE will run slowly. 
To this end, spectral clustering deserves further exploration from both perspectives of effectiveness and efficiency.
}

Our main contributions are:

\noindent$\bullet$
We address the multi-scale data problem in spectral clustering
and propose the ROSC algorithm. ROSC uses PI to derive a coefficient matrix $Z$, which is 
regularized by a TKNN graph. The regularized $Z$ replaces the similarity matrix $S$ in the spectral
clustering process.

\noindent{\small$\bullet$}
We mathematically prove that the regularized $Z$ possesses the grouping effect. 
Hence, it is effective in improving clustering results.

\noindent{\small$\bullet$}
We conduct extensive experiments %using synthetic and real datasets 
to evaluate the performance of ROSC
against $9$ other clustering methods. 
Our results show that ROSC performs very well against the competitors. 
In particular, it is very robust in that it consistently performs well over all the datasets tested. 
Also, it outperforms others by wide margins for datasets that are highly multi-scale. 

The rest of the paper is organized as follows.
%In Section~\ref{sec:preliminary} we give more details of spectral clustering and briefly 
%describe the power iteration method.
Section~\ref{sec:related} mentions related works and describes a number of key
spectral clustering algorithms.
Section~\ref{sec:algorithm} presents the ROSC algorithm.
Section~\ref{sec:exp} describes the experiments and presents experimental results.
Finally, Section~\ref{sec:conclusion} concludes the paper.

%\noindent{\small$\bullet$}
%We propose a transitive $K$ nearest neighbor (TKNN) graph.
%In the graph,
%objects in the same cluster but far away in the feature space can be connected
%while objects in different clusters but close to each other can be disconnected.
%
%\noindent{\small$\bullet$}
%We put forward a robust spectral clustering method ROSC on multi-scale data.
%Based on power iteration, ROSC fuses more cluster-separation information in more eigenvectors
%and reduces the redundancy and noise contained.
%%and it integrates the noise reduction with matrix rectification.
%It further applies the TKNN graph to rectify the raw similarity matrix 
%and derives a more effective one,
%leading to a more robust spectral clustering.
%
%\noindent{\small$\bullet$}
%We conduct extensive experiments to
%prove the robustness of ROSC.
%We compare ROSC with state-of-the-art methods on both synthetic and real datasets
%with respect to three clustering measures.
%All the experimental results show that ROSC is indeed a robust spectral clustering method.

\comment{
Recently, in machine learning and computer vision areas, 
%data can be viewed as points drawn from multiple low-dimensional subspaces, with each subspace corresponding to one category or class.
subspace segmentation has been well studied,
which aims to segment (cluster) high-dimensional objects into the low-dimensional subspaces where they are drawn from.
A number of methods have been proposed, such as LRR~\cite{liu2010robust}, LSR~\cite{lu2012robust}, etc.
In these methods, given a set of data vectors $\mathcal{X} = [\bm{x}_1, \bm{x}_2,...,\bm{x}_n]$ (each column is an object) in $\mathbb{R}^d$, 
each object is represented by 
a linear combination of the basis in a dictionary $A = [\bm{a}_1,\bm{a}_2,...\bm{a}_m]$:
\begin{equation}
\nonumber
X = AZ,
\end{equation}
where $Z$ is a coefficient matrix.
Considering the case that data vectors are corrupted by noise, 
a more robust model is thus modified as 
\begin{equation}
\nonumber
X = AZ + E,
\end{equation}
where $E$ is utilized to capture noise.
Subject to the constraint, different optimization objectives on $Z$ (and $E$) are adopted by different methods,
and they desire to derive a low-dimensional mapping $Z$ which can reflect the true subspace structure in the original data.

In this paper, we aim to improve spectral clustering from both effectiveness and efficiency.
Similar to FUSE, we first use power iteration to derive a set of pseudo-eigenvectors 
which inherit all the cluster-separation information in all the eigenvectors.
Then we resort to the subspace segmentation model to purify the pseudo-eigenvectors.

To be continued.
}
}

\section{Related Work}
\label{sec:related}

Spectral clustering is a widely studied topic~\cite{bojchevski2017robust,correa2012locally,li2007noise,chen2018spectral,xiang2008spectral}.
%Given a similarity matrix $S$,
%standard spectral clustering methods
%vary in the way they normalize the graph Laplacian $L = D - S$, where $D$ is the diagonal matrix with $D_{ii} = \sum_j S_{ij}$.
%For example,
%the NCuts~\cite{shi2000normalized} method uses the random-walk-based normalization $D^{-1} (D-S)$
%while the NJW~\cite{ng2001spectral} method employs the symmetric normalization $D^{-\frac{1}{2}}(D-S)D^{-\frac{1}{2}}$.
There are many works that study various aspects of spectral clustering, such as computational efficiency~\cite{chen2011large,cullum2002lanczos,yan2009fast},
clustering performance on data with different characteristics~\cite{huang2009spectral,von2008consistency,zhu2014constructing},
and the theoretic foundations of the method~\cite{lafon2006diffusion,nadler2005diffusion,meila2001random}.
%Moreover, spectral clustering has been widely applied in various fields [xxx].
An introduction to spectral clustering is given in~\cite{von2007tutorial,kannan2004clusterings,ng2001spectral}.

Despite the success of spectral clustering,
previous works~\cite{nadler2006fundamental,ye2016fuse} have pointed out that
spectral methods can be adversely affected when data is multi-scale.
To address the problem,
the \emph{self-tuning spectral clustering} method ZP~\cite{zelnik2004self}
uses \emph{local scaling} to extend a Gaussian kernel based similarity  
$\sij = \exp \left(-\frac{||\bmx_i-\bmx_j||^2}{2\sigma^2}\right)$ to $\sij = \exp\left (-\frac{||\bmx_i-\bmx_j||^2}{\sigma_i\sigma_j}\right)$,
where $\bm x$ (boldface) denotes the feature vector of an object $x$, 
%$||\cdot||$ denotes the standard Euclidean distance
and $\sigma$ is a scaling parameter.
The original formulation uses a global scaling parameter $\sigma$ for every object pair,
which is difficult to set.
When $\sigma$ is set small, $S_{ij}$ is small
and it cannot effectively
capture the high correlation between distant objects in a large sparse cluster.
On the contrary,
when $\sigma$ is set large,
$S_{ij}$ is large. Objects from different but nearby dense clusters 
will  then more likely to be mis-judged as similar.
To address the issue,
ZP introduces a local scaling parameter $\sigma_i$ for each object $x_i$,
which is defined as the distance between $x_i$ and its $l$-th nearest neighbor ($l$ can be empirically set).
For an object $x_i$ in a sparse cluster,
$\sigma_i$ is large. This enlarges the similarity between $x_i$ and other distant objects in the same cluster of $x_i$.
Also, a dense cluster gives a small $\sigma_i$, which effectively
%of an object $x_i$ in a dense cluster
decreases the similarity between $x_i$ and objects from nearby clusters.

%Although standard spectral clustering methods use the $k$ smallest eigenvectors 
Other previous works~\cite{alpert1995spectral} have suggested to use more eigenvectors 
to capture more cluster separation information to 
improve the effectiveness of spectral clustering on multi-scale data.
There are methods that employ pseudo-eigenvectors generated by the power iteration (PI) technique.
%which are weighted linear combination of all the eigenvectors.
%and thus contain more cluster separation information.
The PI method is used to compute
the dominant eigenvector of a matrix.
%For details of the PI method, please see [xxx].
Lin et al.~\cite{lin2010power}  point out that one can truncate the iterative PI process to 
obtain an intermediate pseudo-eigenvector.
The pseudo-eigenvector is a weighted linear combination of all the eigenvectors
and thus contains rich cluster separation information.
%The basic spectral clustering methods use
%the $k$ smallest eigenvectors of the (normalized) graph Laplacian,
%e.g., $D^{-1}(D-S)$,
%which are equivalent to the $k$ largest eigenvectors of $W = D^{-1}S$.
%However, as pointed out in~\cite{lin2010power},
%it is possible that 
%some of 
%these smallest eigenvectors correspond to salient noise in the data,
%while some non-top-k-smallest eigenvectors contain rich cluster separation information.
%Also, there exist some works [xxx] suggesting
%using more eigenvectors to improve the effectiveness of spectral clustering on multi-scale data.
They propose the \emph{Power Iteration Clustering} (PIC) method based on the idea.
%which truncates the power iteration process on $W$ to obtain an intermediate pseudo-eigenvector.
%which is a weighted linear combination of all the eigenvectors.
%where the pseudo-eigenvector is used as a replacement of the $k$ smallest eigenvectors
%and
%fed into the step of $k$-means clustering.
Under PIC, however, each object has only one feature value given by the lone \pev.
When the number of clusters is large,
a single pseudo-eigenvector is not enough to handle the cluster collision problem~\cite{lin2012scalable}. 
The PIC-$k$ method~\cite{lin2012scalable} is subsequently proposed to address this issue.
PIC-$k$ runs PI multiple times to generate multiple pseudo-eigenvectors.
These pseudo-eigenvectors provide more features of objects for more effective clustering.
One issue of the PIC-$k$ method is that the {\pev}s are not orthogonal and thus are redundant.   
To reduce redundancy, \cite{thang2013deflation} proposes
a \emph{Deflation-based Power Iteration Clustering} (DPIC) method
that uses Schur complement to generate orthogonal pseudo-eigenvectors.
Another issue of the PIC-based methods is that
the more dominant eigenvectors are assigned larger weights in the PI iteration. 
Generally, they overshadow other lesser but indispensable {\ev}s.
To address this issue,
the \emph{Diverse Power Iteration Embedding} (DPIE) method~\cite{huang2014diverse} is proposed.
When a new pseudo-eigenvector is generated in DPIE,
information of previously generated pseudo-eigenvectors is removed from the new one.
%DPIE wisely sets threshold to ensure the process to generate the $i$-th pseudo-eigenvector
%stops earlier than the one to generate the ($i$-1)-th pseudo-eigenvector.
Ye et al.~\cite{ye2016fuse} put forward 
a \emph{Full Spectral Clustering} (FUSE) method.
It uses independent component analysis to rotate $p > k$ generated pseudo-eigenvectors 
and make them statistically independent.
The $k$ most informative ones are used for clustering. 
Finally, in~\cite{li2018rosc}, a spectral clustering algorithm for multi-scale data, ROSC, is proposed. 
Details of ROSC will be covered in the next section.

\comment{
Spectral clustering is a well-studied topic.
For an introduction and an analysis of the method, see~\cite{von2007tutorial,ng2001spectral,kannan2004clusterings}.
%Figure~\ref{figure:flow_graph}(a) shows the basic pipeline of spectral clustering. 
There are a number of variants of the basic method, some of which differ in the 
way they normalize the graph Laplacian,
$D-S$, where $D$ is the diagonal matrix with $D_{ii} = \sum_j S_{ij}$.
For example, the \emph{normalized cuts} (NCuts) method~\cite{shi2000normalized} 
employs random-walk-based normalization $D^{-1}(D-S)$
while the \emph{Ng-Jordan-Weiss} (NJW) method~\cite{ng2001spectral}  uses symmetric normalization $D^{-\frac{1}{2}}(D-S)D^{-\frac{1}{2}}$.

There are many previous works studying the various aspects of spectral clustering techniques.
For example, there are studies
that focus on the performance of spectral clustering under different data characteristics~\cite{huang2009spectral, von2008consistency, zhu2014constructing},
on computational efficiency~\cite{cheng2003recursive,
chen2011large,cullum2002lanczos,gittens2013approximate,yan2009fast,dhillon2007weighted},
and on the probabilistic theory of the method~\cite{meila2001random,nadler2005diffusion,lafon2006diffusion}.

\comment{
The first one is the high computation complexity due to eigen-decomposition, 
which would be $O(n^3)$ in general.
A great number of methods have been proposed to solve the problem~\cite{cheng2003recursive,
song2008parallel,cullum2002lanczos,gittens2013approximate,wang2009approximate,dhillon2007weighted}
and next we summarize some major works as follows.

\textbf{[Random sampling based approaches]}
This kind of methods aim to speed up spectral clustering by the sampling technique.
Yan et al.~\cite{yan2009fast} proposed k-means-based approximate spectral clustering (KASP)
in which k-means is first used as a preprocessor to derive $k$ cluster centroids as representatives. 
Then the original spectral clustering method is performed on 
these centroids. Finally, the cluster membership of each object is 
determined by its associated nearest centroid. 
In~\cite{shinnou2008spectral}, a slightly different way is adopted to reduce the data size. It first runs k-means to derive a clustering result
and then pick some objects near to the center of each cluster to form the so-called committee data.
Finally, the original spectral clustering is performed on the data except committee.
%Both the two methods above require an accurate initial clustering result by k-means.
Based on the Nystr\"{o}m method, 
\cite{fowlkes2004spectral} randomly chooses samples to calculate small-size eigenvectors
and then extrapolates eigenvectors to the entire data. 
Although these methods reduce the computation cost, 
they are quite sensitive to the sampling quality.

\textbf{[Singular value decomposition (SVD) based approaches]}
Since eigen-decomposition on the $n\times n$ matrix is computationally costly, methods in this kind generally 
select $m$ representative objects (or construct $m$ virtual objects)
and compute a $n\times m$ similarity matrix,
where $m\ll n$. 
Then instead of performing eigen-decomposition,
truncated SVD is used to derive a set of singular vectors, 
which could approximate the principle components of the subspace spanned by the top-k eigenvectors of the $n\times n$ matrix.
The main difference in these methods is the construction of similarity matrix.
With random projection and sampling technique, \cite{sakai2009fast} first maps the original data space into low dimensions and
randomly choose $m$ objects as examples. 
Later, a similarity matrix between original objects and example objects is approximately calculated by the low dimensional feature vectors.
Chen et al.~\cite{chen2011large} select a set of points as landmarks
and use them to encode all the data points. 
A sparse similarity matrix is then computed by Nadaraya-Watson kernel regression~\cite{hardle1990applied}.
In~\cite{Liu:2013:LSC:2540128.2540342}, a small number of dummy supernodes are generated to connect to the regular nodes
and a bipartite graph is thereby constructed.
Through such transformation, a small size similarity matrix can be computed between regular nodes and supernodes.
}

There are also works that point out the degradation of spectral clustering's effectiveness when
data is multi-scale~\cite{nadler2006fundamental} or noisy~\cite{li2007noise}. 
To address these problems,
a number of methods have been proposed~\cite{zelnik2004self,correa2012locally,li2007noise,bojchevski2017robust}.
%Some approaches~\cite{zelnik2004self,correa2012locally} consider the local density of each object to compute a locally scaled similarity matrix.
For example,
the \emph{self-tuning spectral clustering} method ZP~\cite{zelnik2004self} 
considers the local statistics information of each object and
constructs a locally scaled similarity matrix as we have described in Section~\ref{sec:intro}.
Moreover, ZP estimates the number of clusters by rotating the eigenvectors
to best align them with a canonical coordinate system.
The number of clusters that gives the minimal alignment cost is then selected.

Standard spectral clustering uses the so-called ``most informative" eigenvectors.
Heuristically, the $k$ smallest {\ev}s are usually taken as the most informative ones. 
However, it is pointed out in~\cite{lin2010power} that it is possible
that some of these smallest {\ev}s correspond to some particularly salient noise in the data,
while other non-top-$k$-smallest vectors contain good cluster separation information.
This observation leads to studies of how {\ev}s should be selected in spectral clustering.
For example, Xiang et al.~\cite{xiang2008spectral} 
enhance spectral clustering by measuring the relevance of 
an eigenvector according to how well it can separate data into different clusters, particularly in the
presence of noise.

As we mentioned in Section~\ref{sec:intro}, 
power iteration (PI) can be used to find {\pev}s in spectral clustering.
Given a matrix $W$, PI starts with a random vector $\bm{v}_0\neq \bm{0}$
and iterates:
\begin{equation}
\bm{v}_{t+1} = \frac{W\bm{v}_t}{||W\bm{v}_t||_1}, \;\;\; t \geq 0.
\end{equation} 
Suppose $W$ has eigenvalues $\lambda_1>\lambda_2>...>\lambda_n$
with associated eigenvectors $\bm{e}_1,\bm{e}_2,...,\bm{e}_n$.
We express $\bm{v}_0$ as
\begin{equation}
\label{eq:v0}
\bm{v}_{0} = c_1\bm{e}_1 + c_2\bm{e}_2 + ... + c_n\bm{e}_n,
\end{equation}
for some constants $c_1 \ldots, c_n$.
Let $R = \prod_{i=0}^{t-1} \lVert W\bmv_i \rVert_1$, we have,
\begin{equation}
\begin{split}
\bmv_t & = W^t \bm{v}_0  / R\\
& = (c_1 W^t\bm{e}_1 + c_2 W^t\bm{e}_2 + ... + c_n W^t\bm{e}_n)/R\\
& = (c_1 \lambda_1^t\bm{e}_1 + c_2 \lambda_2^t\bm{e}_2 + ... + c_n \lambda_n^t\bm{e}_n)/R \\
& = \frac{c_1\lambda_1^t}{R}\left[\bm{e}_1 + \frac{c_2}{c_1}\left(\frac{\lambda_2}{\lambda_1}\right)^t\bm{e}_2+...+ \frac{c_n}{c_1}\left(\frac{\lambda_n}{\lambda_1}\right)^t\bm{e}_n\right].
\end{split}
\label{eq:pi}
\end{equation}
$\bmv_{t}$ is thus a linear combination of the {\ev}s.
Moreover, if
$\bmv_0$ has a nonzero component in the direction of $\bm{e}_1$ (i.e., $c_1 \neq 0$),
$\bmv_{t}$ converges to a scaled version of the dominant \ev\ $\bm{e}_1$. 

In the context of spectral clustering, we set $W=D^{-1}S$.
It is easy to see that the $k$ smallest eigenvectors of $D^{-1}(D-S)$ computed in NCuts
are equivalent to the $k$ largest eigenvectors of $W$.
Lin et al.~\cite{lin2010power} propose the \emph{Power Iteration Clustering} (PIC) method.
PIC employs truncated power iteration to obtain an intermediate \pev\ $\bmv_t$,
which is shown to be a weighted linear combination of all the eigenvectors.
The $j$-th component of $\bmv_t$, $\bmv_t[j]$, is taken as the feature of object $x_j$.
Based on these feature values, 
$k$-means clustering is applied to cluster objects. 

\comment{
To handle noise, a warping model was proposed in~\cite{li2007noise} to map objects into 
new space, in which clusters become relatively compact and well separated, 
including a noise cluster composed of noise objects.
Then a similarity matrix is computed in the new space and 
standard spectral mapping is employed to derive low-dimensional representation for objects.
The combination of warping mapping and spectral mapping is used as the final embedding for objects
and k-means will be applied to return clusters.}

PIC uses PI to obtain one single \pev\ from which objects' feature values are extracted. 
A single \pev, however, is generally not enough when the number
of clusters is large due to the {\it cluster collision problem}~\cite{lin2012scalable}. 
In~\cite{lin2012scalable}, the PIC-$k$ method is proposed which 
runs PI multiple times to obtain multiple {\pev}s. These generated {\pev}s provide 
more dimensions of object features that are used in the $k$-means clustering step of spectral clustering. 
These {\pev}s, however, could be similar to each other and are thus redundant. 
In ~\cite{thang2013deflation}, a \emph{Deflation-based Power Iteration Clustering} (DPIC) method is
 proposed. DPIC applies Schur complement to generate multiple orthogonal pseudo-eigenvectors
to address the redundancy issue.
Another issue of the PIC method is that the more dominant eigenvectors 
contribute higher weights in the \pev\ (see Equation~\ref{eq:pi}).
They thus overshadow the other minor but indispensable eigenvectors,
especially in the case of multi-scale and noisy data.
To deal with this problem, 
Huang et al.~\cite{huang2014diverse} put forward a \emph{Diverse Power Iteration Embedding} (DPIE) method.
In DPIE, when a new \pev\ is generated, information of other previously obtained {\pev}s is removed
from
the new one.
Also, certain threshold values are wisely set to prevent minor {\ev}s from being overshadowed. 
Recently, a \emph{Full Spectral Clustering} method FUSE~\cite{ye2016fuse} is proposed.
It first generates $p > k$ {\pev}s and then uses independent component analysis (ICA)
to rotate {\pev}s so that they become pairwise statistically independent.
The $k$ rotated pseudo-eigenvectors with the most cluster-separation information are selected 
for clustering.
%Although ICA is employed, it still does not deal with the noise in their formula.

Our algorithm ROSC differs from these previous works in two aspects.
First, ROSC uses PI to obtain {\pev}s not as input to the $k$-means clustering step, but to 
construct a coefficient matrix $Z$ which expresses how one object is characterized by other
objects.
Second, we propose the TKNN graph, which is derived from a locally-scaled similarity matrix $S$.
The TKNN graph is used to regularize the matrix $Z$ so that $Z$ possesses the {\it grouping effect}.
The matrix $Z$ serves as input to the spectral clustering pipeline in place of the original similarity
matrix $S$. ROSC is thus orthogonal to other existing techniques of spectral clustering.

\comment{
Although the two bottlenecks have attracted great attention,
few methods aim to solve both.
In fact, power iteration based methods can
not only reduce the high time complexity,
but also utilize all the relevant eigenvectors rather than just the top-$k$ eigenvectors,
which further leads to a good performance in
data containing multi-scale clusters or much noise.
Recently, without any assumptions, 
FUSE~\cite{ye2016fuse} was proposed 
which has been proved both effective and efficient even on data containing multi-scale clusters.
Its success arises from the fusion of all the cluster-separation information,
but noise is also encoded. Further, based on ICA, it might be trapped into local optimum
and run slowly. These shortcomings motivate us to further explore spectral clustering.
}
}

\section{Algorithms}
\label{sec:algorithm}
In this section we first outline the ROSC algorithm, which is the basis of our algorithm \algo. 
We define TKNN graph and grouping effect.
Then, we discuss the effect of different regularization methods on the coefficient matrix $Z$, which is constructed to derive object correlation. 
After that, we present \algo\ and prove that the matrix $Z$ \algo\ derives has grouping effect and that it promotes sparsity. 

%In this section, we first revisit the ROSC algorithm
%and give formal definitions of the TKNN graph and grouping effect.
%We then show how $\ell_1$-norm and trace Lasso can be 
%used to construct
%a coefficient matrix $Z$, respectively. 
%In the end, 
%we present our proposed algorithm CASC
%and 
%prove that the derived $Z$ of CASC has grouping effect.

\subsection{ROSC}

%The distance-based similarity measure could not effectively capture correlations between objects when data is multi-scale.
%To solve it,
%Li et al.~\cite{li2018rosc} propose a robust spectral clustering method ROSC,
%which constructs a new similarity matrix
%as a replacement of the original one.
%Given a similarity matrix $S$,
%ROSC runs PI multiple times to generate multiple pseudo-eigenvectors.
%Taking these pseudo-eigenvectors as low dimensional embeddings of objects,
%ROSC constructs a coefficient matrix $Z$,
%which represents how well an object can be characterized by others.
%For distant objects in a large cluster,
%ROSC enhances their correlations by regularizing $Z$ with a TKNN graph.
%The matrix $Z$ is proved to possess grouping effect for highly correlated objects.
%However,
%for uncorrelated objects,
%ROSC could not weaken their connections.
%Our algorithm CASC
%uses trace Lasso to regularize $Z$,
%which derives a matrix that can more effectively capture objects' correlations in the presence of multi-scale data.

The basic idea of ROSC is to construct a coefficient matrix $Z$ from a given similarity matrix $S$ and then perform 
spectral clustering based on $Z$. 
To construct $Z$, ROSC first applies PI multiple times to generate $p$ pseudo-eigenvectors.
Whitening~\cite{kessy2017optimal} is used to reduce the redundancy of these pseudo-eigenvectors.
The  $p$ {\pev}s together form a $p \times n$ matrix $X$.
The $q$-th column of $X$ is taken as the feature vector $\bmx_{q}$ of an object $x_q$.
That is, ROSC takes the {\pev}s as low dimensional embeddings of objects. 
Assuming the linear subspace model~\cite{liu2013robust},
ROSC characterizes each object by others with
\begin{equation}
X = XZ + O,
\end{equation}
where $Z \in \mathbb{R}^{n\times n}$ denotes a coefficient matrix
such that each entry $Z_{ij}$ describes how well an object $x_i$ characterizes another object $x_j$,
and $O \in \mathbb{R}^{p \times n}$ is a matrix that captures the noise in the pseudo-eigenvectors.
The more similar two objects are, the more likely one object can be represented by the other.

To correlate objects that are located at far ends of a cluster,
ROSC regularizes $Z$ by a TKNN graph. 
\begin{definition}
\label{def:trans_relation}
\textbf{(Transitive $K$-nearest-neighbor (TKNN) graph)}
Given a set of objects $\mathcal{X} = \{x_1, x_2,..., x_n\}$,
the TKNN graph $\mathcal{G}_K = (\mathcal{X},\mathcal{E})$
is an undirected graph,
where $\mathcal{X}$ is the set of vertices and $\mathcal{E}$ is the set of edges.
Specifically, the edge ($x_i$, $x_j$) $\in \mathcal{E}$
iff there exists a sequence <$x_i,...x_j$>
such that adjacent objects in the sequence are $K$-nearest-neighbors of each other.
We use a reachability matrix $\mathcal{W}$ to represent the TKNN graph, 
whose ($i$,$j$)-entry $\mathcal{W}_{ij} = 1$ if ($x_i$, $x_j$) $\in \mathcal{E}$; 0 otherwise.
\hfill$\Box$
\end{definition}
ROSC optimizes the objective function
\begin{equation}
\label{eq:rosc_l2}
\min_Z\ \| X - XZ\|_F^2 + \alpha_1 \|Z\|_F^2 + \alpha_2 \| Z-\mathcal{W} \|_F^2,
\end{equation}
where the first term reduces noise $O$, the second term is the Frobenius norm of $Z$
and the third term regularizes $Z$ by the TKNN graph.
It is shown in~\cite{li2018rosc} that 
Eq.~\ref{eq:rosc_l2} has a closed-form solution $Z^*$ that has the following grouping effect.

\begin{definition}
\label{def:group_effect}
\textbf{(Grouping effect)}. 
Given a set of objects $\mathcal{X} = \{x_1, x_2,..., x_n\}$,
let $\bmw_q$ denote the $q$-th column of $\mathcal{W}$.
Assume that objects' features are normalized, i.e., $\bmx_q^T\bmx_q = 1$, $\forall 1 \leq q \leq n $.
Let $\xarrow{i}{j}$ denote the conditions:
(1) $\bmx_i^T \bmx_j \rightarrow 1$ and 
(2) $\lVert \bmw_i - \bmw_j \rVert_2 \rightarrow 0$.
A matrix $Z$ is said to have grouping effect
if
\[
\left( \xarrow{i}{j} \right) \Rightarrow \left( |Z_{ip} - Z_{jp}| \rightarrow 0\; \forall 1 \leq p \leq n \right).
\]
\end{definition}
Grouping effect considers both feature similarity and reachability similarity of objects.
Feature similarity is measured based on the closeness of objects' feature vectors (columns in $X$),
while reachability similarity is evaluated by columns in $\mathcal{W}$ that express the connectivity of objects in the TKNN graph.
Since the optimal solution $Z^*$ has grouping effect, 
highly correlated objects in both similarities will have similar coefficient vectors in $Z^*$.
ROSC will thus group the objects in the same cluster.
Note that $Z^*$ may be asymmetric and contain negative values.
ROSC derives a correlation matrix $\tilde{Z} = (|Z^*| + |(Z^*)^T|)/2$ and uses that as input to standard spectral clustering in place of 
the original similarity matrix $S$.
%Finally,
%$\tilde{Z}$ will be taken as the new similarity matrix and 
%a standard spectral clustering method will be applied on $\tilde{Z}$ to return clusters.

\subsection{Sparsity}
\label{subsec:scm}
ROSC focuses on generating a coefficient matrix $Z$ that has grouping effect so that highly correlated objects have similar vector representations
in $Z$. 
For effective clustering, we also require that objects from different clusters have sparse connections.
%The generated matrix $Z$ in ROSC possesses grouping effect for highly correlated objects.
%However,
%for objects from different clusters,
%they should have sparse connections.
%%their sparse correlations should also be taken into consideration.
To enforce sparsity,
a common approach is to use $\ell_1$ regularization.
We thus modify Eq.~\ref{eq:rosc_l2} as
%We thus replace the term $\fnorm{Z}^2$ by $\lonorm{Z}$ in Eq.~\ref{eq:rosc_l2}. Hence,
\begin{equation}
\begin{split}
\label{eq:rosc_l1}
 \min_Z \quad & \frac{1}{2}\| X - XZ\|_F^2 + \alpha_1 \|Z\|_1 + \frac{\alpha_2}{2} \| Z-\mathcal{W} \|_F^2, \\
 \text{s.t.} \quad & \text{diag}(Z) = {\bm 0}, \\
\end{split}
\end{equation}
where the second term is the $\ell_1$-norm of $Z$.
To avoid self-representation of objects,
we further add the constraint $\text{diag}(Z) = {\bm 0}$,
where $\text{diag}(Z)$ is the main diagonal vector of $Z$.
%$Z_{ii} = 0$, $\forall 1 \leq i \leq n$.
The optimization problem in Eq.~\ref{eq:rosc_l1}
is convex and 
it is equivalent to solving the problem:
\begin{equation}
\begin{split}
\label{eq:rosc_l1_transform}
 \min_{Z,J} \quad & \frac{1}{2}\| X - XJ\|_F^2 + \alpha_1 \|Z\|_1 + \frac{\alpha_2}{2} \| J-\mathcal{W} \|_F^2 \\
 \text{s.t.} \quad & J = Z - \text{Diag}(Z), \\
\end{split}
\end{equation}
where $\text{Diag}(Z)$ returns a diagonal matrix whose main diagonal vector is that of $Z$.
Eq.~\ref{eq:rosc_l1_transform} can be solved by
the inexact Augmented Lagrange Multiplier (ALM) method~\cite{zhang2010recent,lin2010augmented},
which minimizes the following augmented Lagrangian function:
\begin{equation}
\nonumber
\begin{split}
L(J,Z) & = \frac{1}{2}||X-XJ||_F^2 + \alpha_1||Z||_1 + \frac{\alpha_2}{2}||J-\mathcal{W}||_F^2 \\
& +tr(Y^T(J-Z+\text{Diag}(Z)))+\frac{\mu}{2}||J-Z+\text{Diag}(Z)||_F^2,
\end{split}
\end{equation}
where $Y$ is the Lagrangian multiplier and $\mu > 0$ is a penalty parameter.
$L$ can be minimized by alternatively updating one variable with the others fixed. 
To update $J$, 
we set $\frac{\partial L}{\partial J} = 0$ and derive
\begin{equation}
\label{eq:update_J_sparse}
J = (X^TX+\alpha_2 I + \mu I)^{-1} (X^TX+\alpha_2 \mathcal{W} - Y + \mu Z -\mu \cdot \text{Diag}(Z)).
\end{equation}
To update $Z$,
we set $\frac{\partial L}{\partial Z} = 0$ and derive
\begin{equation}
\label{eq:update_Z_sparse}
Z = A - \text{Diag}(A)\quad \text{and} \quad A = \mathcal{T}_{\frac{\alpha_1}{\mu}} \left(\frac{Y}{\mu} + J \right),
\end{equation}
where $\mathcal{T}_\eta(\cdot)$ is the shrinkage-thresholding operator
acting on each entry of a given matrix,  
which is defined as
$\mathcal{T}_\eta(v) = (|v| - \eta)_+sgn(v)$.
The operator $(\cdot)_+$ returns the argument value if it is non-negative; 0 otherwise.
The operator $sgn(\cdot)$ gives the sign of the argument value.
Algorithm~\ref{alg_sparse} in Appendix shows
the algorithm that uses inexact ALM to generate a sparse coefficient matrix $Z$.
By the theory of inexact ALM,
the convergence of Algorithm~\ref{alg_sparse} is guaranteed~\cite{lin2010augmented}.

%\begin{algorithm}
%\begin{small}
%\caption{Solving Eq.~\ref{eq:rosc_l1} by inexact ALM}
%\label{alg_sparse}
%\begin{algorithmic}[1]
%\Require $X$, $\mathcal{W}$, $k$, $\rho$, $\mu_{\max}$, $\epsilon$
%\Ensure $Z$
%\State Initialize $J$, $Z$, $Y$, $\mu$
%\While{$\|J-Z+\text{Diag}(Z)\|_{\infty} > \epsilon$}
%\State Update $J$ by Eq.~\ref{eq:update_J_sparse} with the others fixed
%\State Update $Z$ by Eq.~\ref{eq:update_Z_sparse} with the others fixed
%\State Update the multiplier $Y = Y + \mu (J-Z+ \text{Diag}(Z))$
%\State Update $\mu = \min (\rho \mu, \mu_{\max})$
%\EndWhile
%\State \Return $Z$
%\end{algorithmic}
%\end{small}
%\end{algorithm}

%\subsection{Correlation Adaptive Coefficient Matrix}
\subsection{\algo}
\label{subsection:tracelasso}
Although the coefficient matrix $Z$ derived with $\ell_1$ regularization (Eq.~\ref{eq:rosc_l1})
has sparse entries for uncorrelated objects,
the regularization also weakens the connections between correlated objects,
which goes against the grouping effect.
To construct a matrix with both grouping effect 
for highly correlated objects and sparsity for uncorrelated ones,
we regularize $Z$ by the trace Lasso, 
which is a regularizer that takes object correlation into consideration.
Let $X$ be a feature matrix
whose $q$-th column is the feature vector $\bmx_q$ of an object $x_q$.
We normalize $\bmx_q$ such that $\bm{x}_q^T \bm{x}_q = 1$, $\forall 1 \leq q \leq n$.
Let $\bmz$ be the coefficient vector (a vector in the  coefficient matrix $Z$) that corresponds to an object $x$,
the trace Lasso of $\bmz$ is defined as 
\begin{equation}
\label{eq:trace_lasso}
\Omega (\bmz) = \|X \: \text{Diag}(\bmz)\|_*,
\end{equation}
where $\text{Diag}(\bmz)$
%\footnote{If the argument is a matrix, $\text{Diag}(\cdot)$ returns a diagonal matrix whose main diagonal is that of the matrix.}
is the diagonal matrix with $\bmz$ as its main diagonal.
The integration of $X$ into the norm
distinguishes the trace Lasso from other commonly used norms like the $\ell_1$-norm and the $\ell_2$-norm.
Note that,

%It is adaptive to the object correlation,

\begin{equation}
\label{eq:trace_lasso_decomp}
X\text{Diag}(\bmz) = \sum_{q=1}^n z_q \bmx_q \bme_q^T,
\end{equation}
where $\bm{e}_q$'s are  vectors of canonical basis.
Consider $X^T X$ as an encoding of object correlation. 
Then, 
if objects are uncorrelated, i.e., $X^TX = I$, 
the trace Lasso equals the $\ell_1$-norm
due to the orthogonality of $\bmx_i$ and $\bme_i$:
\begin{equation}
\|X\text{Diag}(\bmz)\|_* = \sum_{q=1}^n \|\bmx_q\|_2|z_q| = \sum_{q=1}^n |z_q| = \|\bmz\|_1.
\end{equation}
On the other hand,
if all the objects are highly correlated  and have the same feature vector $\bmx$, i.e., $X^TX = \bm{1}\bm{1}^T$, the
trace Lasso is equivalent to the $\ell_2$-norm:
\begin{equation}
\|X\text{Diag}(\bmz)\|_* = \|\bmx \bm{z}^T\|_* = \|\bmx\|_2\|\bmz\|_2 = \|\bmz\|_2.
\end{equation}
For other cases, the
trace Lasso falls in between the $\ell_1$- and the $\ell_2$- norms:
\begin{equation}
\|\bmz\|_2 \leq \|X\text{Diag}(\bmz)\|_* \leq \|\bmz\|_1.
\end{equation}

Due to this adaptability, 
%The correlation adaptivity enables trace Lasso to
%automatically balance the effect of $\ell_1$-norm and $\ell_2$-norm on a matrix.
%Therefore,
we apply the trace Lasso to regularize $Z$ in Eq.~\ref{eq:rosc_l2}.
%To avoid the self-representation of an object $x_i$,
%let $\tilde{X}_i$ be a matrix derived from $X$ with the $i$-th column removed.
%So do $\tilde{\bmz}_i$ and $\tilde{\bmw}_i$.
Specifically, given an object $x$, we optimize:
\begin{equation}
\label{eq:casc}
 \min_{\bmz}\ \frac{1}{2}\| \bmx - X\bmz\|_2^2 + \alpha_1 \| X \text{Diag}(\bmz)\|_* + \frac{\alpha_2}{2} \| \bmz - \bmw \|_2^2.
\end{equation}
%where the second term is the trace Lasso regularization on $\bmz$.
The optimization problem is convex and can be solved by the inexact ALM method.
We first transform the problem into:
\begin{equation}
\label{eq:casc_ialm}
\begin{split}
\min \quad & \frac{1}{2}\|\bme\|_2^2 + \alpha_1 \|J\|_* + \frac{\alpha_2}{2}\|\bmh\|_2^2 \\
s.t. \quad & \bme = \bmx-X\bmz,\ J = X \text{Diag}(\bmz),\ \bmh = \bmz - \bmw. \\
\end{split}
\end{equation}
The augmented Lagrangian function of Eq.~\ref{eq:casc_ialm} is 
\begin{equation}
\nonumber
\begin{split}
\hat{L}(\bme, J, \bmh, \bmz) & = \frac{1}{2}||\bme||_2^2 + \alpha_1||J||_* + \frac{\alpha_2}{2}||\bmh||_2^2\\
& + \bm{\lambda}_1^T(\bme-\bmx+X\bmz) + \bm{\lambda}_2^T(\bmh-\bmz+\bmw) + tr(Y^T(J-X\text{Diag}(\bmz))) \\
& + \frac{\mu}{2}(||\bme-\bmx+X\bmz||_2^2 + ||J-X\text{Diag}(\bmz)||_F^2 + ||\bmh-\bmz+\bmw||_2^2),
\end{split}
\end{equation}
where $\bm{\lambda}_1$, $\bm{\lambda}_2$ and $Y$ are Lagrangian multipliers, and $\mu > 0$ is a penalty parameter.
We adopt an alternative strategy to update variables of $\hat{L}$ as in Sec.~\ref{subsec:scm}
and the update rules are as follows:
\begin{scriptsize}
\begin{equation}
\label{eq:casc_z}
\bmz = (X^TX + I + \text{Diag}(X^TX))^{-1} \left(-\frac{X^T\bm{\lambda}_1}{\mu} - X^T\bme + X^T\bmx + \frac{\bm{\lambda}_2}{\mu} + \bmh + \bmw + \text{diag}((\frac{Y}{\mu} + J)^TX)\right).
\end{equation}
\end{scriptsize}
Here, we overload the notation $\text{Diag}(\cdot)$, which returns a diagonal matrix whose main diagonal is that of the argument matrix. 
Also,  
$\text{diag}(\cdot)$ returns the main diagonal vector of the argument matrix.
\begin{equation}
\label{eq:casc_e}
\bme = \frac{\mu}{\mu + 1}(-\frac{\bm{\lambda}_1}{\mu} + \bmx - X\bmz),
\end{equation}
\begin{equation}
\label{eq:casc_h}
\bmh = \frac{\mu}{\alpha_2 + \mu}(-\frac{\bm{\lambda}_2}{\mu} + \bmz - \bmw).
\end{equation}
Furthermore, updating $J$ is equivalent to solving the sub-problem:
\begin{equation}
\label{eq:casc_min_j}
\min_{J}\ \frac{\alpha_1}{\mu}||J||_* + \frac{1}{2}||J-X\text{Diag}(\bmz) + \frac{Y}{\mu}||_F^2,
\end{equation}
which is convex and has a closed-form solution that can be solved by the Singular Value Thresholding (SVT) operator~\cite{cai2010singular}.
Suppose the singular value decomposition of a rank-$r$ matrix
$X\text{Diag}(\bmz) - \frac{Y}{\mu} = U\Sigma V^*$,
where $\Sigma = \text{Diag}([\sigma_1, ..., \sigma_r])$ and
$\sigma_i$ is the $i$-th largest singular value.
Let $\tau = \frac{\alpha_1}{\mu}$ and $\mathcal{D}_{\tau} (\Sigma) = \text{Diag}([(\sigma_1 - \tau)_+, ..., (\sigma_r - \tau)_+])$.
The solution to Eq.~\ref{eq:casc_min_j} is:
\begin{equation}
\label{eq:casc_j}
J = U\mathcal{D}_{\tau} (\Sigma) V^*.
\end{equation}
%Finally, the algorithm to solve Eq.~\ref{eq:casc} is given in 
Algorithm~\ref{alg_trace_lasso}, which is summarized in Appendix, shows the procedure for solving Eq.~\ref{eq:casc}.
It is pointed out in~\cite{zhang2010recent,liu2013robust} that 
the convergence of inexact ALM cannot be generally proved
when there are three or more variables.
However, 
the convexity of the Lagrangian function $\hat{L}$ guarantees convergence to some extent~\cite{liu2013robust}.
Moreover, there are ways to ensure convergence, e.g., by observing that
$\mu$ is upper bounded by Step 10 of Alg.~\ref{alg_trace_lasso}.
While it is difficult to prove the convergence theoretically, 
inexact ALM has been empirically observed to perform well in practice~\cite{zhang2010recent}.
%We also observe its superior performance in this paper, which will be reported in Sec.~\ref{sec:exp}.

%\begin{algorithm}
%\begin{small}
%\caption{Solving Eq.~\ref{eq:casc} by inexact ALM}
%\label{alg_trace_lasso}
%\begin{algorithmic}[1]
%\Require $\bmx$, $X$, $\bmw$, $k$, $\rho$, $\mu_{\max}$, $\epsilon$
%\Ensure $\bmz$
%\State Initialize $J$, $\bmz$, $\bme$, $\bmh$, $\bm{\lambda}_1$, $\bm{\lambda}_2$, $Y$, $\mu$
%\While{$\|\bme-\bmx+X\bmz\|_{\infty} > \epsilon$ or $\|\bmh-\bmz+\bmw\|_{\infty} > \epsilon$ or $\|J-X\text{Diag}(\bmz)\|_{\infty} > \epsilon$}
%\State Update $\bmz$ by Eq.~\ref{eq:casc_z} with other variables fixed
%\State Update $\bme$ by Eq.~\ref{eq:casc_e} with other variables fixed
%\State Update $\bmh$ by Eq.~\ref{eq:casc_h} with other variables fixed
%\State Update $J$ by Eq.~\ref{eq:casc_j} with other variables fixed
%\State Update the multiplier $\bm{\lambda}_1 = \bm{\lambda}_1 + \mu (\bme-\bmx+X\bmz)$
%\State Update the multiplier $\bm{\lambda}_2 = \bm{\lambda}_2 + \mu (\bmh-\bmz+\bmw)$
%\State Update the multiplier $Y = Y + \mu (J-X\text{Diag}(\bmz))$
%\State Update $\mu = \min (\rho \mu, \mu_{\max})$
%\EndWhile
%\State \Return $\bmz$
%\end{algorithmic}
%\end{small}
%\end{algorithm}

\subsection{Grouping Effect}
Previous works~\cite{hu2014smooth,lu2013correlation,li2018rosc,lu2012robust} have shown that spectral clustering
is effective when applied to data with grouping effect.
As we defined in Def.~\ref{def:group_effect}, with grouping effect, if two objects are highly correlated,
their characterizations of other objects are similar.
We measure object correlation by both a feature similarity and a reachability similarity. 
We next prove that the coefficient vectors $\bmz$'s regularized by the trace Lasso (Eq.~\ref{eq:casc}) result in grouping effect.

%We start with a lemma as follows.

\begin{lemma}
\label{lemma:epsilon1}
Given $\bmw \in \mathbb{R}^d$, let $\bar{w} = \bmw^T\mathbf{1}/d$.
The optimal solution $\bmy^*$ to the problem:
%\begin{equation}
$min_{\bmy} \|\bmy-\bmw\|_2^2,\ \text{s.t.}\ \bmy^T\mathbf{1} = d\bar{y}$,
%\end{equation}
satisfies $y_j^* = \bar{y} + (w_j - \bar{w})$, $\forall 1 \leq j \leq d$.
Moreover, if $\|\bmw - \bar{w}\mathbf{1}\|_2 \leq \epsilon$, then, 
$\bar{y} - \epsilon \leq y_j^* \leq \bar{y} + \epsilon$,
$\forall 1 \leq j \leq d$.
\end{lemma}
\begin{proof}
The Lagrangian function $L'$ of the problem can be written as:
$L' = \sum_{j=1}^d(y_j-w_j)^2 + \beta(\sum_{j=1}^dy_j - d\bar{y})$,
where $\beta$  is the Lagrangian multiplier.
By setting $\frac{\partial L'}{\partial y_j} = 0$,
we get $y_j^* = w_j - \frac{\beta}{2}$.
Since $\bmy$ satisfies 
$\bmy^T\bm{1} = d\bar{y}$,
we substitute $y_j^*$ into the equation
and get $\beta = 2(\bar{w} - \bar{y})$,
$y_j^* = \bar{y} + (w_j - \bar{w})$.
If $\|\bmw - \bar{w}\mathbf{1}\|_2 \leq \epsilon$,
then $|w_j - \bar{w}| \leq \epsilon$,
i.e., 
$- \epsilon \leq w_j - \bar{w} \leq \epsilon$.
Hence,
$\bar{y} - \epsilon \leq y_j^* \leq \bar{y} + \epsilon$,
$\forall 1 \leq j \leq d$.
\end{proof}

Given a set of objects $\mathcal{X} = \{x_1,...,x_n\}$,
let $X$ denote the feature matrix of the objects
and
$\bmz^*$ denote the optimal solution of Eq.~\ref{eq:casc}.
We rearrange $X$ as $X= [\hat{X}, \tilde{X}]$,
where $\tilde{X} \in \mathbb{R}^{d\times q}$ consists of $q$ column vectors that are similar to each other
and $\hat{X} \in \mathbb{R}^{d \times (n-q)}$ consists of the remaining columns.
%is formed by the rest columns of $X$.
In particular,
$\tilde{X}$ satisfies:
\begin{equation}
\nonumber
\max\{\| \tilde{X} - \bar{\bmx}_0\mathbf{1}^T\|_*, \| \tilde{X} - \bar{\bmx}_0\mathbf{1}^T\|_F, \| \tilde{X} - \bar{\bmx}_0\mathbf{1}^T\|_2\} \leq \epsilon,
\end{equation}
where $\epsilon$ is a small positive value,
$\mathbf{1} \in \mathbb{R}^q$ is the all one's vector and
$\bar{\bmx}_0 = \tilde{X}\mathbf{1}/q$ is the mean vector of columns of $\tilde{X}$. 
Similarly, we rearrange $\bmz^* = [\hat{\bmz};\tilde{\bmz}]$.
To prove $\bmz^*$ has grouping effect,
we only need to show that
if $\left \| \tilde{\bmz} - \bar{z} \mathbf{1}\right \|_2 > \delta$,
then $f([\hat{\bmz};\tilde{\bmz}]) > f([\hat{\bmz};\bar{z}\mathbf{1}])$,
where $\bar{z} = \mathbf{1}^T\tilde{\bmz}/q$ is the average value of $\tilde{\bmz}$,
$\delta$ is a positive value
and $f(\bmz) = \frac{1}{2}\| \bmx - X\bmz\|_2^2 + \alpha_1 \| X \text{Diag}(\bmz)\|_* + \frac{\alpha_2}{2} \| \bmz - \bmw \|_2^2$.
Formally,

\begin{theorem}
\label{th:casc_ge}
Given a matrix $X = [\hat{X},\tilde{X}]$ and a vector $\bmw = [\hat{\bmw}; \tilde{\bmw}]$,
let $\bmy^*$ be the optimal solution to the problem: $\min_{\bmy} \| \bmy - \tilde{\bmw} \|_2^2,\ \text{s.t.}\ \bmy^T\mathbf{1} = q\bar{z}$.
$\tilde{X}$ satisfies 
$\max\{\| \tilde{X} - \bar{\bmx}_0\mathbf{1}^T\|_*, \| \tilde{X} - \bar{\bmx}_0\mathbf{1}^T\|_F, \| \tilde{X} - \bar{\bmx}_0\mathbf{1}^T\|_2\} \leq \epsilon$,
and $\tilde{\bmw}$ satisfies
$\bar{w} = \mathbf{1}^T\tilde{\bmw}/d$
and $\|\tilde{\bmw} - \bar{w}\mathbf{1}\|_2 \leq \epsilon$.
%let $\bmz^* = [\hat{\bmz};\tilde{\bmz}]$ be the optimal solution of Eq.~\ref{eq:casc},
%$\bar{z} = \tilde{\bmz}^T\mathbf{1}/q$ and
If $\left\| \tilde{\bmz} - \bar{z}\mathbf{1} \right\|_2 > \delta$, $f([\hat{\bmz};\tilde{\bmz}]) > f([\hat{\bmz};\bar{z}\mathbf{1}])$,
where
\begin{equation}
\nonumber
\delta = \sqrt{\frac{\left (2\gamma - \alpha_2 \sum_{j=1}^q [(y_j^* - \bar{z})(y_j^* + \bar{z} - 2\tilde{w}_j)] \right )(\|[\hat{X}\text{Diag}(\hat{\bmz})\ \bar{\bmx}_0\tilde{\bmz}^T]\|_2)}{\alpha_1\|\bar{\bmx}_0\|_2^2}}
\end{equation}
and $\gamma = ((\alpha_1 + \|\bmx - \hat{X}\hat{\bmz} - \tilde{X} (\bar{z} \mathbf{1})\|_2) \|\tilde{\bmz}\|_2 + \alpha_1 |\bar{z}|)\epsilon.$
\end{theorem}
The proof of Theorem~\ref{th:casc_ge} is given in the Appendix.

\begin{theorem}
$\bmz^*$ has grouping effect.
\end{theorem}
\begin{proof}
Given two objects $x_i$ and $x_j$ in $\tilde{X}$.
When $\epsilon \rightarrow 0$,
$\tilde{X}$ will be close to $\bar{\bmx}_0 \bm{1}^T$
and $\tilde{\bmw}$ will be close to $\bar{w}\bm{1}$.
Hence, $x_i \rightarrow x_j$.
From Lemma~\ref{lemma:epsilon1}, 
$\bmy^*$ in Theorem~\ref{th:casc_ge} satisfies
$-\epsilon \leq y_j^* - \bar{z} \leq \epsilon$.
If $\epsilon \rightarrow 0$,
we have
$\gamma \rightarrow 0$ and 
$\sum_{j=1}^q [(y_j^* - \bar{z})(y_j^* + \bar{z} - 2\tilde{w}_j)] \rightarrow 0$.
Further, we get $\delta \rightarrow 0$.
According to Theorem~\ref{th:casc_ge},
$\tilde{\bmz}$ has to be very close to $\bar{z}\bm{1}$.
As a result,
given two highly correlated objects $x_i$ and $x_j$
such that $x_i \rightarrow x_j$,
we have $z_i^* \rightarrow z_j^*$.
$\bmz^*$ thus has grouping effect.
\end{proof}

\subsection{Clustering Procedure}
Given a set of objects $\mathcal{X} = \{x_1,..., x_n\}$,
%let $\bmz^*_i$ be the optimal solution to Eq.~\ref{eq:casc} for an object $x_i$.
we solve Eq.~\ref{eq:casc} for each object 
and construct an optimal solution for the coefficient matrix $Z^* = [\bmz^*_1, ..., \bmz^*_n]$.
To prevent self-representation,
when we determine $\bmz_i^*$ using Eq.~\ref{eq:casc}, we remove the $i$-th column vector from $X$.
% for an object $x_i$,
%we replace $X$ by $X_i$, which is a matrix obtained by removing the $i$-th column of $X$.
We note that $Z^*$ may be asymmetric and contain negative values. To fix,
%To fix the problem and construct an affinity matrix,
%similar to~\cite{liu2013robust},
\algo\ computes a new matrix $\check{Z} = (|Z^*| + |(Z^*)^T|)/2$ as in~\cite{liu2013robust}.
It is easy to prove that $|Z^*|$, $|(Z^*)^T|$
and thus $\check{Z}$ all have grouping effect.
Moreover, since the trace Lasso automatically self-adjusts to either the $\ell_1$- or the $\ell_2$- norm, 
%automatically balances the effect of $\ell_1$-norm and $\ell_2$-norm,
%based on object correlations,
$Z^*$ enhances sparsity for objects of different clusters.
(We will further illustrate this sparsity effect in the next section.)
%which will be shown in Fig.~\ref{figure:syn1_Z} and~\ref{figure:syn2_Z}.
The matrix $\check{Z}$ is then fed to the pipeline of a standard spectral clustering method (e.g., NCuts) in place of the original similarity matrix $S$.
Algorithm~\ref{alg:casc} in Appendix
%~\ref{sec:codes} 
outlines \algo.

%\begin{algorithm}
%\begin{small}
%\caption{\algo}
%\label{alg:casc}
%\begin{algorithmic}[1]
%\Require $S$, $k$.
%\Ensure $\mathcal{C} = \{C_1, ..., C_k\}$
%\State Compute the TKNN graph and the weight matrix $\mathcal{W}$
%\State Calculate $W = D^{-1}S$, where $D_{ii} = \sum_jS_{ij}$
%%\For $j \leftarrow 1$ do
%%\State$t = 0$
%%\Repeat
%%\State $v_j^{t+1} \leftarrow \frac{Wv_j^t}{||Wv_j^t||_1}$
%%\State $ \delta^{t+1} \leftarrow |v_j^{t+1} - v_j^t|$
%%\State $t$++
%%\Until $||\delta_j^t+1 - \delta_j^t||_{max} \leq \epsilon$ or $t\geq T$
%%\EndFor
%\State Apply PI on $W$ and generate $p$ pseudo-eigenvectors $\{\bm{v}_r\}_{r=1}^p$
%\State $X = \{\bm{v}_1^T; \bm{v}_2^T; ...; \bm{v}_p^T\}$; $X$ = whiten($X$)
%\State Normalize each column vector $\bm{x}$ of $X$ such that $\bm{x}^T\bm{x} = 1$
%%\State Solve Eq.~\ref{eq:casc} for each object and construct a matrix $Z^*$
%\For{$i = 1$ to $n$} 
%\State Solve Eq.~\ref{eq:casc} for an object $x_i$ by inexact ALM and get $\bmz_i^*$
%\EndFor
%\State Calculate the coefficient matrix $Z^* = [\bmz_1^*,...,\bmz_n^*]$
%\State Construct $\check{Z} = (|Z^*| + |(Z^*)^T|)/2$
%\State Run NCuts on $\check{Z}$ to obtain clusters $\mathcal{C} = \{C_r\}_{r=1}^k$
%%\State Decode $\{C_r\}_{r=1}^k$ from $\{{\bm z_r}\}_{r=1}^k$
%\State \Return $\mathcal{C} = \{C_1, ..., C_k\}$
%\end{algorithmic}
%\end{small}
%\end{algorithm}

\comment{
\subsection*{Pseudo-Eigenvectors}
\label{sec:sec:pseudo}
Given a similarity matrix $S$, we normalize it by $D^{-1}S$ and apply PI to obtain {\pev}s.
%When dealing with multi-scale data,
%spectral clustering
%using only the top-$k$ eigenvectors may fail while some other eigenvectors may be useful~\cite{lin2010power}.
%Therefore, it is necessary to fuse the discriminative cluster-separation information in all eigenvectors
%and PI provides such fusion by returning a pseudo-eigenvector.
%However, when the number of clusters is large,
%a single pseudo-eigenvector is insufficient due to the cluster-collision problem. 
%%Since PI returns a pseudo-eigenvector that is a combination of all the eigenvectors,
Similar to~\cite{ye2016fuse},
we run PI multiple times with different random initial vectors to generate a set of {\pev}s, 
which maps each object into a low dimensional embedding.
Note that small {\ev}s are {\it shrunk} by PI~\cite{lin2010power}.
%\footnote{From Equation~\ref{eq:pi}, we see that the $i$-th 
%largest \ev\ is 
%shrunk at a rate of $\lambda_i/\lambda_1$ per iteration.}.
To alleviate the shrinkage of small {\ev}s,
we follow the approach of~\cite{huang2014diverse}
and gradually decrease the number of iterations executed in PI as more {\pev}s are obtained. 

Since the {\pev}s approximate the most dominant {\ev}, they could be similar. 
To reduce this redundancy,
whitening~\cite{kessy2017optimal} is used to make the pseudo-eigenvectors uncorrelated.
Moreover, noise in the {\pev}s are reduced 
by a rectification process, which will be discussed later.

\comment{
First, we generate a set of pseudo-eigenvectors using different random vectors $\bm{v}_0$.
From Eq.~\ref{eq:v0}, $\bm{v}_0$ can be represented by the subspace spanned by all the eigenvectors.
Suppose $\bm{e}_i$ is an informative eigenvector with a relatively small eigenvalue $\lambda_i$.
Given a $\bm{v}_0$, if it has a dominant component in the direction of $\bm{e}_i$, 
i.e., the associated coefficient $c_i$ is large, 
with an appropriate number of iterations, $c_i\lambda_i^t\approx c_2\lambda_2^t$,
the effect of $\bm{e}_i$ will be finally retained in the pseudo-eigenvector.
Therefore, multiple pseudo-eigenvectors generated by different $\bm{v}_0$
may increase the chance to hold all the informative eigenvectors.

Second, we can control the number of iterations in PI.
In PIC, the \emph{velocity} and \emph{acceleration} at the $t$-th iteration are respectively defined as
$\bm{\delta}_t = \bm{v}_t - \bm{v}_{t-1}$ and $\bm{\epsilon}_t = \bm{\delta}_t-\bm{\delta}_{t-1}$.
A threshold $\epsilon$ is introduced to measure the acceleration change in two consecutive iterations 
and stop PI when $|\bm{\epsilon}_t-\bm{\epsilon}_{t-1}|<\epsilon$.
However,
since the effect of eigenvectors corresponding to smaller eigenvalues will shrink as $t$ increases,
we can decrease the number of iterations by gradually increasing the $\epsilon$ value,
which provides another way to keep useful information in ``less important'' eigenvectors. 
}

\subsection*{Transitive {\it K} Nearest Neighbor (TKNN) Graph}

Our objective is to capture the high correlations between objects that belong to the same cluster even 
 though the objects could be located at distant far ends of a cluster. 
 These correlations are expressed via a TKNN graph, which is used to 
 regularize the coefficient matrix $Z$.

\begin{definition}
\label{def:nei_relation}
\textbf{(Mutual  $K$-nearest neighbors)}
Let $N_K(x)$ be the set of $K$ nearest neighbors of an object $x$.
Two objects $x_i$ and $x_j$ are said to be mutual $K$-nearest neighbors of each other,
denoted by $x_i \sim x_j$, 
iff $x_i \in N_K(x_j)$ and $x_j \in N_K(x_i)$.
\hfill$\Box$
\end{definition}

%In the case of multi-scale data,
%the KNN graph is sometimes insufficient to reflect whether two objects belong to the same cluster or not.
%For example, objects $\bmx_2$ and $\bmx_3$ in Fig.~\ref{figure:example1} belong to the same cluster, 
%but they are both far away from each other.
%%objects $x_1$ and $x_2$ are in different clusters despite their closeness.
%To construct an effective KNN graph, $\bmx_2$ and $\bmx_3$
%are expected to be linked.
%Therefore,
%the reachability between two objects is introduced.

\begin{definition}
\label{def:reachability}
\textbf{(Reachability)}
Two objects $x_i$ and $x_j$ are said to be reachable from each other
if there exists a sequence of $h \geq 2$ objects 
 $\{x_i = x_{a_1}, \ldots, x_{a_h} = x_j\}$ such that
 $x_{a_r} \sim x_{a_{r+1}}$ for $1 \leq r < h$.
 \hfill$\Box$
\end{definition}

\begin{definition}
\label{def:trans_relation}
\textbf{(Transitive $K$-nearest neighbor (TKNN) graph)}
Given a set of objects $\mathcal{X} = \{x_1, x_2,..., x_n\}$,
the TKNN graph $\mathcal{G}_K = (\mathcal{X},\mathcal{E})$
is an undirected graph
where $\mathcal{X}$ is the set of vertices and $\mathcal{E}$ is the set of edges.
Specifically, the edge ($x_i$, $x_j$) $\in \mathcal{E}$
iff $x_i$ and $x_j$ are reachable from each other.
We represent the TKNN graph by an $n \times n$ {\it reachability matrix}
$\mathcal{W}$ 
whose ($i$,$j$)-entry $\mathcal{W}_{ij} = 1$ if ($x_i$, $x_j$) $\in \mathcal{E}$; 0 otherwise.
\hfill$\Box$
\end{definition}

%We note that mutual-KNN is a symmetric relation and so is reachability. 

%Compared with the traditional KNN graph, 
%the TKNN graph can more accurately reflect 
%whether two objects belong to the same cluster or not.
%It will be used in the next section to rectify the matrix of similarity between objects.
 
\subsection*{Coefficient Matrix}
%In the case of multi-scale data,
%the distance-based similarity is likely to be ineffective. 
%To improve the clustering performance,
%it is necessary to rectify the matrix.
%The rectification should follow two principles.
%First, for the effective similarities, the efficacy should be reserved.
%Second, for the ineffective ones, the efficacy should be enhanced.
%For example,
%the closeness between objects $\bmx_3$ and $\bmx_4$ in Fig.~\ref{figure:example1}
%leads to a large similarity value.
%Since the two objects are of the same cluster,
%the new value should still be large.
%In comparison, 
%both the similarity between $\bmx_1$ and $\bmx_2$ and the one between $\bmx_2$ and $\bmx_3$
%should be rectified, as these distance-based similarities are ineffective.

%In the multi-scale data,
%different clusters may have different sizes, densities or even geometric structures.
The similarity between two objects describes the degree to which they share common characteristics.
The more similar they are, 
the more likely that one object can be represented by the other. 
Therefore,
to depict the relationships between objects in multi-scale data,
each object is linearly characterized by other objects,
assuming the well-known linear subspace model~\cite{liu2013robust}.

We generate $p$ {\pev}s using PI. 
%just as a non-linear manifold can be locally approximated by linear subspaces.
%First, these vectors, which retain the effects of dominant eigenvectors,
%can certainly retain the original accurate similarities between objects.
%Second, they absorb the cluster-separation information from the seemingly ``less important'' eigenvectors,
%which further provides a reliable source of information to rectify the erroneous similarities.
Let $X\in \mathcal{R}^{p\times n}$ be a matrix whose rows are the {\pev}s.
The $q$-th column of $X$ can be taken as a feature vector $\bmx_q$ of an object $x_q$.
We normalize the column vectors of $X$ such that $\bmx_q^T \bmx_q = 1 \; \forall 1 \leq q \leq n$.
We determine a coefficient matrix $Z \in \mathcal{R}^{n \times n}$ by\footnote{We will regularize $Z$ to avoid the trivial solution of $Z$ being the identity matrix.}
\begin{equation}
\label{eq:nonoise}
%X = XZ,\;\;\; \text{s.t. } diag(Z) = 0.
X = XZ.
\end{equation} 
One can interpret $Z_{ij}$ 
as a value that reflects how well object $x_i$ characterizes object $x_j$.
%To avoid the trivial solution of $Z$ equalling the identity matrix,
%we will regularize $Z$ later.
%For simplicity, we do not explicitly mention this constraint in the rest of the paper.

As indicated by previous works~\cite{ye2016fuse}, the generated pseudo-eigenvectors are likely to be noise-corrupted.
We thus extend Equation~\ref{eq:nonoise} by introducing a noise matrix $O$, giving:
%However, although the pseudo-eigenvectors inherit useful cluster-separation information,
%they may also be contaminated by noise.
%Thus a matrix $O$ is introduced to capture the noise,
\begin{equation}
X = XZ + O.
\label{eq:O}
\end{equation} 
%The TKNN graph conveys information on whether two objects belong to the same cluster or not,

As we have explained, the TKNN graph  conveys useful clustering information, bringing highly
correlated objects that are located at distant far ends of a cluster together.
We thus use the TKNN graph to regularize matrix $Z$.
We derive the following objective function:
%Since our objective is to derive a $Z$ which can better reflect the true similarities between objects,
\comment{
\begin{equation}
%\label{eq:obj}
\min_Z ||O||_F^2+\lambda_1||Z||_F^2 + \lambda_2||Z-\mathcal{W}_K||_F^2,
s.t., X = XZ+O,
\end{equation}
}
%By substituting $O$, we derive
\begin{equation}
\label{eq:obj_constraint}
\min_Z ||X-XZ||_F^2 + \alpha_1 ||Z||_F^2 + \alpha_2 ||Z-\mathcal{W}||_F^2,
\end{equation}
where $\alpha_1 > 0,\alpha_2 \geq 0$ are two 
weighting factors that adjust the relative weights of the three components that constitute
the objective function.
The objective function consists of three terms.
The first term aims to reduce the noise matrix $O$ (see Equation~\ref{eq:O}),
the second term is the Frobenius norm on $Z$, and
%that restricts any item in $Z$ from being dominant.
the third term regularizes $Z$ by the TKNN graph.
%For simplicity, as in~\cite{lu2012robust},
%we remove the constraint and formulate the problem as
%\begin{equation}
%\label{eq:obj}
%\min_Z ||X-XZ||_F^2 + \lambda_1 ||Z||_F^2 + \lambda_2 ||Z-\mathcal{W}_K||_F^2.
%\end{equation}
%$\lambda_1$ and $\lambda_2$ are two parameters to balance the effects of three parts.
A closed-form solution, $Z^*$, to the optimization problem is
\begin{equation}
\label{eq:solution}
Z^* = (X^TX + \alpha_1I + \alpha_2I)^{-1}(X^TX+\alpha_2\mathcal{W}).
\end{equation}

\subsection*{Grouping Effect}
For an object $x_p$, 
let $\bmz_p$ be the $p$-th column of the coefficient matrix $Z$.
We interpret the entries of $\bmz_p$ as the coefficients that express $x_p$ as a linear combinations
of other objects. 
Previous works~\cite{lu2012robust,lu2013correlation,hu2014smooth} have shown that if $Z$ has {\it grouping effect}, then performing spectral 
clustering based on $Z$ would be effective. 
Intuitively, $Z$ has grouping effect if, given two {\it highly correlated} objects $x_i$ and $x_j$,
their characterizations of other objects are similar.
Existing works mostly consider {\it high correlation} between objects as {\it high similarity} of their
feature vectors. With ROSC, we consider object similarity in terms of both feature similarity and 
reachability similarity. 
Feature similarity is measured by the objects' feature vectors as given by the columns of matrix $X$.
Reachability similarity is measured by the columns of matrix $\mathcal{W}$, 
each of which shows the reachability of an object to all others. 
Formally,

%It has been shown that the first two terms in the objective function decide the \emph{grouping effect} of $Z$~\cite{lu2012robust}.
%In Problem~\ref{eq:obj}, the first two terms decides whether $Z$ has the grouping effect~\cite{lu2012robust} defined as follows.
%i.e., the more highly correlated two objects are,
%the larger weight one will have in characterizing the other~\cite{lu2012robust}.

\begin{definition}
\label{def:grouping}
\textbf{(Grouping effect)}. 
Given a set of objects $\mathcal{X} = \{x_1, x_2,..., x_n\}$,
let $\bmw_q$ be the $q$-th column of $\mathcal{W}$. 
Further, let $\xarrow{i}{j}$ denote the condition:
(1) $\bmx_i^T \bmx_j \rightarrow 1$ and 
(2) $\lVert \bmw_i - \bmw_j \rVert_2 \rightarrow 0$.
A matrix $Z$ is said to have grouping effect
if
\[
\xarrow{i}{j} \Rightarrow |Z_{ip} - Z_{jp}| \rightarrow 0\; \forall 1 \leq p \leq n.
\]
%for an arbitrary object $\bmx$ and its associated coefficient vector $\bm{z}$,
%If $\bmx_i\rightarrow \bmx_j$,
%then $z_{pi}\rightarrow z_{pj}$,
%where $z_{pi}$ and $z_{pj}$ are the $i$th and $j$th entry in $\bm{z}_p$ respectively.
%Then we say the self-representation matrix $Z$
%has the grouping effect.
%\hfill$\Box$
\end{definition}

Our next task is to prove that the optimal solution $Z^*$ (as given in Equation~\ref{eq:solution})
% of Equation~\ref{eq:obj_constraint}
has the grouping effect.
In the following discussion, we use $\bmz_q^*$ to denote the $q$-th column vector of $Z^*$.

\begin{lemma}
\label{lemma1}
Given a set of objects $\mathcal{X}$,
the matrix
$X\in \mathcal{R}^{p\times n}$ that is composed of the {\pev}s as rows,
 the reachability matrix $\mathcal{W}$,
 and the optimal soution $Z^*$ of Equation~\ref{eq:obj_constraint},
\begin{equation}
\label{eq:zi}
Z_{ip}^* = \frac{\bm{x}_i^T(\bm{x}_p-X\bm{z}_p^*) + \alpha_2 \mathcal{W}_{ip}}{\alpha_1+\alpha_2}, \;\;\; \forall 1 \leq i, p \leq n.
\end{equation}
\end{lemma}

\begin{proof}
For $1 \leq p \leq n$,
let $J(\bm{z}_p) =  ||\bm{x}_p-X\bm{z}_p||_2^2 + \alpha_1 ||\bm{z}_p||_2^2 + \alpha_2 ||\bm{z}_p-\bm{w}_p||_2^2$.
Since $Z^*$ is the optimal solution of Equation~\ref{eq:obj_constraint}, we have $\frac{\partial{J}}{\partial{Z}_{ip}}|_{\bm{z}_p = \bm{z}_p^*} = 0\; \forall 1\leq i \leq n$.
Hence, $-2\bm{x}_i^T(\bm{x}_p-X\bm{z}_p^*)+2\alpha_1Z_{ip}^*+2\alpha_2(Z_{ip}^*-\mathcal{W}_{ip}) = 0$,
which induces Equation~\ref{eq:zi}.
\end{proof}

%\begin{lemma}
%\label{lemma1}
%Given a set of objects $X\in \mathcal{R}^{d\times n}$ and a TKNN graph $\mathcal{G} = (\mathcal{V}, \mathcal{E}, \mathcal{W}_K)$,
%let $\bm{x}_p$ be an arbitrary object and $\bm{w}_p$ be the column vector in $\mathcal{W}_K$ describing the connectivity of $\bmx_p$.
%Each item $z_{pi}^*$ in the optimal solution $\bm{z}_p^*$ to the problem 
%$\min_{\bm{z}} ||\bm{x}_p-X\bm{z}_p||_2^2 + \lambda_1 ||\bm{z}_p||_2^2 + \lambda_2 ||\bm{z}_p-\bm{w}_p||_2^2$ is
%\begin{equation}
%\label{eq:zi}
%z_{pi}^* = \frac{\bm{x}_i^T(\bm{x}_p-X\bm{z}_p^*) + \lambda_2 w_{pi}}{\lambda_1+\lambda_2},
%\end{equation}
%where $w_{pi} = 1$, if $\bm{x_p}$ and $\bm{x}_i$ are reachable; $0$, otherwise. 
%
%Proof. Let $J(\bm{z}_p) =  ||\bm{x}_p-X\bm{z}_p||_2^2 + \lambda_1 ||\bm{z}_p||_2^2 + \lambda_2 ||\bm{z}_p-\bm{w}_p||_2^2$.
%Since $\bm{z}_p^*$ is the optimal solution, then $\frac{\partial{J}}{\partial{z}_{pi}}|_{\bm{z}_p = \bm{z}_p^*} = 0$.
%Thus we have $-2\bm{x}_i^T(\bm{x}_p-X\bm{z}_p^*)+2\lambda_1z_{pi}^*+2\lambda_2(z_{pi}^*-w_{pi}) = 0$,
%which induces Eq.~\ref{eq:zi}.
%\end{lemma}

\comment{
Eq.~\ref{eq:zi} calculates the weight of object $\bm{x}_i$ in characterizing $\bm{x}$.
For $w_i$, if $\bm{x}$ and $\bm{x}_i$ are reachable in the TKNN graph, $w_i = 1$; otherwise, $w_i = 0$. 
Eq.~\ref{eq:zi} rectifies the ineffective similarity as follows.
When $\bm{x}$ and $\bm{x_i}$ are far away from each other in the feature space
but reachable in the TKNN graph,
$z_i$ will be increased by $w_i = 1$.
%which rectifies the original incorrect similarity value between them.
Moreover, if $\bm{x}$ and $\bm{x}_i$ are highly correlated,
the first two terms in the problem can decide a large $z_i$,
even in the case that they are unreachable.
%$z_i$ can still keep large with a small $\lambda_2$.
In summary, Eq.~\ref{eq:zi} can not only hold the original accurate similarities but also rectify the incorrect ones.

\noindent{\small$\bullet$}
For any two objects dissimilar in the feature space but connected in the TKNN graph, 
their similarity will be increased.

\noindent{\small$\bullet$}
For any two objects dissimilar in the feature space and disconnected in the TKNN graph, 
their similarity will be further decreased.

\noindent{\small$\bullet$}
For any two objects similar in the feature space and connected in the TKNN graph,
their similarity will be further increased.

\noindent{\small$\bullet$}
For any two objects similar in the feature space but disconnected in the TKNN graph,
their similarity will be decreased. In this case,
provided that the raw similarity is effective already,
even though the new value will be decreased,
but it will be still
}

\begin{lemma}
\label{lemma2}
%Given a set of objects $X\in R^{d\times n}$ and a TKNN graph $\mathcal{G} = (\mathcal{V}, \mathcal{E}, \mathcal{W}_K)$,
%let $\bm{x}_p$ be an arbitrary object, 
%$\bm{w}_p$ be the column vector in $\mathcal{W}_K$ describing the connectivity of $\bm{x}_p$, and
%$\bm{z}_p^*$ be the optimal solution to the problem 
%$\min_{\bm{z}_p} ||\bm{x}_p-X\bm{z}_p||_2^2 + \lambda_1 ||\bm{z}_p||_2^2 + \lambda_2 ||\bm{z}_p-\bm{w}_p||_2^2$.
%Assume all the objects have been normalized.
%For any two objects $\bm{x}_i$ and $\bm{x}_j$,
$\forall 1 \leq i, j, p \leq n$,
\begin{equation}
\label{eq:norm}
|Z_{ip}^*-Z_{jp}^*| \leq \frac{c\sqrt{2(1-r)} + \alpha_2|\mathcal{W}_{ip}-\mathcal{W}_{jp}|}{\alpha_1+\alpha_2},
\end{equation}
where $c = \sqrt{1+\alpha_2||\bm{w}_p||_2^2}$ and $r = \bm{x}_i^T\bm{x}_j$.
\end{lemma}

\begin{proof}
From Equation~\ref{eq:zi}, we have 
\begin{equation}
\nonumber
Z_{ip}^*-Z_{jp}^* = \frac{(\bm{x}_i^T - \bm{x}_j^T)(\bm{x}_p-X\bm{z}_p^*) + \alpha_2 (\mathcal{W}_{ip}-\mathcal{W}_{jp})}{\alpha_1+\alpha_2}.
\end{equation}
That implies 
\begin{small}
\begin{equation}
\label{eq:zizj}
\begin{split}
|Z_{ip}^*-Z_{jp}^*| & \leq \frac{|(\bm{x}_i^T - \bm{x}_j^T)(\bm{x}_p-X\bm{z}_p^*)| + \alpha_2 |\mathcal{W}_{ip}-\mathcal{W}_{jp}|}{\alpha_1+\alpha_2}\\
& \leq \frac{||\bm{x}_i - \bm{x}_j||_2||\bm{x}_p-X\bm{z}_p^*||_2 + \alpha_2 |\mathcal{W}_{ip}-\mathcal{W}_{jp}|}{\alpha_1+\alpha_2}\\
\end{split}
\end{equation}
\end{small}

Since the column vectors of $X$ are normalized (i.e., $\bmx_q^T \bmx_q = 1 \; \forall 1 \leq q \leq n$) , we have
$||\bm{x}_i - \bm{x}_j||_2 = \sqrt{2(1-r)}$,
where $r = \bm{x}_i^T\bm{x}_j$.
%measuring the closeness between $\bm{x}_i$ and $\bm{x}_j$ in the feature space.
%As $\bm{z}_p^*$ is the optimal solution, 
As $Z^*$ is the optimal solution of Equation~\ref{eq:obj_constraint}, we have
\begin{equation}
\begin{split}
J(\bm{z}_p^*) & = ||\bm{x}_p-X\bm{z}_p^*||_2^2 + \alpha_1 ||\bm{z}_p^*||_2^2 + \alpha_2 ||\bm{z}_p^*-\bm{w}_p||_2^2 \leq  \\
J(\bm{0}) & = ||\bm{x}_p||_2^2 + \alpha_2 ||\bm{w}_p||_2^2 = 1 + \alpha_2 ||\bm{w}_p||_2^2.
\end{split}
\end{equation}
Hence, $||\bm{x}_p-X\bm{z}_p^*||_2 \leq \sqrt{1 + \alpha_2 ||\bm{w}_p||_2^2} = c$.
Equation~\ref{eq:zizj} can be further simplified as
\begin{equation}
\nonumber
|Z_{ip}^*-Z_{jp}^*| \leq \frac{c\sqrt{2(1-r)}+ \alpha_2 |\mathcal{W}_{ip}-\mathcal{W}_{jp}|}{\alpha_1+\alpha_2}.
\end{equation}
\end{proof}

\begin{lemma}
Matrix $Z^*$ has grouping effect.
\label{lemma:z-star}
\end{lemma}
\begin{proof}
Given two objects $x_i$ and $x_j$ such that $\xarrow{i}{j}$,
we have,  %by definition, 
(1) $\bmx_i^T \bmx_j \rightarrow 1$ and (2) $||\bmw_{i}-\bmw_{j}||_2 \rightarrow 0$.
These imply
$r = \bmx_i^T \bmx_j \rightarrow 1$ and  $|\mathcal{W}_{ip}-\mathcal{W}_{jp}| \rightarrow 0$.
Hence, the two terms of the numerator of the R.H.S of Equation~\ref{eq:norm} are close to 0. 
Therefore, $|Z_{ip}^*-Z_{jp}^*| \rightarrow 0$ and thus $Z^*$ has grouping effect.
\end{proof}

Indeed, Equation~\ref{eq:norm} shows how our algorithm ROSC enhances the effectiveness of 
spectral clustering on multi-scale data. 
Comparing with traditional approaches, which focus on feature similarity, ROSC uses $Z^*$
to integrate feature similarity  ($r$) with
reachability similarity ($|\mathcal{W}_{ip}-\mathcal{W}_{jp}|$).
In particular, two distant objects $x_i$ and $x_j$ of a cluster may not share a strong feature similarity.
This leads to a small $r$ and traditional approaches will likely put them into separate clusters.
On the contrary, ROSC considers the strong reachability of the objects
to derive a small value of $|\mathcal{W}_{ip}-\mathcal{W}_{jp}|$, and thus keeping them in the same cluster.
Moreover, for $x_i$ and $x_j$ that belong to two different dense clusters but happen to be close in 
space (i.e., $x_i$ and $x_j$ have strong feature similarity), 
traditional approaches may inadvertently merge them into the same cluster. 
ROSC, however, would discover their low reachability (via the mutual-KNN relation)
and derive a large value of $|\mathcal{W}_{ip}-\mathcal{W}_{jp}|$.
This regulates matrix $Z^*$ and avoids the incorrect merging.
As we will see in the next section, ROSC's approach 
greatly improves clustering quality and is more robust than other algorithms in handling
multi-scale data.

\comment{
\noindent{\small$\bullet$}
$\bmx_i^T \bmx_j \rightarrow 0$ but reachable, $S_{ij}$ will be increased.

\noindent{\small$\bullet$}
%For any two objects dissimilar in the feature space and disconnected in the TKNN graph, their similarity will be further decreased.
$\bmx_i^T \bmx_j \rightarrow 0$ and unreachable, $S_{ij}$ will be further decreased.

\noindent{\small$\bullet$}
%For any two objects similar in the feature space and connected in the TKNN graph, their similarity will be further increased.
$\bmx_i^T \bmx_j \rightarrow 1$ and reachable, $S_{ij}$ will be further increased.

\noindent{\small$\bullet$}
$\bmx_i^T \bmx_j \rightarrow 1$ but unreachable, $S_{ij}$ will be decreased.
}

\subsection*{ROSC: Robust Spectral Clustering}
%The coefficient matrix $Z$ take advantage of both cluster-separation information in all the eigenvectors
%and the TkNN graph, which not only retains the original accurate similarities but also rectifies the incorrect ones.

We note that the matrix $Z^*$ obtained may be asymmetric and it may contain negative values.
To construct a matrix of object similarity, a common fix~\cite{liu2013robust,lu2012robust} is to 
compute $\tilde{Z} = (|Z^*|+|(Z^*)^T|)/2$.
After $\tilde{Z}$ is computed, ROSC executes a standard spectral clustering method
(e.g., NCuts) using $\tilde{Z}$ as the similarity matrix in place of $S$.
It can be proved that $|Z^*|$, $|(Z^*)^T|$, and hence $\tilde{Z}$ all have grouping effect.
Due to space limitations, readers are referred to~\cite{technicalreportrobust} for the proofs. 
Finally, ROSC is summarized in Algorithm~\ref{alg}.

\comment{
\begin{lemma}
\label{lemma3}
$(Z^*)^T$ has the grouping effect. 

Proof. The problem is equivalent to if $x_i \rightarrow x_j$, $Z^*_{pi} \rightarrow Z^*_{pj}$.
From Lemma~\ref{lemma1},
\begin{equation}
\nonumber
z_{pi}^*-z_{pj}^* = \frac{\bm{x}_p^T ( \bmx_i - \bmx_j - X(\bm{z}_i^*-\bm{z}_j^*)) + \alpha_2 (W_{pi}-W_{pj})}{\alpha_1+\alpha_2}.
\end{equation}
Since $\bm{z}_i^* = (X^TX + \alpha_1I + \alpha_2I)^{-1}(X^T\bmx_i+\alpha_2\bm{w}_i),\bm{z}_j^* = (X^TX + \alpha_1I + \alpha_2I)^{-1}(X^T\bmx_j+\alpha_2\bm{w}_j)$,
let $Y = X(X^TX + \alpha_1I + \alpha_2I)^{-1}$. Then
\begin{equation}
\begin{split}
||Z_{pi}^*-Z_{pj}^*||_2 & \leq \frac{||\bmx_p^T(\bm{x}_i - \bm{x}_j)||_2 + ||\bmx_p^TX(\bm{z}_i^*-\bm{z}_j^*)||_2 + \alpha_2 ||(W_{pi}-W_{pj})||_2}{\alpha_1+\alpha_2}\\
& \leq \frac{||\bmx_p^T(\bm{x}_i - \bm{x}_j)||_2 + ||\bmx_p^TYX^T(\bmx_i-\bmx_j)||_2}{\alpha_1+\alpha_2}\\
& + \frac{\alpha_2||\bmx_p^TY(\bm{w}_i - \bm{w}_j)||_2 + \alpha_2 ||W_{pi}-W_{pj}||_2}{\alpha_1+\alpha_2}\\
\end{split}
\end{equation}
If $x_i \rightarrow x_j$, i.e., $\bmx_i^T\bmx_j \rightarrow 1$ and $||\bm{w}_i - \bm{w}_j||_2 \rightarrow 0$, 
we have
$||\bm{x}_i - \bm{x}_j||_2 \rightarrow 0$ and $||W_{pi}-W_{pj}||_2 \rightarrow 0$.
Then $||Z_{pi}^*-Z_{pj}^*||_2 \rightarrow 0$,
which proves the grouping effect of $(Z^*)^T$.
%\begin{equation}
%\begin{split}
%||\bmx_p^TX(\bm{z}_i^*-\bm{z}_j^*)||_2 & \leq ||\bmx_p^TY(X^T(\bmx_i-\bmx_j) +\alpha_2(\bm{w}_i - \bm{w}_j))||_2\\
%& \leq ||\bmx_p^TYX^T(\bmx_i-\bmx_j)||_2 +\alpha_2||\bmx_p^TY(\bm{w}_i - \bm{w}_j)||_2\\
%\end{split}
%\end{equation}
\end{lemma}

\begin{lemma}
\label{lemma4}
$\tilde{Z}$ has the grouping effect. 

Proof. From Lemma~\ref{lemma2} and~\ref{lemma3},
both $Z^*$ and $(Z^*)^T$ have the grouping effect,
which can be easily extended to
$|Z^*|$ and $|(Z^*)^T|$.
Since $\tilde{Z} = (|Z^*|+|(Z^*)^T|)/2$,
$\tilde{Z}$ also has the grouping effect.
\end{lemma}

To this end, we summarize ROSC as follows.
ROSC first computes the TKNN graph $\mathcal{G}$
and the associated weight matrix $\mathcal{W}_K$. Then it applies power iteration to generate $p$
pseudo-eigenvectors that form a $p \times n$ matrix $X$.
After whitening and normalization on $X$,
the coefficient matrix $Z^*$ can be calculated by solving Eq.~\ref{eq:solution},
which further leads to the construction of a rectified similarity matrix $\tilde{Z}$.
Since $\tilde{Z}$ has the grouping effect, 
the standard spectral clustering methods (e.g., NCuts) will be finally applied to derive more robust clustering results.
The time complexity of ROSC will be no more than the standard spectral clustering methods,
which is $O(n^3)$ in general.
In the future, we will attempt to improve it. 
}

\begin{algorithm}
\begin{small}
\caption{ROSC}
\label{alg}
\begin{algorithmic}[1]
\Require $S$, $k$.
\Ensure $\mathcal{C} = \{C_1, ..., C_k\}$
\State Compute the TKNN graph and the reachability matrix $\mathcal{W}$
\State Calculate $W = D^{-1}S$, where $D_{ii} = \sum_jS_{ij}$
%\For $j \leftarrow 1$ do
%\State$t = 0$
%\Repeat
%\State $v_j^{t+1} \leftarrow \frac{Wv_j^t}{||Wv_j^t||_1}$
%\State $ \delta^{t+1} \leftarrow |v_j^{t+1} - v_j^t|$
%\State $t$++
%\Until $||\delta_j^t+1 - \delta_j^t||_{max} \leq \epsilon$ or $t\geq T$
%\EndFor
\State Apply PI on $W$ and generate $p$ pseudo-eigenvectors $\{\bm{v}_r\}_{r=1}^p$
\State $X = \{\bm{v}_1^T; \bm{v}_2^T; ...; \bm{v}_p^T\}$; $X$ = whiten($X$)
\State Normalize each column vector $\bm{x}$ of $X$ such that $\bm{x}^T\bm{x} = 1$
\State Calculate the coefficient matrix $Z^*$ by Eq.~\ref{eq:solution}
\State Construct $\tilde{Z} = (|Z^*| + |(Z^*)^T|)/2$
\State Run standard spectral clustering methods, e.g., NCuts, with $\tilde{Z}$ as the
similarity matrix to obtain clusters $\mathcal{C} = \{C_r\}_{r=1}^k$
%\State Decode $\{C_r\}_{r=1}^k$ from $\{{\bm z_r}\}_{r=1}^k$
\State \Return $\mathcal{C} = \{C_1, ..., C_k\}$
\end{algorithmic}
\end{small}
\end{algorithm}

}

\section{Experiment}
\label{sec:exp}
In this section,
we conduct extensive experiments
to evaluate the performance of \algo.
We compare \algo\ with 10 other clustering methods on
a wide range of datasets
w.r.t. three popular measures, 
namely, \emph{purity}, \emph{adjusted mutual information (AMI)}, and \emph{rand index (RI)}.
These measures evaluate the clustering quality with values in the range from 0 to 1.
A larger value indicates a better
clustering quality. 
For details of the three measures, see~\cite{vinh2010information,lin2010power}.
We include the experiment settings in the Appendix.
%We first describe the performance measures (Section~\ref{sec:measures}) and
%existing methods against which ROSC is compared (Section~\ref{sec:algo-comp}).
%We then show the performance results on both real and synthetic datasets (Section~\ref{sec:results}).
%We illustrate the grouping effect of matrix $\tilde{Z}$ (Section~\ref{}).
%Finally, we conduct a parameter sensitivity study of ROSC (Section~\ref{}).

\comment{
\subsection{Measures}
\label{sec:measures}
We use three popular measures, namely, \emph{purity}, \emph{adjusted mutual information (AMI)}, and \emph{rand index (RI)}, to evaluate clustering quality~\cite{vinh2010information,lin2010power}.

Consider a clustering $\mathcal{C} = \{C_1, \ldots, C_k\}$ produced by a clustering algorithm
and a gold standard (true) clustering
$\mathcal{C}_t = \{ \hat{C}_1, \ldots, \hat{C}_k\}$.
For each cluster $C_i \in \mathcal{C}$, we find the cluster $\hat{C}_j \in \mathcal{C}_t$ that overlaps
with $C_i$ the most. 
The purity of cluster $C_i$ is the fraction of objects in $C_i$ that fall in the overlap, i.e., 
($\max_j |C_i \cap \hat{C}_j|) / |C_i|$. 
The purity of a clustering $\mathcal{C}$ is the average of its clusters' purities, weighted by the cluster sizes:
\begin{equation}
purity(\mathcal{C}_t,\mathcal{C}) = \frac{1}{n}\sum_{i}\max_{j}|C_i \cap \hat C_j|.
\end{equation}
The adjusted mutual information ({\it AMI}) is mutual information with the agreement due to chance between clusterings corrected, and
is given by,
\begin{equation}
\mathit{AMI}(\mathcal{C}_t,\mathcal{C}) = \frac{MI(\mathcal{C}_t,\mathcal{C}) - E\{MI(\mathcal{C}_t,\mathcal{C})\}}{\max\{H(\mathcal{C}_t),H(\mathcal{C})\} - E\{MI(\mathcal{C}_t,\mathcal{C})\}},
\end{equation}
where $MI(\mathcal{C}_t,\mathcal{C})$ is the mutual information between $\mathcal{C}_t$ and $\mathcal{C}$,
$H(\mathcal{C}_t)$ and $H(\mathcal{C})$ are the entropies of $\mathcal{C}_t$ and $\mathcal{C}$, respectively,
and $E\{MI(\mathcal{C}_t,\mathcal{C})\}$ is the expected mutual information between the two clusterings
$\mathcal{C}_t$ and $\mathcal{C}$.

Rand index ({\it RI}) considers object pairs in measuring clustering quality. It is defined as:
\begin{equation}
RI(\mathcal{C}_t,\mathcal{C}) = (N_{00} + N_{11}) / {\tbinom n2},
\end{equation}
where $N_{00}$ is the number of object pairs that are put into the same cluster in $\mathcal{C}_t$
as well as in the same cluster in $\mathcal{C}$, and
$N_{11}$ is the number of object pairs that are put into different clusters in 
$\mathcal{C}_t$ and also in different clusters in $\mathcal{C}$.
Note that values of all three measures range from 0 to 1, with a larger value indicating a better
clustering quality. 
}

\subsection{Algorithms for comparison}
\label{sec:algo-comp}
We group the algorithms into the following four categories.

\noindent$\bullet$
\textbf{(Standard spectral clustering methods)}: 
{\bf NCuts} and {\bf NJW} are two standard methods.
They differ in the way they normalize the graph Laplacian, $D-S$, 
where $D$ is a diagonal matrix with $D_{ii} = \sum_{j=1}^nS_{ij}$.
NCuts uses the random-walk-based normalization $D^{-1}(D-S)$
while NJW employs the symmetric normalization $D^{-\frac{1}{2}}(D-S)D^{-\frac{1}{2}}$.
%which follow the execution pipelines shown
%in Figure~\ref{figure:flow_graph}(a). 

\noindent$\bullet$
\textbf{(Power iteration (PI)-based methods)}: 
{\bf PIC}, {\bf PIC-$k$}, {\bf DPIC}, and {\bf DPIE} apply PI to generate pseudo-eigenvectors 
as a replacement of eigenvectors. 
They were described in Section~\ref{sec:related}.
%PIC is a PI based method which generates only one pseudo-eigenvector.
%In comparison, PIC-\emph{k} generates $\lceil log(k)\rceil$ pseudo-eigenvectors.
%DPIC employs Schur complement deflation to generate mutually orthogonal pseudo-eigenvectors.
%Finally, DPIE generates a set of diverse pseudo-eigenvectors.

\noindent$\bullet$
\textbf{(Multi-scale-data-oriented methods)}: 
{\bf ZP} and {\bf FUSE} are two methods that are specifically designed to handle multi-scale data.
They were discussed in Section~\ref{sec:related}.
Note that ZP automatically estimates the number of clusters.  For a fair comparison,
we modify ZP so that it  returns $k$ (the number of true) clusters.

%
%
%ZP is a self-tuning spectral clustering method which considers the local scale information for objects in multi-scale data.
%It can automatically estimate the number of clusters.
%To make a fair comparison, we directly set $k$ as in other methods.
%Although FUSE is based on PI, 
%it further uses ICA to rotate the pseudo-eigenvectors to make them pair-wise statistically independent.
%The rotated pseudo-eigenvectors can lead to desiring performance on multi-scale data.

\noindent$\bullet$
\textbf{(Matrix-reconstruction methods)}: 
This group includes {\bf ROSC} and {\bf \algo}.
We also consider a version of ROSC that regularizes the $Z$ matrix using the $\ell_1$-norm instead of  the Frobenius norm.
(That is, we replace Eq.~\ref{eq:rosc_l2} by Eq.~\ref{eq:rosc_l1} and
solve the optimization problem using Alg.~\ref{alg_sparse}.)
We call the method {\it ROSC with Sparse Matrix}, or {\bf ROSC-S}.

%This group of methods construct a new coefficient matrix ($Z$) based on which 
%new similarity matrix as a replacement of the original one.
%{\bf ROSC} generates a matrix with grouping effect.
%To derive a sparse matrix, we 
%replace Eq.~\ref{eq:rosc_l2} in ROSC by Eq.~\ref{eq:rosc_l1} and
%solve it by Alg.~\ref{alg_sparse}.
%We call such method \textbf{ROSC-S} (\textbf{ROSC}-\textbf{S}parse).
%Moreover, \textbf{CASC} is our proposed method that is summarized in Alg.~\ref{alg:casc}.

\comment{
\noindent{\small$\bullet$}
\textbf{NJW}: A standard spectral clustering method with symmetric normalization.

\noindent{\small$\bullet$}
\textbf{NCuts}: A standard spectral clustering method with divisive normalization.

\noindent{\small$\bullet$}
\textbf{ZP}: A self-tuning spectral clustering method for multi-scale clusters.
It uses eigenvector rotation to estimate the number of clusters.
To make a fair comparison, we directly set $k$ as in other methods.

\noindent{\small$\bullet$}
\textbf{PIC}: A power iteration based method which generates only one pseudo-eigenvector.

\noindent{\small$\bullet$}
\textbf{PIC-$k$}:
A power iteration based method which generates $\lceil log(k)\rceil$ pseudo-eigenvectors.

\noindent{\small$\bullet$}
\textbf{DPIC}: 
A power iteration based method which employs Schur complement deflation to generate mutually orthogonal pseudo-eigenvectors.

\noindent{\small$\bullet$}
\textbf{DPIE}: A power iteration based method which generates a set of diverse pseudo-eigenvectors.
%As suggested by~\cite{huang2014diverse}, 
%we set the maximum number of pseudo-eigenvectors to be $\lceil log(k) \rceil * 6 $ 
%out of $\max(\lceil log(k) \rceil * 30, 2k)$ random initial vectors.

\noindent{\small$\bullet$}
\textbf{FUSE}:
A full spectral clustering method based on power iteration and ICA. 
}

%\textbf{[Experiment settings]}
%%The parameters of all the methods are set according to their original papers.
%%For text data, cosine similarity is used to calculate the similarity matrix $S$.
%%For attributed data, 
%For all the datasets,
%the similarity matrix $S$ is computed based on Euclidean distance of objects' attributes.
%$S$ is also locally scaled as is done in ZP.
%%which are locally scaled based on ZP's local scaling procedure.
%All the methods employ \emph{k}-means as the last step of the clustering pipeline to return clusters.
%For this step, 
%we run \emph{k}-means $100$ times with random starting centroids
%and the most frequent cluster assignment is used~\cite{lin2010power}.
%For ROSC, ROSC-S and \algo,
%%we generate $k$ pseudo-eigenvectors with random starting vectors as is done in~\cite{thang2013deflation}.
%we set $K = 4$ in constructing the TKNN graph as suggested in~\cite{li2018rosc}, and 
%fine tune the parameters by grid search
%for $\alpha_1,\alpha_2 \in \{0.001,0.01,0.1,1,10\}$ 
%%and $\alpha_2 \in \{0.001,0.01,0.1,1\}$
%to report the best results.
%For other methods,
%parameters are set according to their original papers.
%For each method and dataset,
%we run the experiment 50 times and report average results.
%Our codes and datasets are publicly available at
%\url{https://github.com/lixiang3776/CAST}.
%%https://www.dropbox.com/s/1edmi8z5xmpl3eh/CAST.zip?dl=0.

\subsection{Performance results}
\label{sec:results}
%We evaluate CASC on two synthetic datasets and 8 real datasets.

\noindent{\bf[Synthetic datasets]}
We first use two synthetic datasets to illustrate the characteristics of the 11 methods.
Both datasets consist of clusters with various densities and sizes. 
Fig.~\ref{figure:syn1}(a) shows \textsc{Syn1}, in which a sparse rectangular cluster (magenta)
is sandwiched between a small dense circular cluster (yellow)
and a large dense rectangular one (navy blue).
The second dataset \textsc{Syn2} is illustrated in Fig.~\ref{figure:syn2}(a), in which
two dense square clusters are very close to a sparse half-ring cluster (red).
In both datasets,
an object in an elongated cluster 
%(e.g., the navy blue cluster in \textsc{Syn1} and the half ring in \textsc{Syn2}) 
can be closer to an object of another cluster than to
an object that is at a far end of the same cluster. 
This is purposely done to make clustering very difficult and so that we can visually compare 
the matrix-reconstruction methods.

%
%the distance between objects in adjacent clusters could be smaller than that of those in the elongated cluster
%(e.g., the navy blue cluster in \textsc{Syn1} and the half ring in \textsc{Syn2}),
%which increases the difficulty of clustering on the datasets.

\begin{table*}[!htbp]
\centering
\resizebox{0.85\linewidth}{!}
{
\begin{tabular}{|c||c|c||c|c|c|c||c|c||c|c|c|} \hline
Measure &NJW & NCuts & PIC & PIC-$k$ & DPIC & DPIE & ZP & FUSE & ROSC-S & ROSC & \algo \\ \hline 
Purity & $0.8000$ & $0.8000$ &  $0.7229$ & $0.7220$ & $0.6085$ & $0.7564$ & $0.8000$ & $0.7607$ & $0.7786$ & $0.7826$ & $\bm{0.8122}$\\ \hline
AMI    & $0.4213$ & $0.4216$ &  $0.4092$ & $0.4221$ & $0.1406$ & $0.4523$ & $0.4217$ & $0.4691$ & $0.5060$ & $0.4874$ & $\bm{0.5430}$ \\ \hline
RI       & $0.6953$ & $0.6956$ & $0.6474$ & $0.6586$ & $0.5421$ & $0.6605$ & $0.6958$ & $0.6898$ & $0.7070$ & $0.7021$ & $\bm{0.7438}$\\ \hline
\end{tabular}
}
\caption{Purity, AMI, and RI scores of methods for dataset \textsc{Syn1}}
\label{table:syn1}

\resizebox{0.85\linewidth}{!}
{
\begin{tabular}{|c||c|c||c|c|c|c||c|c||c|c|c|c|} \hline
Measure &NJW & NCuts & PIC & PIC-$k$ & DPIC & DPIE & ZP &  FUSE & ROSC-S & ROSC & \algo\\ \hline 
Purity & $0.6923$ & $0.6917$ & $0.6741$ & $0.6648$ & $0.5556$ & $0.5971$ & $0.6917$ & $0.7298$ & $0.7257$ & $0.7797$ & $\bm{0.8188}$\\ \hline
AMI    & $0.4472$ & $0.4468$ & $0.4361$ & $0.4158$ & $0.2202$ & $0.1485$ &$0.4468$ &  $0.4856$ & $0.4859$ & $0.5611$ & $\bm{0.6340}$\\ \hline
RI       & $0.6635$ & $0.6632$ & $0.6470$ & $0.6354$ & $0.5139$ & $0.4637$ & $0.6632$ & $0.6976$ & $0.6888$ & $0.7397$ & $\bm{0.7683}$\\ \hline
\end{tabular}
}
\caption{Purity, AMI, and RI scores of methods for dataset  \textsc{Syn2}}
\label{table:syn2}
\end{table*}

\begin{figure*}[!htbp]
\vspace{-5mm}
    \centering
        \includegraphics[width = 0.8\linewidth]{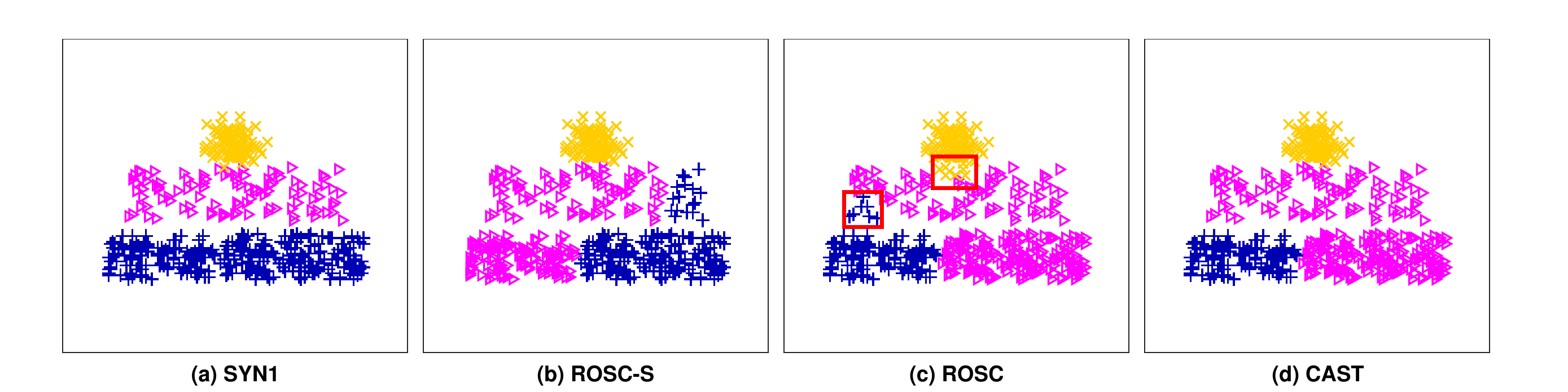}
        \caption{Clustering results for \textsc{Syn1}}
        \label{figure:syn1}

        \includegraphics[width = 0.8\linewidth]{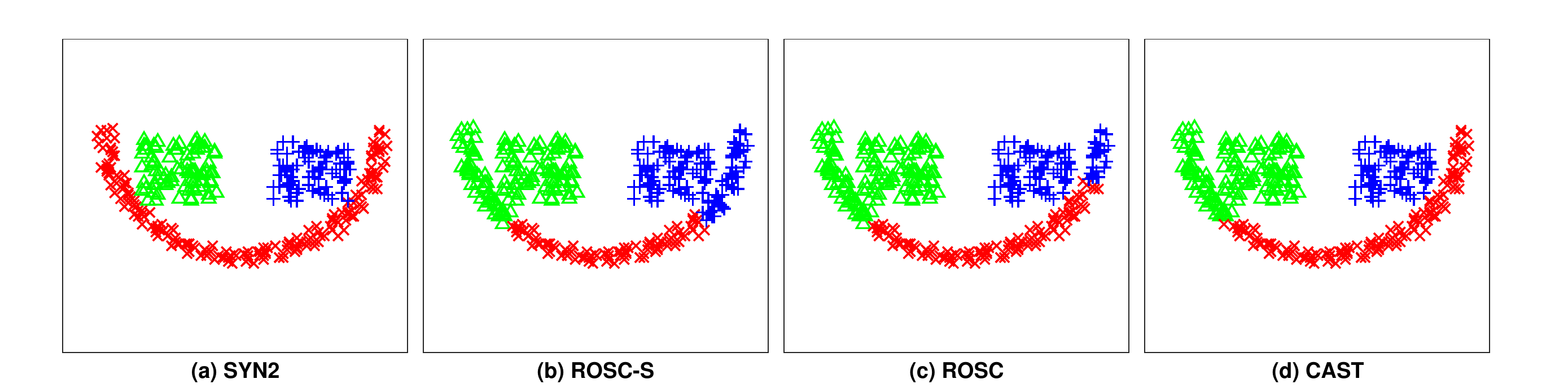}
        \caption{Clustering results for \textsc{Syn2}}
        \label{figure:syn2}
\end{figure*}

 \begin{figure*}[!htbp]   
 %\vspace{-5mm}
 \centering    
        \includegraphics[width = 0.8\linewidth]{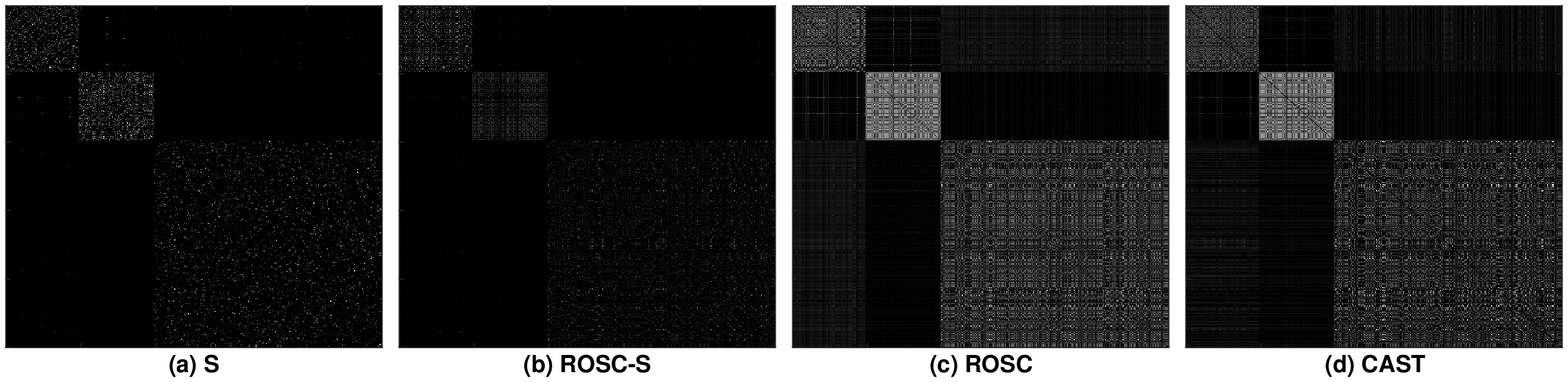}
        \caption{Similarity matrix $S$ (a) and constructed matrices by ROSC-S (b), ROSC (c), and \algo\ (d) for \textsc{Syn1}}
        \label{figure:syn1_Z}

    \centering
        \includegraphics[width = 0.8\linewidth]{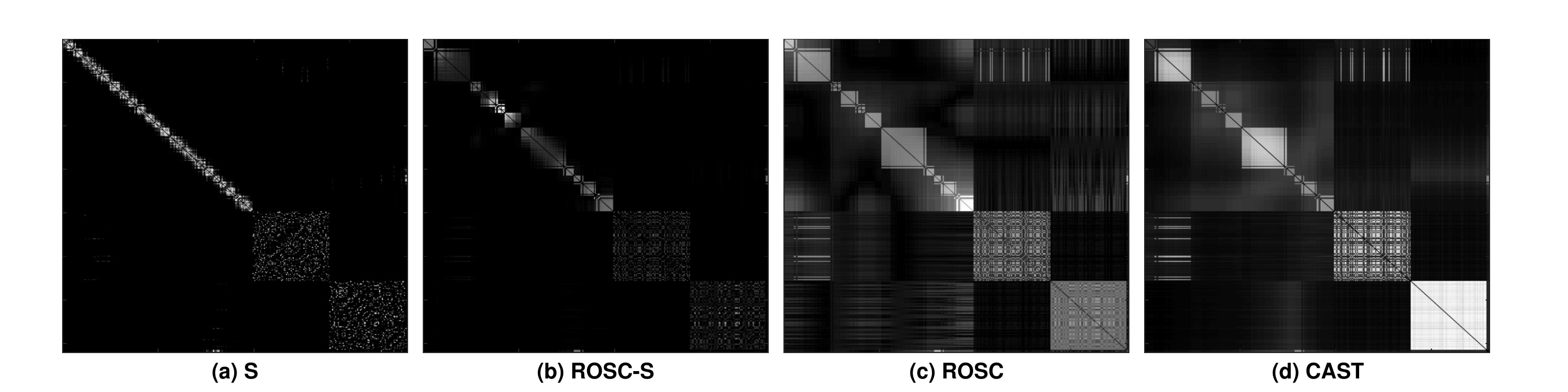}
        \caption{Similarity matrix $S$ (a) and constructed matrices by ROSC-S (b), ROSC (c), and \algo\ (d) for \textsc{Syn2}}
        \label{figure:syn2_Z}
\end{figure*}

Tables~\ref{table:syn1} and \ref{table:syn2} show the clustering performance of the methods
for the datasets \textsc{Syn1} and \textsc{Syn2}, respectively. 
From the tables, 
we see that matrix-reconstruction methods generally perform very well.
For \synone\ (Table~\ref{table:syn1}), w.r.t. Purity measure, ROSC, ROSC-S and \algo\ either perform better or
comparably with the other methods. 
For measures AMI and RI, the matrix-reconstruction methods outperform others.
For \syntwo\ (Table~\ref{table:syn2}),
\algo\ and ROSC outperform others by wide margins w.r.t. all three measures. 
ROSC-S has either better or comparable performance with other non-matrix-reconstruction methods.
Moreover, \algo\ achieves the best performance over all three measures on both \synone\ and \syntwo. 

We visually compare the performance of the matrix-reconstruction methods for \synone\ in Fig.~\ref{figure:syn1}.
Recall that ROSC-S uses the $\ell_1$-norm to regularize the coefficient matrix. 
This promotes sparsity. However, as shown in Fig.~\ref{figure:syn1}(b), the sparse magenta cluster is incorrectly 
chipped off on the right side. 
In contrast, ROSC, which uses the Frobenius norm and promotes object correlation, is connecting objects aggressively. 
This causes some clusters to overspread into close neighboring clusters (see the regions enclosed in red boxes in Fig.~\ref{figure:syn1}(c)).
From Fig.~\ref{figure:syn1}(d), we see that \algo\ rectifies the two problems by striking a balance between object connectivity and sparsity using
the trace Lasso regularizer. This explains \algo\ being the best method for \synone. 
Figs.~\ref{figure:syn2}(b), (c), (d) visually compare the three methods for \syntwo. 
Again, ROSC-S promotes sparsity, and for \syntwo, it inadvertently segments the half-ring cluster into three parts. 
ROSC, which promotes object correlation, recovers more objects of the half-ring cluster,
but the half-ring is still split into three segments.
%most of the half-ring but is a bit too aggressive resulting in the
%half-ring cluster spreading to the green cluster a little.
For \syntwo, \algo\ avoids the merging of objects on the right side of the half-ring cluster with the square cluster.
%The half-ring cluster can thus be more correctly identified.
This shows the adaptability of the trace Lasso regularizer. 

%These methods use the TKNN graph to correlate objects at far ends of an elongated cluster,
%and the similarity matrix is thereby rectified.
%Compared with ROSC and ROSC-S,
%CASC constructs a matrix with grouping effect for highly correlated objects
%and sparsity for uncorrelated ones.
%It thus has an outstanding performance.
%In particular,
%CASC achieves the best performance on \textsc{Syn1}.
%Despite a small gap with ROSC on \textsc{Syn2},
%CASC still outscores other methods.
%
%To make an in-depth anlysis on these matrix-reconstruction-based methods,
%Fig.~\ref{figure:syn1}(b)-(d) and Fig.~\ref{figure:syn2}(b)-(d)
%show their clustering results on \textsc{Syn1} and \textsc{Syn2}, respectively.
%In Fig.~\ref{figure:syn1},
%we see that ROSC-S and CASC correctly identify the yellow circular cluster,
%while ROSC incorrectly merges some objects in the magenta cluster with the circular cluster. 
%%This is because
%%ROSC-S and CASC construct a matrix with sparsity,
%%which largely reduces connections between objects from different clusters.
%For ROSC and ROSC-S,
%some objects in the magenta cluster
%are merged with the navy blue cluster,
%which is avoided in the clustering of CASC.
%For \textsc{Syn2},
%ROSC-S splits the half ring into three segments,
%two of which are incorrectly merged with the two square clusters.
%In contrast,
%ROSC and CASC can recover a large fraction of the half-ring cluster
%and group most green objects as a single cluster.

%We further compare $\check{Z}$ used in these methods.
Recall that the objective of the matrix-reconstruction methods is to construct a new matrix in place of the original similarity matrix $S$
and use the constructed matrix as input to the spectral clustering pipeline. 
Figs.~\ref{figure:syn1_Z} and \ref{figure:syn2_Z} display the constructed matrices (together with the original similarity matrice $S$) for 
\synone\ and \syntwo, respectively.
%Fig.~\ref{figure:syn1_Z} further compares
%the original similarity matrix $S$
%with the constructed matrices
%$\check{Z}$ of ROSC-S, ROSC and CASC on \textsc{Syn1}.
%Similarly, Fig.~\ref{figure:syn2_Z} show these matrices on \textsc{Syn2}.
Each figure displays values in a matrix by pixel brightness.
Rows and columns in the matrix are reordered by gold-standard clusters.
Readers are advised to view the figures magnified on a computer screen.  
Intuitively, each luminous rectangular block corresponds to a cluster.
Ideally, a figure should have 3 blocks (because there are 3 clusters in each dataset); each block is brightly lit (showing high intra-cluster correlation); and
pixels outside the blocks are black (showing sparse inter-cluster correlation). 

From Figs.~\ref{figure:syn1_Z}(a) and \ref{figure:syn2_Z}(a), we see that the original similarity matrices $S$ do not have the desired properties. 
In particular, pixels in blocks are sparse making the blocks not very visible. 
This results in poor clustering when spectral clustering is applied on $S$ directly. 
From Figs.~\ref{figure:syn1_Z}(c) and \ref{figure:syn2_Z}(c), we see that by using the Frobenius norm, ROSC is able to significantly amplify intra-cluster correlations,
resulting in brightly-lit blocks. However, inter-cluster correlations are inadvertently amplified as well, which is particularly striking in Fig~\ref{figure:syn2_Z}(c). 
On the other hand, ROSC-S, which uses the $\ell_1$-norm to promote sparsity, reduces inter-cluster correlation at the expense of  less defined blocks. 
Finally, Figs.~\ref{figure:syn1_Z}(d) and \ref{figure:syn2_Z}(d) show that \algo\ strikes a better balance between intra-cluster correlation and inter-cluster sparsity.
This gives better-lit blocks with dimer regions outside the blocks compared with those of ROSC.

%From the figures,
%we see that
%the original distance-based measure cannot effectively capture correlations between objects in elongated clusters.
%For example,
%objects in the navy blue cluster of \textsc{Syn1}
%have very weak correlations
%and their 
%corresponding cluster block in Fig.~\ref{figure:syn1_Z}(a) is very sparse.
%So do objects in the half-ring cluster of \textsc{Syn2}.
%Due to the $\ell_1$-norm regularization,
%matrices generated by ROSC-S are very sparse.
%ROSC-S weakens not only connections between objects from different clusters,
%but that of objects in the same cluster.
%%Even though objects from different clusters have few connections,
%%relations between objects in the same cluster are also weakened.
%For ROSC,
%matrices exhibit very strong grouping effect,
%but inter-cluster connections are also strengthened.
%This could adversely affect the clustering performance.
%By using trace Lasso regularization,
%CASC constructs $\check{Z}$ with both grouping effect and sparsity in Fig.~\ref{figure:syn1_Z}(d) and Fig.~\ref{figure:syn2_Z}(d),
%which thus enhances its performance on multi-scale data.

We further study how all 11 methods perform on multi-scale data by varying the densities and sizes
of some clusters in the synthetic datasets. 
We make two changes:
(1) increase the density of a cluster while keeping its size unchanged, and
(2) increase the size of a cluster while maintaining its density unchanged.
Here, we show some representative results. 
Specifically, we increase the density of the middle magenta cluster in \textsc{Syn1} and 
use $\Delta d$ to denote the density change
(e.g., $\Delta d$ = 20\% means that the density of the cluster is 1.2 times larger than the original one).
We change the size of the cluster by enlarging the length sideways with the height fixed. 
We use $\Delta s$ to denote the size change
(e.g., $\Delta s$ = 100\% means that the size of the cluster is doubled).
We make similar changes to the half-ring cluster of \textsc{Syn2}.
In particular, we gradually enlarge the size of the cluster from a half ring ($\smile$, $\Delta s$ = 0\%)
to a whole ring ($\bigcirc$, $\Delta s$ = 100\%).
The clustering results are shown in Fig.~\ref{figure:syn_ds}.
From the figure, 
we see that 
\algo\ gives the best and the most stable performances among all the methods over all the test cases.
This shows that a similarity matrix with intra-cluster correlation and inter-cluster sparsity contributes positively to the clustering of multi-scale data.
In contrast,
ROSC-S, which constructs a sparse matrix but
reduces connections between highly correlated objects,
and ROSC, which computes a matrix with grouping effect but amplifies the inter-cluster correlations,
are thus much less robust. 
%For example,
%while it achieves a similar performance as ROSC for \textsc{Syn1}, 
%it performs poorly for \textsc{Syn2}.

\comment{
Experimental results on \emph{AMI}, \emph{purity} and \emph{RI} are respectively 
summarized in Table~\ref{table:ami_synthetic},~\ref{table:purity_synthetic} and~\ref{table:ri_synthetic}.
To better present the results,
we select some methods and show their results graphically. 
%in Fig.~\ref{figure:syn1},~\ref{figure:syn2},~\ref{figure:syn3} and~\ref{figure:syn4}. 

The first dataset consists of 
three uniformly distributed clusters with 500, 80 and 120 objects respectively.
Since the largest rectangular cluster is of large length,
the two-end objects in the cluster are far away from each other,
i.e., they are less similar. 
Further, the closeness between the two small clusters and the large cluster increases more difficulty in clustering.
We observe that 
ROSC achieves $0.4384$ in AMI, $0.8515$ in purity and $0.7121$ in RI,
which outperforms all the comparison methods.
From Fig.~\ref{figure:syn1},
both ZP and FUSE perform poorly in that they separate the large cluster into small subclusters.
In comparison, ROSC correctly identifies the largest cluster.

The second dataset shown in Fig.~\ref{figure:syn2}(a)
is composed of five clusters:
two Gaussian distributed with 100 objects each,
two uniformly distributed with 150 and 200 objects respectively 
and an annular cluster with 100 objects.
The annular cluster is close to the other clusters and it is hard to be identified.
Experimental results show that ZP is the best method on this dataset. 
The clustering result shown in Fig.~\ref{figure:syn2}(c) indicates that 
ZP correctly finds Gaussian distributed and uniformly distributed clusters
with some misclassification in the annular cluster.
We further notice that the standard spectral clustering methods NJW and NCuts 
are also effective because of the high-quality locally scaled similarity matrix.
In comparison, ROSC is not the best method, but it performs well in identifying the annular cluster.
It is also among the best ones in all the measures.
}

%\begin{figure}[!htbp]
%        \centering
%        \includegraphics[width=0.45\textwidth]{figure/syn1_d.eps}
%        \vspace{-0.15cm}
%        \caption{Results vs. varying cluster's density in \textsc{Syn1}}
%        \label{figure:syn1_d}
%  
%         \centering
%    %     \vspace{0.2cm}
%        \includegraphics[width=0.45\textwidth]{figure/syn1_s.eps}
%        \vspace{-0.15cm}
%        \caption{Results vs. varying cluster's size in \textsc{Syn1}}
%        \label{figure:syn1_s}
%\end{figure}
%
%\begin{figure}[!htbp]
%         \centering
%      %   \vspace{0.2cm}
%        \includegraphics[width=0.45\textwidth]{figure/syn2_d1.eps}
%        \vspace{-0.15cm}
%        \caption{Results vs. varying cluster's density in \textsc{Syn2}}
%        \label{figure:syn2_d}
%
%         \centering
%      %  \vspace{0.2cm}
%        \includegraphics[width=0.45\textwidth]{figure/syn2_s1.eps}
%        \vspace{-0.15cm}
%    \caption{Results vs. varying cluster's size in \textsc{Syn2}}
%     \label{figure:syn2_s}
%     %\vspace{-1mm}
%\end{figure}

\begin{figure}[!htbp]
    \centering
   % \subfigure[Varying the middle cluster's density in \synone]{\includegraphics[width=\columnwidth]{figure/syn1_d.eps}} \\[-0.5mm]
    %\subfigure[Varying the middle cluster's size in \synone]{\includegraphics[width=\columnwidth]{figure/syn1_s.eps}} \\[-0.5mm]
    %\subfigure[Varying the half-ring cluster's density in \syntwo]{\includegraphics[width=\columnwidth]{figure/syn2_d1.eps}} \\[-0.5mm]
    %\subfigure[Varying the half-ring cluster's size in \syntwo]{\includegraphics[width=\columnwidth]{figure/syn2_s1.eps}}
    \subfloat[Varying the middle cluster's density in \synone]{\includegraphics[width=\columnwidth]{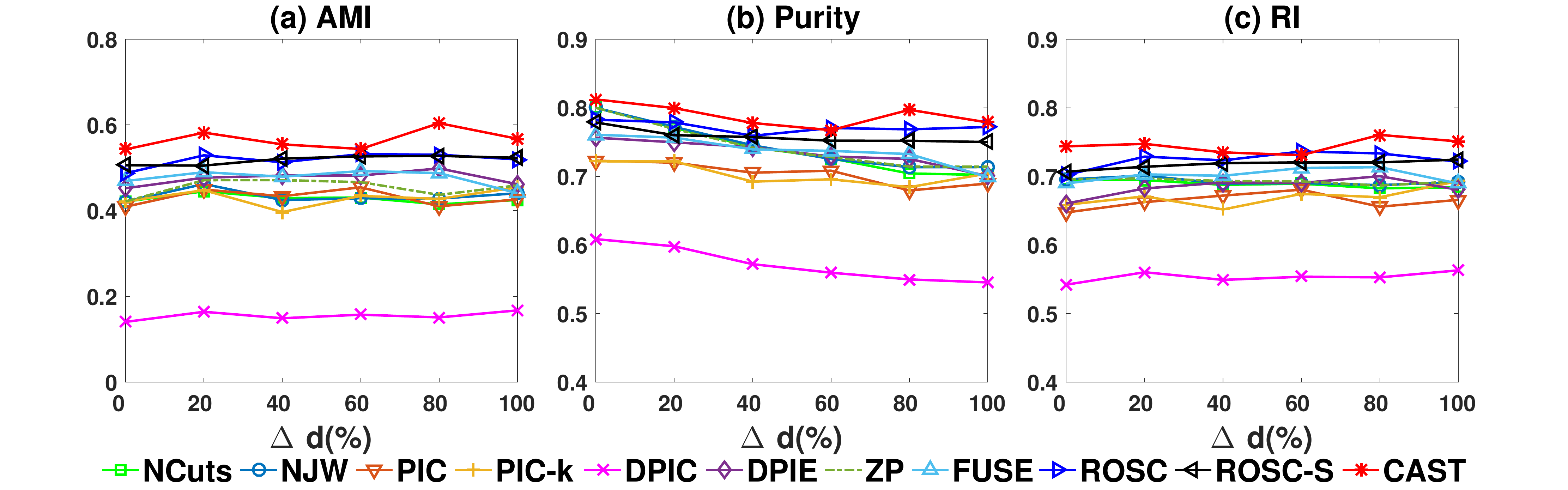}} \\[-0.5mm]
    \subfloat[Varying the middle cluster's size in \synone]{\includegraphics[width=\columnwidth]{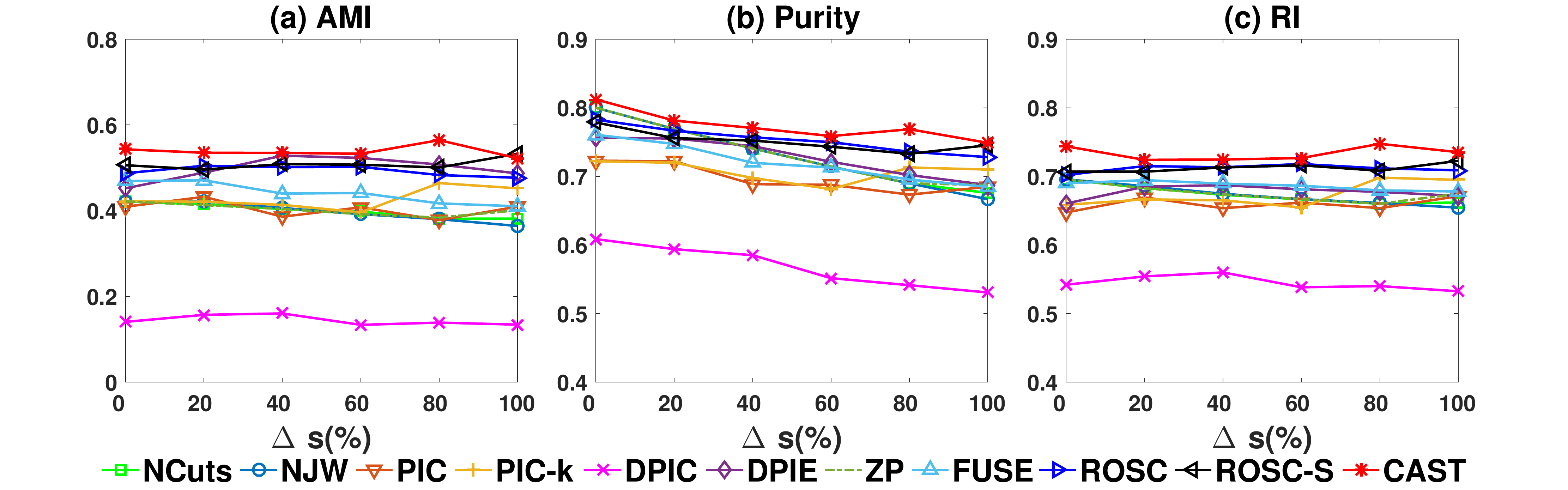}} \\[-0.5mm]
    \subfloat[Varying the half-ring cluster's density in \syntwo]{\includegraphics[width=\columnwidth]{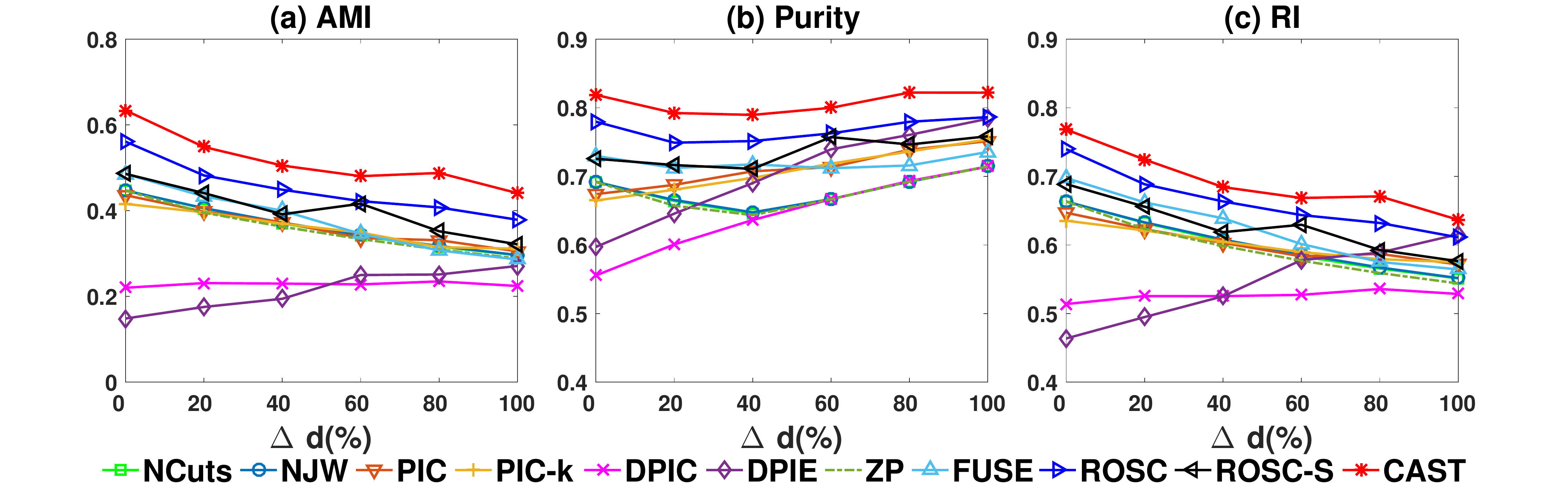}} \\[-0.5mm]
    \subfloat[Varying the half-ring cluster's size in \syntwo]{\includegraphics[width=\columnwidth]{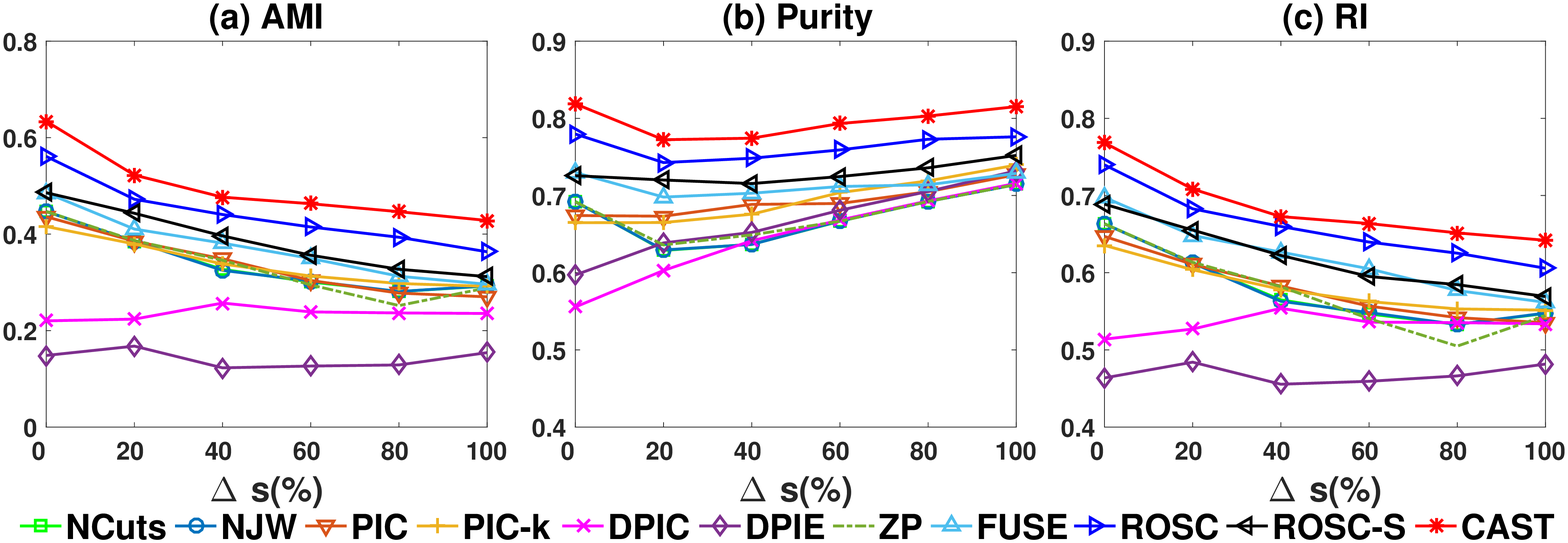}}
    \vspace{-0.3cm} 
     \caption{Results vs. varying the clusters in \synone\ and \syntwo}
     \label{figure:syn_ds}
\end{figure}

\begin{table*}[!htbp]
%\end{table*}

%\begin{table*}[!htbp]
\centering
\resizebox{0.85\linewidth}{!}
{
\begin{tabular}{|c||c|c||c|c|c|c||c|c||c|c|c|} \hline
Dataset &NJW & NCuts  & PIC & PIC-$k$ & DPIC & DPIE & ZP & FUSE & ROSC-S & ROSC & CAST \\ \hline 
%20ngD & $0.4570$ & $0.4750$  & $0.4891$ & $0.4858$ & $0.3406$ & $0.3074$ & $0.5075$ & $0.4672$ & $0.5029$ & $\bm{0.5076\ (1)}$ \\ \hline
COIL20 &$0.4115$ & $0.3926$  & $0.2801$ & $0.2801$ & $0.2361$ & $0.3496$ & $0.5028$ & $0.4177$ & $\bm{0.9500}$ & $\bm{0.9500}$ & $\bm{0.9500\ (1)}$ \\ \hline
%seg\_7\textsc{class} &$0.5608$ & $0.5403$  & $0.3483$ & $0.3566$ & $0.3000$ & $0.4756$ & $0.5143$ & $0.5912$ & $\bm{0.6904}$ & $0.6613$ & $0.6724\ (2)$\\ \hline
%ecoli &$0.5338$ & $0.5365$ & $0.5339$ & $0.4366$ & $0.4814$ & $0.2073$ & $0.4550$ & $0.4325$ & $0.4770$ \\ \hline
glass &$0.5234$ & $0.5187$  & $0.4976$ & $0.5029$ & $0.5245$ & $0.5158$ & $0.5374$ & $0.5390$ & $0.5497$ & $\bm{0.5822}$ & $0.5785\ (2)$ \\ \hline
MNIST0127 & $0.5066$ & $0.4970$  & $0.4975$ & $0.4924$ & $0.5898$ & $0.4395$ & $0.5066$ & $0.6436$ & $0.6971$ & $0.6776$ & $\bm{0.7146\ (1)}$\\ \hline
isolet\_5\textsc{class}&$0.8120$ & $0.7967$  & $0.5863$ & $0.5867$ & $0.3033$ & $0.8572$ & $0.7767$ & $0.7825$ & $\bm{0.8860}$ & $0.8253$ & $0.8671\ (2)$\\ \hline
%statlog &$0.6571$ & $0.6590$ & $\bm{0.6882}$ & $0.5465$ & $0.6052$ & $0.4322$ & $0.3573$ & $0.6193$ & $0.6202\ (4)$ \\ \hline
%wine &$0.8614$ & $0.8633$ & $0.8627$ & $0.8218$ & $0.8229$ & $0.5159$ & $0.5083$ & $0.6919$ & $0.6968$\\ \hline
%Yeast\_4\textsc{class}&$0.4819$ & $0.4665$  & $0.4428$ & $0.4557$ & $0.3831$ & $0.4671$ & $0.4819$ & $\bm{0.4999}$ & $0.4893$ & $0.4920$ & $0.4941\ (2)$\\ \hline
Yale\_5\textsc{class}&$0.5273$ & $0.5091$  & $0.4516$ & $0.4596$ & $0.4000$ & $0.5225$ & $0.5091$ & $0.5458$ & $0.5422$ & $0.5693$ & $\bm{0.5753\ (1)}$\\ \hline
%ecoli &$0.8365$ & $0.8319$  & $0.7251$ & $0.7481$ & $0.4274$ & $0.7506$ & $\bm{0.8393}$ & $0.8117$ & $0.8298\ (4)$ & $0.8026$ & $0.8244\ (5)$\\ \hline
\end{tabular}
}
\caption{Purity scores, real datasets}
\label{table:purity_real}

%\end{table*}

\centering
\resizebox{0.85\linewidth}{!}
{
\begin{tabular}{|c||c|c||c|c|c|c||c|c||c|c|c|} \hline
Dataset &NJW & NCuts & PIC & PIC-$k$ & DPIC & DPIE & ZP & FUSE & ROSC-S & ROSC & CAST\\ \hline 
%20ngD & $0.1913$ & $0.2094$ & $0.2363$ & $0.2558$ & $0.0305$ & $0.0445$ & $\bm{0.2618}$ & $0.2039$ & $0.2320$ & $0.2382\ (3)$ \\ \hline
COIL20 &$0.4718$ & $0.4258$  & $0.2989$ & $0.2781$ & $0.2507$ & $0.3642$ & $0.5702$ & $0.4448$ & $\bm{0.9758}$ & $\bm{0.9758}$ & $\bm{0.9758\ (1)}$\\ \hline
%seg\_7\textsc{class} &$0.5043$ & $0.4603$ & $0.2339$ & $0.2385$ & $0.0915$ & $0.3954$ & $0.4298$  & $0.5049$ & $\bm{0.5971}$ & $0.5608$ & $0.5738\ (2)$\\ \hline
%ecoli &$0.5338$ & $0.5365$ & $0.5339$ & $0.4366$ & $0.4814$ & $0.2073$ & $0.4325$ & $0.4550$ & $0.4770$ \\ \hline
glass &$\bm{0.3469}$ & $0.3465$  & $0.3162$ & $0.3193$ & $0.2807$ & $0.2683$ & $0.3426$ & $0.2589$ & $0.3245$ & $0.2988$ & $0.3390\ (4)$ \\ \hline
MNIST0127 & $0.4353$ & $0.4241$  & $0.3623$ & $0.3822$ & $0.3714$ & $0.2059$& $0.4219$ & $0.4125$ & $0.5243$ & $0.4731$ & $\bm{0.5311\ (1)}$ \\ \hline
isolet\_\textsc{5class}&$0.7595$ & $0.7204$  & $0.5280$ & $0.5292$ & $0.0489$ & $0.7481$ & $0.7379$ & $0.6516$ & $\bm{0.8038}$ & $0.7518$ & $0.7662\ (2)$ \\ \hline
%statlog &$0.6571$ & $0.6590$ & $\bm{0.6882}$ & $0.5465$ & $0.6052$ & $0.4322$ & $0.3573$ & $0.6193$ & $0.6202\ (4)$ \\ \hline
%wine &$0.8614$ & $0.8633$ & $0.8627$ & $0.8218$ & $0.8229$ & $0.5159$ & $0.5083$ & $0.6919$ & $0.6968$\\ \hline
%Yeast\_4\textsc{class}&$0.1173$ & $0.1052$  & $0.1081$ & $0.1165$ & $0.0214$ & $0.1318$ & $0.1138$ & $\bm{0.1816}$ & $0.1478$ & $0.1532$ & $0.1586\ (2)$ \\ \hline
Yale\_5\textsc{class}&$0.3121$ & $0.3321$  & $0.2357$ & $0.2320$ & $0.1468$ & $0.3305$ & $0.2788$ & $0.3465$ & $0.3218$ & $0.3475$ & $\bm{0.3477\ (1)}$ \\ \hline
%ecoli &$0.5243$ & $0.5202$  & $0.4443$ & $0.4674$ & $0.0366$ & $0.4237$ & $\bm{0.5339}$ & $0.4665$ & $0.5275\ (2)$ & $0.4702$ & $0.5161\ (5)$\\ \hline
\end{tabular}
}
\caption{AMI scores, real datasets}
\label{table:ami_real}

%\begin{table*}[!htbp]
\centering
\resizebox{0.85\linewidth}{!}
{
\begin{tabular}{|c||c|c||c|c|c|c||c|c||c|c|c|} \hline
Dataset &NJW & NCuts  & PIC & PIC-$k$ & DPIC & DPIE & ZP & FUSE & ROSC-S & ROSC& CAST \\ \hline 
%20ngD & $0.5902$ & $0.6312$  & $0.6221$ & $0.6153$ & $0.5896$ & $0.3334$ & $0.6569$ & $0.6148$ & $0.6588$ & $\bm{0.6609\ (1)}$ \\ \hline
COIL20 &$0.7303$ & $0.6245$  & $0.4940$ & $0.4481$ & $0.7737$ & $0.6114$ & $0.8534$ & $0.7424$ & $\bm{0.9938}$ & $\bm{0.9938}$ & $\bm{0.9938\ (1)}$ \\ \hline
%seg\_7\textsc{class} &$0.8242$ & $0.7962$  & $0.4830$ & $0.5000$ & $0.7212$ & $0.7162$ & $0.8208$ & $0.8210$ & $\bm{0.8713}$ & $0.8573$ & $0.8583\ (2)$\\ \hline
%ecoli &$0.5338$ & $0.5365$ & $0.5339$ & $0.4366$ & $0.4814$ & $0.2073$ & $0.4325$ & $0.4550$ & $0.4770$ \\ \hline
glass &$0.6890$ & $0.6880$  & $0.6808$ & $0.6851$ & $0.6556$ & $0.6281$ & $0.6949$ & $0.6693$ & $0.6992$ & $\bm{0.7117}$ & $0.7022\ (2)$\\ \hline
MNIST0127 & $0.5683$ & $0.5459$  & $0.5941$ & $0.5887$ & $0.6598$ & $0.4648$ & $0.6018$ & $0.7022$ & $0.7693$ & $0.7533$ & $\bm{0.7867\ (1)}$\\ \hline
isolet\_5\textsc{class}&$0.9058$ & $0.8942$  & $0.7288$ & $0.7296$ & $0.6792$ & $0.9123$ & $0.8993$ & $0.8695$ & $\bm{0.9293}$ & $0.9026$ & $0.9132\ (2)$\\ \hline
%statlog &$0.6571$ & $0.6590$ & $\bm{0.6882}$ & $0.5465$ & $0.6052$ & $0.4322$ & $0.3573$ & $0.6193$ & $0.6202\ (4)$ \\ \hline
%wine &$0.8614$ & $0.8633$ & $0.8627$ & $0.8218$ & $0.8229$ & $0.5159$ & $0.5083$ & $0.6919$ & $0.6968$\\ \hline
%Yeast\_4\textsc{class}&$0.6046$ & $0.5929$  & $0.5643$ & $0.5733$ & $0.5770$ & $0.5037$ & $0.6201$ & $0.6346$ & $$ & $0.6443$ & $\bm{0.6468\ (1)}$\\ \hline0.6423
Yale\_5\textsc{class} &$0.7626$ & $0.7519$  & $0.6772$ & $0.6843$ & $0.6846$ & $0.7542$ & $0.7600$ & $0.7363$ & $0.7721$ & $0.7817$ & $\bm{0.7833\ (1)}$\\ \hline
%ecoli &$\bm{0.8070}$ & $0.8044$  & $0.7919$ & $0.8010$ & $0.5169$ & $0.7753$ & $0.8069$ & $0.7847$ & $0.8051\ (3)$ & $0.7894$ & $0.8047\ (4)$\\ \hline
\end{tabular}
}
\caption{Rand index scores, real datasets}
\vspace{-5mm}
\label{table:ri_real}
\end{table*}

\begin{figure}[!htbp]
    \centering
    %\subfigure[\isolet\_$\alpha_1$]{\includegraphics[width=0.4\columnwidth]{figure/isolet_alpha1.eps}} 
    %\subfigure[\yale\_$\alpha_1$]{\includegraphics[width=0.4\columnwidth]{figure/yale_alpha1.eps}} \\[-2ex]
    %\subfigure[\isolet\_$\alpha_2$]{\includegraphics[width=0.4\columnwidth]{figure/isolet_alpha2.eps}} 
    %\subfigure[\yale\_$\alpha_2$]{\includegraphics[width=0.4\columnwidth]{figure/yale_alpha2.eps}}
    \subfloat[\isolet\_$\alpha_1$]{\includegraphics[width=0.4\columnwidth]{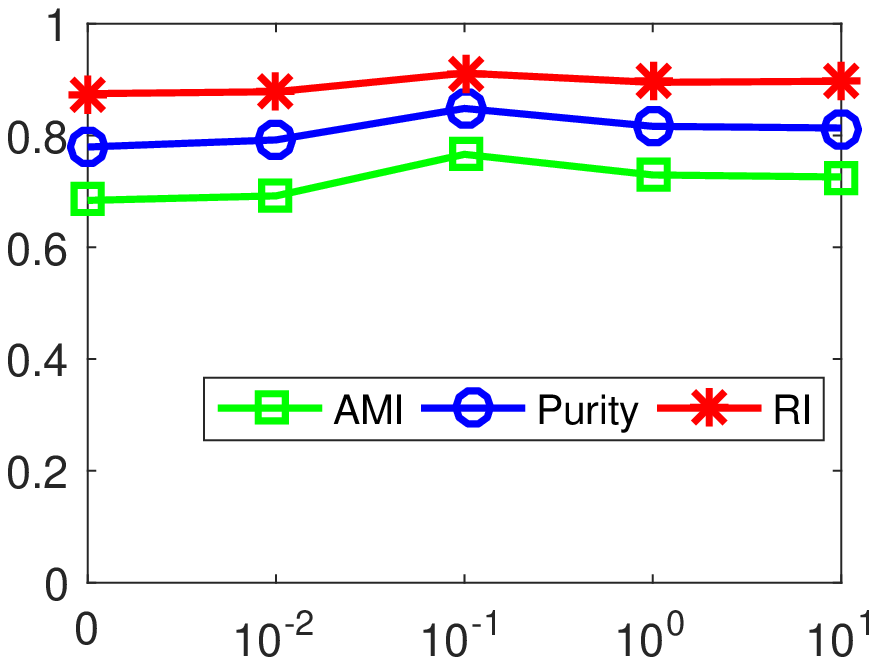}} 
    \subfloat[\yale\_$\alpha_1$]{\includegraphics[width=0.4\columnwidth]{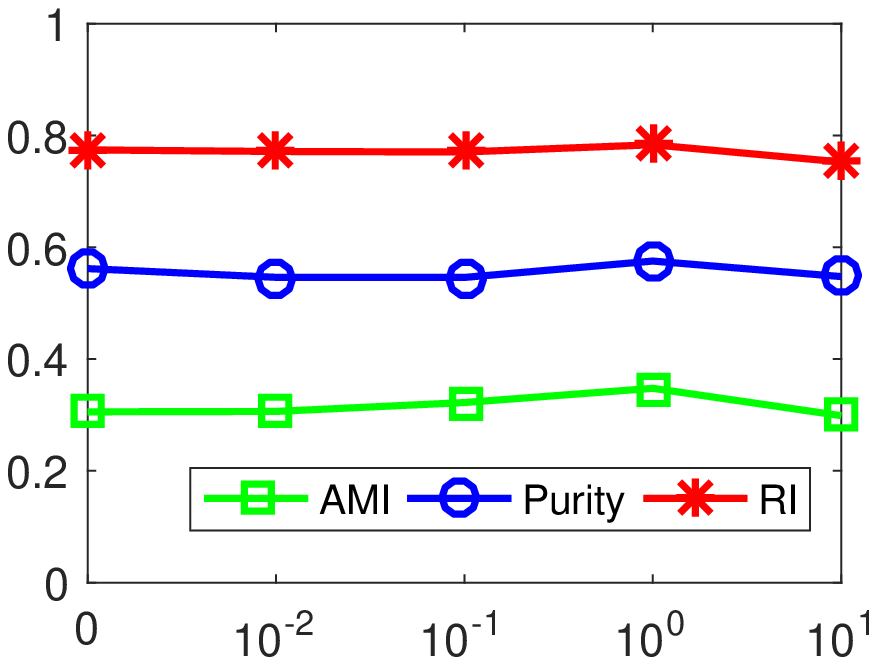}} \\[-2ex]
    \subfloat[\isolet\_$\alpha_2$]{\includegraphics[width=0.4\columnwidth]{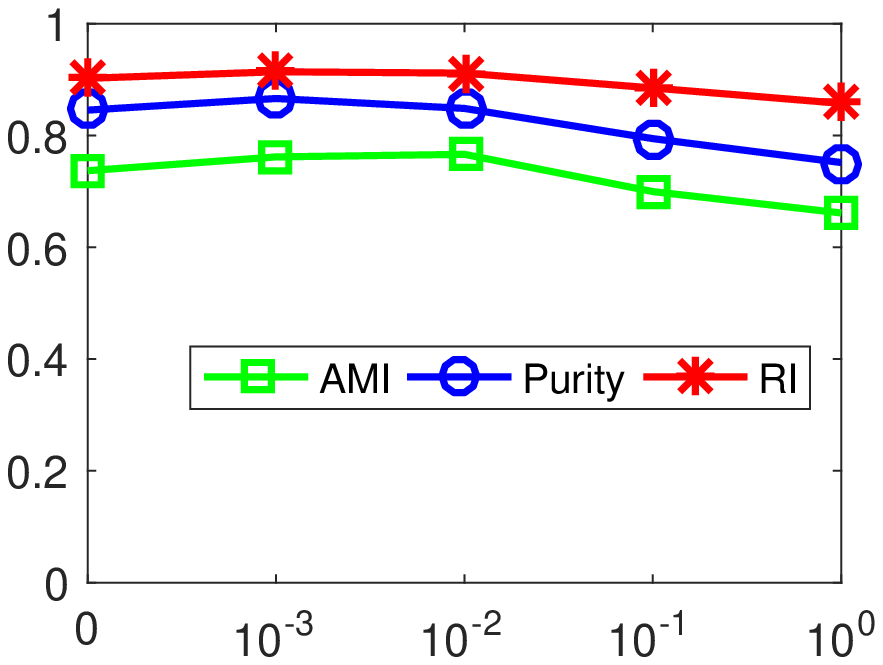}} 
    \subfloat[\yale\_$\alpha_2$]{\includegraphics[width=0.4\columnwidth]{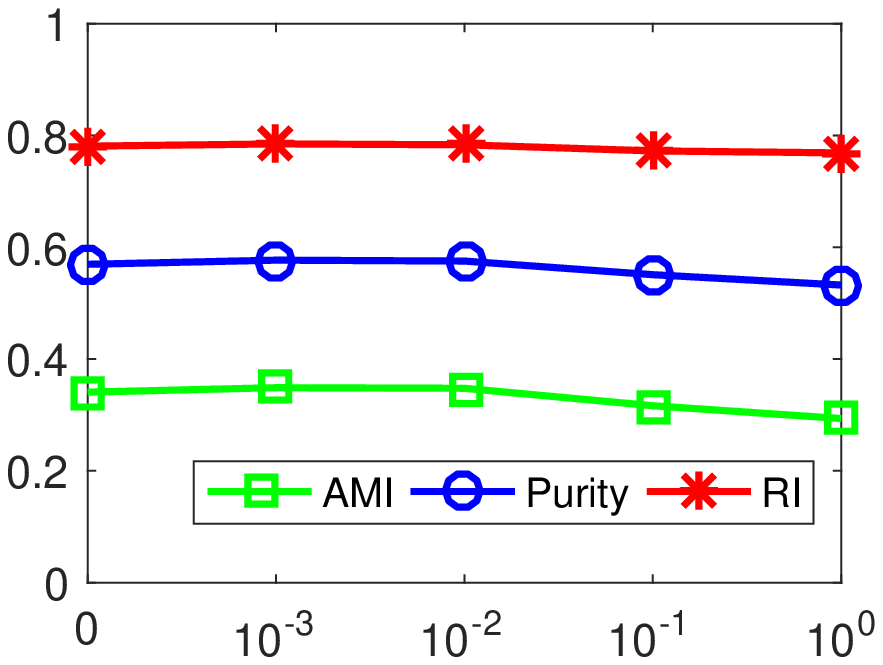}}
    \vspace{-0.3cm} 
     \caption{CAST's performance scores vs. $\alpha_1$ and $\alpha_2$}
     \label{figure:para}
\end{figure}

\noindent{\bf[Real datasets]}
We further compare the methods using 5 real datasets.
%For a fair comparison, we use the same datasets as reported in~\cite{li2018rosc}, which include clusters of various densities and sizes.
%~\footnote{The datasets are publicly available at.}
They are:
%\ngd\ (text documents from 20Newsgroup),
\coil\ (images),
\glass\ (UCI repository),
\mnist\ (hand-written digit images),
\isolet\ (speech, UCI repository),
%\seg\ (images, UCI repository),
%\yeast\ (biological data, UCI repository),
and \yale\ (facial images).
Some statistics of these datasets are given in the Appendix.

%\begin{table}[H]
%\centering
%\resizebox{0.75\linewidth}{!}
%{
%%\begin{tabular}{|c|c|c|c|c|c|} \hline
%\begin{tabular}{|c|r|r|r|} \hline
%Dataset & \#objects & \#dimensions & \#clusters \\ \hline 
%%20ngD & $800$ & $26214$ & $4$ & - & - \\ \hline
%COIL20 &$1,440$ & $1,024$ & $20$  \\ \hline
%%seg\_7\textsc{class} &$210$ & $19$ & $7$ \\ \hline
%glass &$214$ & $9$ & $6$  \\ \hline
%MNIST0127 & $1,666$ & $784$ & $4$ \\ \hline
%isolet\_5\textsc{class}&$300$ & $617$ & $5$ \\ \hline
%%Yeast\_4\textsc{class}&$1299$ & $8$ & $4$ \\ \hline
%Yale\_5\textsc{class}&$55$ & $1024$ & $5$ \\ \hline
%\end{tabular}
%}
%\caption{Statistics of 5 real datasets}
%\vspace{-6mm}
%\label{table:des_real}
%\end{table}

\comment{
For each dataset, we show the number of objects, the number of dimensions, and the number of clusters.
Also, to show whether a dataset is multi-scale or not, we measure the {\it size} and {\it density} of each
gold-standard cluster in each dataset. 
Specifically, for each cluster, we find the largest distance of any object-pair of the cluster. 
This distance is taken as the {\it diameter} of cluster, reflecting how big in size the cluster is.
Moreover, for each cluster, we find the average distance, $\rho$, of all object-pairs of the cluster. 
Then, we use $1/\rho$ as a measure of density. These cluster sizes and densities are shown in
bar graphs in Table~\ref{table:des_real}.
%\footnote{We do not show the sizes and densities of the clusters
%for the {\it 20ngD} dataset because the objects are text documents with extremely sparse and high dimensions.}.
The sizes (densities) shown are all normalized by the size (density) of the biggest (densest) cluster
of the corresponding dataset to the range [0, 1].
Intuitively, the more variations in the bars of a graph indicate the more multi-scale the corresponding
dataset is.
}

Tables~\ref{table:purity_real},~\ref{table:ami_real} and~\ref{table:ri_real} show the performance results.
%of the methods on real datasets w.r.t. purity, AMI and RI, respectively.
%For standard spectral clustering methods (NJW and NCuts), 
%PI-based methods (PIC, PIC-$k$, DPIC and DPIE)
%and multi-scale-oriented-methods (ZP and FUSE),
%we directly use the results given in~\cite{li2018rosc}.
Since we evaluate the methods on 5 datasets w.r.t. 3 measures,
there are in total 15 ``contests''. 
Each row in the tables corresponds to one contest.
We highlight the winner's score of each contest in bold.
We also give the 
ranking of \algo\ in each contest next to its score.
From the tables,
we make the following observations:

%\noindent{\small$\bullet$}
%The standard spectral clustering methods and 
%PI based methods win only 1 of 21 contests.
%Although
%multi-scale-data-oriented methods can
%achieve good performance in some cases,
%e.g., 
%the score 0.1816 of FUSE in the (AMI-\yeast) contest outscores all other methods,
%they are not very robust across all the 7 datasets.  

\noindent{\small$\bullet$}
Matrix-reconstruction methods (ROSC, ROSC-S and \algo) win in all but one contest (AMI-\glass).
In this case, 
%their performances are close to the winner's (NJW).
%For example, 
%for the contest Purity-\yeast, \algo\ ranks 2nd with score 0.4941, which is very close to that of the winner (FUSE, 0.4999).
%As another example,
%\algo\ ranks 2nd. 
\algo's score (0.3390) is very close to that of the winner NJW (0.3469).
This shows that matrix-reconstruction methods
are superior in dealing with multi-scale data.
Compared with other competitors, these methods
derive new matrices that can more effectively capture object correlations,
which explains their excellent performance.
For example, all the three methods significantly outperform the rests for \coil.

\noindent{\small$\bullet$}
ROSC and ROSC-S each win in 5 and 6 contests, respectively. 
We also observe that there are quite a few cases in which their performances differ significantly. 
%ROSC strengthens correlations between distant objects in big and sparse clusters, while
%ROSC-S weakens connections between physically close objects in adjacent clusters. 
%However,
%in some cases, there are big gaps between their performance scores.
For example,
ROSC-S beats ROSC 0.8038 to 0.7518 in AMI-\isolet,
while 
ROSC outperforms ROSC-S 0.5822 to 0.5497 in Purity-\glass.
This is because ROSC lacks sparsity for inter-cluster connections
while ROSC-S loses grouping effect for highly correlated objects.
The results thus show that the relative performance of ROSC-S and ROSC varies across datasets.
They are thus relatively unstable in their performance. 

\noindent{\small$\bullet$}
\algo\ provides a more stable performance across the datasets compared with ROSC and ROSC-S.
First, \algo\ wins in 9 contests and ranks 2nd in 5 others.
For the case that \algo\ is not top-2 (e.g., AMI-\glass), it is
the best algorithm among the matrix-reconstruction methods.
%or is the second among the three with a score that is close to the top one (e.g., in RI-\isolet, \algo's score is 0.9115, while ROSC-S's score is 0.9293).
%
%CASC is the top two 18 times, 7 of which are the winner and 11 of which are the first runner-up.
%For the cases it loses,
%the performance gap with the winner is very small. 
%For example, the score of CASC in (AMI-\glass) is 0.3390
%while the best one is 0.3469.
%We also see that 
%in the cases that ROSC or ROSC-S wins,
%the score of CASC is in between that of ROSC and ROSC-S.
%For example,
%in the (purity-\glass) contest, CASC gets a score of 0.5782 which is larger than that of ROSC-S 0.5523
%but slightly smaller than the score of ROSC 0.5822; 
%for (AMI-\seg), the score of CASC is 0.5738 that falls in between 0.5971 of ROSC-S and 0.5608 of ROSC.
With regularization using the trace Lasso, \algo\ takes advantage of both grouping effect and sparsity.
It is thus more robust when applied to multi-scale data of different characteristics. 

We end this section with a parameter analysis.
\algo\ uses two parameters $\alpha_1$ and $\alpha_2$ to control the 
trace Lasso regularization term
and the TKNN graph regularization term, respectively.
We fix one parameter and vary the other.
Fig.~\ref{figure:para} shows the parameter analysis on the datasets \isolet\ and \yale.
From the figure, 
we see that \algo\ gives very stable performance over a wide range of parameter values.

\section{Conclusions}
\label{sec:conclusion}
In this paper, 
we studied the performance of spectral clustering on data with various sizes and densities.
We reviewed existing spectral methods in handling multi-scale data.
In particular,
we observed that the ROSC algorithm
constructs a matrix with grouping effect,
but it fails to weaken connections between clusters.
We thus proposed the \algo\ algorithm,
which 
uses trace Lasso to balance the effect of $\ell_1$ and $\ell_2$ regularizations.
We mathematically proved that the matrix $\check{Z}$ constructed by \algo\ has grouping effect.
We also show that the matrix achieves sparsity for uncorrelated objects. 
%that has both grouping effect and sparsity.
%We mathematically proved that the matrix $\check{Z}$ has the grouping effect.
We conducted extensive experiments to evaluate \algo's performance and compared \algo\
against other competitors using both synthetic and real datasets.
Our experimental results showed that \algo\ performed very well against its competitors over all the datasets. 
It is thus robust when applied to multi-scale data of different properties.

%\end{document}  % This is where a 'short' article might terminate

% ensure same length columns on last page (might need two sub-sequent latex runs)

%ACKNOWLEDGMENTS are optional

\section{Acknowledgments}
This research is supported by Hong Kong Research Grants Council GRF HKU 17254016.

%\clearpage

\bibliographystyle{ACM-Reference-Format}

%\bibliographystyle{abbrv}
%\balance

\bibliography{sc}  % vldb_sample.bib is the name of the Bibliography in this case

%APPENDIX is optional.
% ****************** APPENDIX **************************************
% Example of an appendix; typically would start on a new page
%\pagebreak
%\clearpage

\appendix
%\subsection{Datasets}
%The details of real datasets are given in Table~\ref{table:des_real}.
%\begin{table}[!htbp]
%\centering
%\resizebox{0.8\linewidth}{!}
%{
%\begin{tabular}{|c|c|c|c|c|c|} \hline
%Dataset & \#objects & \#dimensions & \#clusters \\ \hline 
%%20ngD & $800$ & $26214$ & $4$ & - & - \\ \hline
%COIL20 &$1440$ & $1024$ & $20$  \\ \hline
%%seg\_7\textsc{class} &$210$ & $19$ & $7$ \\ \hline
%glass &$214$ & $9$ & $6$  \\ \hline
%MNIST0127 & $1666$ & $784$ & $4$ \\ \hline
%isolet\_5\textsc{class}&$300$ & $617$ & $5$ \\ \hline
%%Yeast\_4\textsc{class}&$1299$ & $8$ & $4$ \\ \hline
%Yale\_5\textsc{class}&$55$ & $1024$ & $5$ \\ \hline
%\end{tabular}
%}
%\caption{Statistics of 5 real datasets}
%\label{table:des_real}
%\end{table}
\newpage

\section{Experiment}
\label{sec:data}

\subsection{Dataset statistics}
We summarize the statistics of datasets used in our experiments in Table~\ref{table:des_real}.
\begin{table}[H]
\centering
\resizebox{0.75\linewidth}{!}
{
%\begin{tabular}{|c|c|c|c|c|c|} \hline
\begin{tabular}{|c|r|r|r|} \hline
Dataset & \#objects & \#dimensions & \#clusters \\ \hline 
%20ngD & $800$ & $26214$ & $4$ & - & - \\ \hline
COIL20 &$1,440$ & $1,024$ & $20$  \\ \hline
%seg\_7\textsc{class} &$210$ & $19$ & $7$ \\ \hline
glass &$214$ & $9$ & $6$  \\ \hline
MNIST0127 & $1,666$ & $784$ & $4$ \\ \hline
isolet\_5\textsc{class}&$300$ & $617$ & $5$ \\ \hline
%Yeast\_4\textsc{class}&$1299$ & $8$ & $4$ \\ \hline
Yale\_5\textsc{class}&$55$ & $1024$ & $5$ \\ \hline
\end{tabular}
}
\caption{Statistics of 5 real datasets}
%\vspace{-6mm}
\label{table:des_real}
\end{table}

\subsection{Experiment settings}
%\textbf{[Experiment settings]}
%The parameters of all the methods are set according to their original papers.
%For text data, cosine similarity is used to calculate the similarity matrix $S$.
%For attributed data, 
For all the datasets,
the similarity matrix $S$ is computed based on Euclidean distance of objects' attributes.
$S$ is also locally scaled as is done in ZP.
%which are locally scaled based on ZP's local scaling procedure.
All the methods employ \emph{k}-means as the last step of the clustering pipeline to return clusters.
For this step, 
we run \emph{k}-means $100$ times with random starting centroids
and the most frequent cluster assignment is used~\cite{lin2010power}.
For ROSC, ROSC-S and \algo,
%we generate $k$ pseudo-eigenvectors with random starting vectors as is done in~\cite{thang2013deflation}.
we set $K = 4$ in constructing the TKNN graph as suggested in~\cite{li2018rosc}, and 
fine tune the parameters by grid search
for $\alpha_1,\alpha_2 \in \{0.001,0.01,0.1,1,10\}$ 
%and $\alpha_2 \in \{0.001,0.01,0.1,1\}$
to report the best results.
For other methods,
parameters are set according to their original papers.
For each method and dataset,
we run the experiment 50 times and report average results.
Our codes and datasets are publicly available at
\url{https://github.com/lixiang3776/CAST}.
%https://www.dropbox.com/s/1edmi8z5xmpl3eh/CAST.zip?dl=0.

\section{Algorithms}
\label{sec:codes}
%\subsection{Pseudocodes}
We give the details of the pseudocodes.
Algorithm~\ref{alg_sparse} and~\ref{alg_trace_lasso} introduce solving Eq.~\ref{eq:rosc_l1} and Eq.~\ref{eq:casc} by inexact ALM, respectively.
%Algorithm~\ref{alg_trace_lasso} introduces solving Eq.~\ref{eq:casc} by inexact ALM.
Algorithm~\ref{alg:casc} summarizes \algo.

%\subsection{Datasets}
%The details of real datasets are given in Table~\ref{table:des_real}.

\begin{figure}[!hbtp]
\begin{minipage}{\columnwidth}
\begin{algorithm}[H]
%\begin{small}
\begin{algorithmic}[1]
\Require $X$, $\mathcal{W}$, $k$, $\rho$, $\mu_{\max}$, $\epsilon$
\Ensure $Z$
\State Initialize $J$, $Z$, $Y$, $\mu$
\While{$\|J-Z+\text{Diag}(Z)\|_{\infty} > \epsilon$}
\State Update $J$ by Eq.~\ref{eq:update_J_sparse} with the others fixed
\State Update $Z$ by Eq.~\ref{eq:update_Z_sparse} with the others fixed
\State Update the multiplier $Y = Y + \mu (J-Z+ \text{Diag}(Z))$
\State Update $\mu = \min (\rho \mu, \mu_{\max})$
\EndWhile
\State \Return $Z$
\end{algorithmic}
%\end{small}
\caption{Solving Eq.~\ref{eq:rosc_l1} by inexact ALM}
\label{alg_sparse}
\end{algorithm}
\end{minipage}

\begin{minipage}{\columnwidth}
\begin{algorithm}[H]
%\begin{small}
\begin{algorithmic}[1]
\Require $\bmx$, $X$, $\bmw$, $k$, $\rho$, $\mu_{\max}$, $\epsilon$
\Ensure $\bmz$
\State Initialize $J$, $\bmz$, $\bme$, $\bmh$, $\bm{\lambda}_1$, $\bm{\lambda}_2$, $Y$, $\mu$
\While{$\|\bme-\bmx+X\bmz\|_{\infty} > \epsilon$ or $\|\bmh-\bmz+\bmw\|_{\infty} > \epsilon$ or $\|J-X\text{Diag}(\bmz)\|_{\infty} > \epsilon$}
\State Update $\bmz$ by Eq.~\ref{eq:casc_z} with other variables fixed
\State Update $\bme$ by Eq.~\ref{eq:casc_e} with other variables fixed
\State Update $\bmh$ by Eq.~\ref{eq:casc_h} with other variables fixed
\State Update $J$ by Eq.~\ref{eq:casc_j} with other variables fixed
\State Update the multiplier $\bm{\lambda}_1 = \bm{\lambda}_1 + \mu (\bme-\bmx+X\bmz)$
\State Update the multiplier $\bm{\lambda}_2 = \bm{\lambda}_2 + \mu (\bmh-\bmz+\bmw)$
\State Update the multiplier $Y = Y + \mu (J-X\text{Diag}(\bmz))$
\State Update $\mu = \min (\rho \mu, \mu_{\max})$
\EndWhile
\State \Return $\bmz$
\end{algorithmic}
%\end{small}
\caption{Solving Eq.~\ref{eq:casc} by inexact ALM}
\label{alg_trace_lasso}
\end{algorithm}
\end{minipage}

\begin{minipage}{\columnwidth}
\begin{algorithm}[H]
%\begin{small}
\begin{algorithmic}[1]
\Require $S$, $k$.
\Ensure $\mathcal{C} = \{C_1, ..., C_k\}$
\State Compute the TKNN graph and the weight matrix $\mathcal{W}$
\State Calculate $W = D^{-1}S$, where $D_{ii} = \sum_jS_{ij}$
%\For $j \leftarrow 1$ do
%\State$t = 0$
%\Repeat
%\State $v_j^{t+1} \leftarrow \frac{Wv_j^t}{||Wv_j^t||_1}$
%\State $ \delta^{t+1} \leftarrow |v_j^{t+1} - v_j^t|$
%\State $t$++
%\Until $||\delta_j^t+1 - \delta_j^t||_{max} \leq \epsilon$ or $t\geq T$
%\EndFor
\State Apply PI on $W$ and generate $p$ pseudo-eigenvectors $\{\bm{v}_r\}_{r=1}^p$
\State $X = \{\bm{v}_1^T; \bm{v}_2^T; ...; \bm{v}_p^T\}$; $X$ = whiten($X$)
\State Normalize each column vector $\bm{x}$ of $X$ such that $\bm{x}^T\bm{x} = 1$
%\State Solve Eq.~\ref{eq:casc} for each object and construct a matrix $Z^*$
\For{$i = 1$ to $n$} 
\State Solve Eq.~\ref{eq:casc} for an object $x_i$ by inexact ALM and get $\bmz_i^*$
\EndFor
\State Calculate the coefficient matrix $Z^* = [\bmz_1^*,...,\bmz_n^*]$
\State Construct $\check{Z} = (|Z^*| + |(Z^*)^T|)/2$
\State Run NCuts on $\check{Z}$ to obtain clusters $\mathcal{C} = \{C_r\}_{r=1}^k$
%\State Decode $\{C_r\}_{r=1}^k$ from $\{{\bm z_r}\}_{r=1}^k$
\State \Return $\mathcal{C} = \{C_1, ..., C_k\}$
\end{algorithmic}
%\end{small}
\caption{\algo}
\label{alg:casc}
\end{algorithm}
\end{minipage}
\end{figure}

\section{Proof}
\label{sec:formulation}

\comment{
In this section, we introduce the detailed derivation of the two methods.
We first solve the optimization problem
\begin{equation}
\min ||X-XZ||_F^2 + \alpha_1||Z||_1 + \alpha_2||Z-W||_F^2,
s.t.\ diag(Z) = 0
\end{equation}
It is equivalent to solving the problem:
\begin{equation}
\min ||X-XJ||_F^2 + \alpha_1||Z||_1 + \alpha_2||J-W||_F^2,
s.t.\ J = Z-diag(Z).
\end{equation}
\begin{equation}
\begin{split}
L(J,Z) & = \frac{1}{2}||X-XJ||_F^2 + \alpha_1||Z||_1 + \frac{\alpha_2}{2}||J-W||_F^2 \\
& +tr(Y^T(J-Z+diag(Z)))+\frac{\mu}{2}||J-Z+diag(Z)||_F^2
\end{split}
\end{equation}
To update $J$,
\begin{equation}
\begin{split}
\frac{\partial L}{\partial J} & = -X^T(X-XJ) + \alpha_2(J-W) + Y + \mu(J-Z+diag(Z)) \\
& = -X^TX+X^TXJ + \alpha_2 J-\alpha_2 W + Y +\mu J-\mu Z+\mu diag(Z) \\
& = 0 
\end{split}
\end{equation}
\begin{equation}
J = (X^TX+\alpha_2 I + \mu I)^{-1} (X^TX+\alpha_2 W - Y + \mu Z -\mu diag(Z))
\end{equation}
To update $Z$,
\begin{equation}
\frac{\partial L}{\partial Z} = \alpha_1 - Y + diag(Y) + \mu (-J + Z + diag(J) - diag(Z)) = 0
\end{equation}
\begin{equation}
Z-diag(Z) = \frac{1}{\mu}(Y-diag(Y)-\alpha_1) + J - diag(J)
\end{equation}
\begin{equation}
Z = A - diag(A),\ \ A = \mathcal{T}_{\frac{\alpha_1}{\mu}}(\frac{Y}{\mu} + J)
\end{equation}
where $\mathcal{T}_\eta(\cdot)$ is the shrinkage-thresholding operator
acting on each element of the given matrix, and 
it is defined as
$\mathcal{T}_\eta(v) = (|v| - \eta)_+sgn(v)$.
The operator $(\cdot)_+$ returns the argument if it is non-negative and 0, otherwise.
The operator $sgn(\cdot)$ gives the sign of $v$.

To update $Y$:
\begin{equation}
Y = Y + \mu(J - Z + diag(Z))
\end{equation}
We next solve the optimization problem:
\begin{equation}
\min ||y-\tilde{X}z||_2^2 + \alpha_1 ||\tilde{X}diag(z)||_* + \alpha_2||z-w||_2^2
\end{equation}
We can first transform the problem to:
\begin{equation}
\begin{split}
& \min ||y-\tilde{X}z||_2^2 + \alpha_1 ||\tilde{X}diag(z)||_* + \alpha_2||z-w||_2^2 \\
& s.t.\ e = y-\tilde{X}z,\ J = \tilde{X}diag(z),\ h = z - w \\
\end{split}
\end{equation}
\begin{equation}
\begin{split}
L(e,J,h,z) & = \frac{1}{2}||e||_2^2 + \alpha_1||J||_* + \frac{\alpha_2}{2}||h||_2^2 + <\lambda_1, e-y+\tilde{X}z>\\
& + <\lambda_2, h-z+w> + tr(Y^T(J-\tilde{X}diag(z))) \\
& + \frac{\mu}{2}(||e-y+\tilde{X}z||_2^2 + ||J-\tilde{X}diag(z)||_F^2 + ||h-z+w||_2^2)
\end{split}
\end{equation}
To update $z$,
\begin{equation}
\begin{split}
\frac{\partial L}{\partial z} & = \tilde{X}^T\lambda_1 - \lambda_2 - diag(Y^T\tilde{x}) + \mu \tilde{X}^T(e-y+\tilde{X}z) \\
& + \mu(h-z+w) + \mu(-diag(J^T\tilde{X}) + Diag(\tilde{X}^T\tilde{X})z) \\
& = 0 
\end{split}
\end{equation}
\begin{equation}
\begin{split}
\mu\tilde{X}^T\tilde{X}z + \mu z + \mu Diag(\tilde{X}^T\tilde{X})z & = -\tilde{X}^T\lambda_1 + \lambda_2 + diag(Y^T\tilde{X}) \\
& -\mu \tilde{X}^T(e-y) + \mu(h+w) + \mu diag(J^T\tilde{X}) \\
\end{split}
\end{equation}
\begin{equation}
\begin{split}
z & = (\tilde{X}^T\tilde{X} + I + Diag(\tilde{X}^T\tilde{X}))^{-1}\cdot \\
& (-\frac{\tilde{X}^T\lambda_1}{\mu} - \tilde{X}^Te + \tilde{X}^Ty + \frac{\lambda_2}{\mu} + h + w + diag((\frac{Y}{\mu} + J)^T\tilde{X}))\\
\end{split}
\end{equation}
To update $e$,
\begin{equation}
\begin{split}
\frac{\partial L}{\partial e} = e + \lambda1 + \mu(e-y+\tilde{X}z) = 0
\end{split}
\end{equation}
\begin{equation}
e = \frac{\mu}{\mu + 1}(-\frac{\lambda_1}{\mu} + y - \tilde{X}z)
\end{equation}
To update $h$,
\begin{equation}
\begin{split}
\frac{\partial L}{\partial h} = \alpha_2 h + \lambda_2 + \mu(h-z+w) = 0
\end{split}
\end{equation}
\begin{equation}
h = \frac{\mu}{\alpha_2 + \mu}(-\frac{\lambda_2}{\mu} + z - w)
\end{equation}
To update $J$,
it is equivalent to solving the sub-problem:
\begin{equation}
\min_{J} \alpha_1||J||_* + tr(Y^TJ) + \frac{\mu}{2}||J-\tilde{X}diag(z)||_F^2
\end{equation}
We can further transform it into
\begin{equation}
\min_{J} \frac{\alpha_1}{\mu}||J||_* + \frac{1}{2}||J-\tilde{X}diag(z) + \frac{Y}{\mu}||_F^2,
\end{equation}
which is a convex optimization problem and has a closed-form solution.
Suppose the singular value decomposition (SVD) of the rank-$r$ matrix
$H = \tilde{X}diag(z) - \frac{Y}{\mu}$ is $H = U\Sigma V^*$,
where $\Sigma = diag(\{\sigma_i\}_{1\leq i \leq r})$ and $\sigma_i$ is the $i$-th singular value.
Let $\tau = \frac{\alpha_1}{\mu}$ and $\mathcal{D}_{\tau} (\Sigma) = diag(\{\sigma_i - \tau\}_+)$.
Then $J$ can be solved by the singular value thresholding operator:
\begin{equation}
J = U\mathcal{D}_{\tau} (\Sigma) V^*.
\end{equation}
}

\comment{
\begin{theorem}
\label{theo:group_effect}
Given a data vector $y \in \mathbb{R}^d$,
a set of data objects $X = [x_1, ..., x_n] \in \mathbb{R}^{d\times n}$
and two parameters $\alpha_1, \alpha_2 \geq 0$,
let $z^* = [z_1^*, ..., z_n^*]^T \in \mathbb{R}^n$ be the optimal solution to the problem:
$\min \limits_{z} f(z) = \frac{1}{2} \left\| y-Xz \right\|_2^2 + \alpha_1 \left\| XDiag(z)\right\|^* + \frac{\alpha_2}{2}\left\| z-w\right\|_2^2$. 
If $x_i \rightarrow x_j$, then $z_i^* \rightarrow z_j^*$. 
\end{theorem}

We rearrange $X$ and get $X= [\hat{X}, \tilde{X}]$,
where $\tilde{X} \in \mathbb{R}^{d\times q}$ consists of $q$ columns that are close to each other
and $\hat{X} \in \mathbb{R}^{d \times (n-q)}$ consists of the rest columns of $X$.
Moreover,
$\tilde{X}$ satisfies:
\begin{equation}
\max\{\| \tilde{X} - \bar{x}_0\mathbf{1}^T\|_*, \| \tilde{X} - \bar{x}_0\mathbf{1}^T\|_2\} \leq \epsilon,
\end{equation}
where $\epsilon > 0$, $\mathbf{1} \in \mathbb{R}^q$ is the all one's vector and
$\bar{x}_0 = \tilde{X}\mathbf{1}/q$ is the mean of $\tilde{X}$. 
Accordingly, we denote $z^* = [\hat{z};\tilde{z}]$.
To prove Theorem~\ref{theo:group_effect},
we only need to show that
if $\left \| \tilde{z} - \bar{z} \mathbf{1}\right \|_2$ is not small,
then $f([\hat{z};\tilde{z}]) > f([\hat{z};\bar{z}\mathbf{1}])$,
where $\bar{z} = \mathbf{1}^T\tilde{z}/q$ is the average value in $\tilde{z}$.
}

In this section we prove Theorem~\ref{th:casc_ge}.
We first consider two lemmas.
\begin{lemma}
\label{lemma_nucleartoF}
Given $\bmz \in \mathbb{R}^n$, $X\in \mathbb{R}^{d \times n}$,
$\left\| X\text{Diag}(\bmz)\right\|_* \leq \left\| X\right\|_F \left\| \bmz\right\|_2$.
\end{lemma}
\begin{lemma}
\label{lemma:lambda_mu}
If $\eta_i \geq \mu_i \geq 0$, $i = 1,...,n$, and $C = \sum_{i=1}^n(\eta_i - \mu_i)$,
then $\sum_{i=1}^n\sqrt{\eta_i} \geq \sum_{i=1}^n \sqrt{\mu_i} + \frac{C}{2\sqrt{\max\{\eta_i\}}}$.
\end{lemma}
\comment{
\begin{theorem}
Given a set of objects $X = [\hat{X},\tilde{X}]$ and $w = [\hat{w}, \tilde{w}]$,
where $\max\{\| \tilde{X} - \bar{x}_0\mathbf{1}^T\|_*, \| \tilde{X} - \bar{x}_0\mathbf{1}^T\|_2\} \leq \epsilon$ and $\|\tilde{w} - \bar{w}\mathbf{1}\|_2 \leq \epsilon$,
let $z = [\hat{z};\tilde{z}] $, $\bar{z} = \tilde{z}^T\mathbf{1}/q$
and $\tilde{z}^*$ is the optimal solution to $\min_{\tilde{z}} \| \tilde{z} - \tilde{w} \|_2^2,\ s.t.\ \tilde{z}^T\mathbf{1} = q\bar{z}$.
If $\left\| \tilde{z} - \bar{z}\mathbf{1} \right\|_2 > \delta$, $f([\hat{z};\tilde{z}]) > f([\hat{z};\bar{z}\mathbf{1}])$,
where
{\small{
\begin{equation}
\nonumber
\delta = \sqrt{\frac{((\alpha_1 + \|y - \hat{X}\hat{z} - \tilde{X} (\bar{z} \mathbf{1})\|_2) \|\tilde{z}\|_2 + \alpha_1 |\bar{z}|)\epsilon - \frac{\alpha_2}{2} \sum_{i=1}^q [(\tilde{z}_i^* - \bar{z})(\tilde{z}_i^* + \bar{z} - 2\tilde{w}_i)]}{\frac{\alpha_1\|\bar{x}_0\|_2^2}{2\|[\hat{X}_{\hat{z}}\ \bar{z}\bar{x}_0\mathbf{1}^T]\|_2}}}.
\end{equation}
}}
\end{theorem}
}
For proofs of both lemmas, see~\cite{lu2013correlation}.
Next, we prove Theorem~\ref{th:casc_ge}.
%\begin{proof}
Let $\hat{X}_{\hat{\bmz}} = \hat{X}\text{Diag}(\hat{\bmz})$. We get
\begin{small}
\begin{equation}
\nonumber
%\label{eq:init}
\begin{split}
f([\hat{\bmz};\tilde{\bmz}]) &= \frac{1}{2}\|\bmx - \hat{X}\hat{\bmz} - \tilde{X}\tilde{\bmz}\|_2^2 + \alpha_1\| [\hat{X}_{\hat{\bmz}}\ \tilde{X}\text{Diag}(\tilde{\bmz})]\|_* + \frac{\alpha_2}{2}||[\hat{\bmz};\tilde{\bmz}]-\bmw||_2^2.
%&= \frac{1}{2}\|(\bmx - \hat{X}\hat{\bmz} - \bar{\bmx}_0 \mathbf{1}^T\tilde{\bmz}) + (\bar{\bmx}_0 \mathbf{1}^T\tilde{\bmz} - \tilde{X}\tilde{\bmz})\|_2^2 \\
%&+ \alpha_1\| [\hat{X}_{\hat{\bmz}}\  \bar{\bmx}_0\mathbf{1}^T\text{Diag}(\tilde{\bmz})] + [0\ (\tilde{X} - \bar{\bmx}_0\mathbf{1}^T)\text{Diag}(\tilde{\bmz})]\|_* \\
%&+ \frac{\alpha_2}{2}||\hat{\bmz} - \hat{\bmw}||_2^2 + \frac{\alpha_2}{2}||(\tilde{\bmz} - \tilde{\bmw})||_2^2 \\
\end{split}
\end{equation}
\end{small}

Rewrite $f([\hat{\bmz};\tilde{\bmz}]) = \Omega_1 + \Omega_2 + \Omega_3$,
where $\Omega_1 = \frac{1}{2}\|(\bmx - \hat{X}\hat{\bmz} - \bar{\bmx}_0 \mathbf{1}^T\tilde{\bmz}) + (\bar{\bmx}_0 \mathbf{1}^T\tilde{\bmz} - \tilde{X}\tilde{\bmz})\|_2^2$,
$\Omega_2 = \alpha_1\| [\hat{X}_{\hat{\bmz}}\  \bar{\bmx}_0\mathbf{1}^T\text{Diag}(\tilde{\bmz})] + [0\ (\tilde{X} - \bar{\bmx}_0\mathbf{1}^T)\text{Diag}(\tilde{\bmz})]\|_*$
and $\Omega_3 = \frac{\alpha_2}{2}||\hat{\bmz} - \hat{\bmw}||_2^2 + \frac{\alpha_2}{2}||(\tilde{\bmz} - \tilde{\bmw})||_2^2$.
Since $\bmy^*$ is the optimal solution to the problem: $\min_{\bmy} \|\bmy-\tilde{\bmw}\|_2^2,\ \text{s.t.}\ \bmy^T\mathbf{1} = q\bar{z}$,  we have
\begin{equation}
\nonumber
\begin{split}
\|\tilde{\bmz} - \tilde{\bmw}\|_2^2  & \geq \|\bmy^* - \tilde{\bmw}\|_2^2 - \|\bar{z}\mathbf{1} - \tilde{\bmw}\|_2^2 + \|\bar{z}\mathbf{1} - \tilde{\bmw}\|_2^2\\
%& = \sum_{i=1}^d [(z_i^* - \tilde{w}_i)^2 - (\bar{z} - \tilde{w}_i)^2]\\
% = (\tilde{z}^* - \bar{z}\mathbf{1})^T(\tilde{z}^* + \bar{z}\mathbf{1} - 2\tilde{w}) \\
& =\sum_{j=1}^q [(y_j^* - \bar{z})(y_j^* + \bar{z} - 2\tilde{w}_j)] + \|\bar{z}\mathbf{1} - \tilde{\bmw}\|_2^2.\\
\end{split}
\end{equation}
%\begin{equation}
%\label{eq:optimalz}
%\begin{split}
%\|\tilde{\bmz} - \tilde{\bmw}\|_2^2  - \|\bar{z}\mathbf{1} - \tilde{\bmw}\|_2^2  & \geq \|\bmy^* - \tilde{\bmw}\|_2^2 - \|\bar{z}\mathbf{1} - \tilde{\bmw}\|_2^2 \\
%%& = \sum_{i=1}^d [(z_i^* - \tilde{w}_i)^2 - (\bar{z} - \tilde{w}_i)^2]\\
%% = (\tilde{z}^* - \bar{z}\mathbf{1})^T(\tilde{z}^* + \bar{z}\mathbf{1} - 2\tilde{w}) \\
%& =\sum_{j=1}^q [(y_j^* - \bar{z})(y_j^* + \bar{z} - 2\tilde{w}_j)] \\
%\end{split}
%\end{equation}
Let $\Omega_4 = \sum_{j=1}^q [(y_j^* - \bar{z})(y_j^* + \bar{z} - 2\tilde{w}_j)] $.
Since $\|\bmy^* - \tilde{\bmw}\|_2^2$ is the minimum value,
we have $\Omega_4 \leq 0$ and 
$\|\tilde{\bmz} - \tilde{\bmw}\|_2^2  \geq \|\bar{z}\mathbf{1} - \tilde{\bmw}\|_2^2 + \Omega_4$.
We derive lower bounds for $\Omega_1$, $\Omega_2$ and $\Omega_3$:
%For $\Omega_1$,
\begin{small}
\begin{equation}
\label{eq:omega1}
\begin{split}
\Omega_1 & \geq \frac{1}{2}\|\bmx - \hat{X}\hat{\bmz} - \bar{\bmx}_0 (\mathbf{1}^T\tilde{\bmz})\|_2^2 
- \|\bmx - \hat{X}\hat{\bmz} - \bar{\bmx}_0 (\mathbf{1}^T\tilde{\bmz})\|_2 \|(\bar{\bmx}_0 \mathbf{1}^T - \tilde{X})\tilde{\bmz}\|_2 \\
&\geq \frac{1}{2}\|\bmx - \hat{X}\hat{\bmz} - \bar{\bmx}_0 (\mathbf{1}^T\tilde{\bmz})\|_2^2 - \|\bmx - \hat{X}\hat{\bmz} - \bar{\bmx}_0 (\mathbf{1}^T\tilde{\bmz})\|_2 \|\bar{\bmx}_0 \mathbf{1}^T - \tilde{X}\|_2\|\tilde{\bmz}\|_2 \\
\end{split} 
\end{equation}
\end{small}
Based on Lemma~\ref{lemma_nucleartoF},
\begin{equation}
\label{eq:omega2}
\begin{split}
\Omega_2 & \geq \alpha_1\| [\hat{X}_{\hat{\bmz}}\  \bar{\bmx}_0\mathbf{1}^T\text{Diag}(\tilde{\bmz})]\|_* - \alpha_1 \| (\tilde{X} - \bar{\bmx}_0\mathbf{1}^T)\text{Diag}(\tilde{\bmz})\|_* \\
& \geq \alpha_1\| [\hat{X}_{\hat{\bmz}}\  \bar{\bmx}_0\tilde{\bmz}^T]\|_* - \alpha_1 \|\tilde{\bmz}\|_2\| \tilde{X} - \bar{\bmx}_0\mathbf{1}^T\|_F,  \\ 
\end{split}
\end{equation}
%and 
%For $\Omega_3$,
\begin{equation}
\label{eq:omega3}
\Omega_3  \geq \frac{\alpha_2}{2}||[\hat{\bmz}; \bar{z}\mathbf{1}] - \bmw||_2^2 + \frac{\alpha_2}{2} \Omega_4
\end{equation}
Combining Eqs.~\ref{eq:omega1}-\ref{eq:omega3}, we have,
\begin{equation}
\nonumber
\begin{split}
f([\hat{\bmz};\tilde{\bmz}])  &\geq \frac{1}{2}\|\bmx - \hat{X}\hat{\bmz} - \tilde{X} (\bar{z} \mathbf{1})\|_2^2 - (\alpha_1 + \|\bmx - \hat{X}\hat{\bmz} - \tilde{X} (\bar{z} \mathbf{1})\|_2) \|\tilde{z}\|_2 \epsilon\\
& + \alpha_1\| [\hat{X}_{\hat{\bmz}}\  \bar{\bmx}_0\tilde{\bmz}^T]\|_* +   \frac{\alpha_2}{2}||[\hat{\bmz}; \bar{z}\mathbf{1}] - \bmw||_2^2 + \frac{\alpha_2}{2}  \Omega_4 \\
\end{split}
\end{equation}
%To shorten the above equation,
Let $\Omega_5 = \frac{1}{2}\|\bmx - \hat{X}\hat{\bmz} - \tilde{X} (\bar{z} \mathbf{1})\|_2^2 - (\alpha_1 + \|\bmx - \hat{X}\hat{\bmz} - \tilde{X} (\bar{z} \mathbf{1})\|_2) \|\tilde{z}\|_2 \epsilon$
and $\Omega_6 =  \frac{\alpha_2}{2}||[\hat{\bmz}; \bar{z}\mathbf{1}] - \bmw||_2^2 + \frac{\alpha_2}{2}  \Omega_4$, we have,
\begin{equation}
\label{eq:init_extend}
f([\hat{\bmz};\tilde{\bmz}])  \geq \Omega_5 + \alpha_1\| [\hat{X}_{\hat{\bmz}}\  \bar{\bmx}_0\tilde{\bmz}^T]\|_* +  \Omega_6.
\end{equation}

\comment{
\begin{equation}
\label{eq:init_extend}
\begin{split}
f([\hat{\bmz};\tilde{\bmz}]) &\geq \frac{1}{2}\|\bmx - \hat{X}\hat{\bmz} - \bar{\bmx}_0 (\mathbf{1}^T\tilde{\bmz})\|_2^2 + \frac{1}{2}\|(\bar{\bmx}_0 \mathbf{1}^T - \tilde{X})\tilde{\bmz}\|_2^2 \\
&- \|\bmx - \hat{X}\hat{\bmz} - \bar{\bmx}_0 (\mathbf{1}^T\tilde{\bmz})\|_2 \|(\bar{\bmx}_0 \mathbf{1}^T - \tilde{X})\tilde{\bmz}\|_2\\
&+ \alpha_1\| [\hat{X}_{\hat{\bmz}}\  \bar{\bmx}_0\mathbf{1}^T\text{Diag}(\tilde{\bmz})]\|_* - \alpha_1 \| (\tilde{X} - \bar{\bmx}_0\mathbf{1}^T)\text{Diag}(\tilde{\bmz})\|_*\\
&+  \frac{\alpha_2}{2}||\hat{\bmz} - \hat{\bmw}||_2^2 + \frac{\alpha_2}{2}\| \bar{z}\mathbf{1} - \tilde{\bmw} \|_2^2 + \frac{\alpha_2}{2} \sum_{j=1}^q [(y_j^* - \bar{z})(y_j^* + \bar{z} - 2\tilde{w}_j)] \\
&\geq \frac{1}{2}\|\bmx - \hat{X}\hat{\bmz} - \bar{\bmx}_0 (\mathbf{1}^T\tilde{\bmz})\|_2^2 - \|\bmx - \hat{X}\hat{\bmz} - \bar{\bmx}_0 (\mathbf{1}^T\tilde{\bmz})\|_2 \|\bar{\bmx}_0 \mathbf{1}^T - \tilde{X}\|_2\|\tilde{\bmz}\|_2 \\
&+ \alpha_1\| [\hat{X}_{\hat{\bmz}}\  \bar{\bmx}_0\tilde{\bmz}^T]\|_* - \alpha_1 \|\tilde{\bmz}\|_2\| \tilde{X} - \bar{\bmx}_0\mathbf{1}^T\|_F  \\
& +  \frac{\alpha_2}{2}||[\hat{\bmz}; \bar{z}\mathbf{1}] - \bmw||_2^2 + \frac{\alpha_2}{2} \sum_{j=1}^q [(y_j^* - \bar{z})(y_j^* + \bar{z} - 2\tilde{w}_j)] \\
&\geq \frac{1}{2}\|\bmx - \hat{X}\hat{z} - \tilde{X} (\bar{z} \mathbf{1})\|_2^2 - (\alpha_1 + \|\bmx - \hat{X}\hat{z} - \tilde{X} (\bar{z} \mathbf{1})\|_2) \|\tilde{z}\|_2 \epsilon\\
& + \alpha_1\| [\hat{X}_{\hat{\bmz}}\  \bar{\bmx}_0\tilde{\bmz}^T]\|_* +   \frac{\alpha_2}{2}||[\hat{\bmz}; \bar{z}\mathbf{1}] - \bmw||_2^2 \\
& + \frac{\alpha_2}{2} \sum_{j=1}^q [(y_j^* - \bar{z})(y_j^* + \bar{z} - 2w_j)] \\
\end{split}
\end{equation}
}

%To further extend $\| [\hat{X}_{\hat{\bmz}}\  \bar{\bmx}_0\tilde{\bmz}^T]\|_*$,
Let $Y = \hat{X}_{\hat{\bmz}} \hat{X}_{\hat{\bmz}}^T$ and $\lambda_i(M)$ denote the $i$-th largest eigenvalue of a matrix $M$.
%with the order $\lambda_1 \geq \lambda_2 \geq ... \lambda_n$.
We have,
\begin{scriptsize}
\begin{equation}
\label{eq:lambda}
\begin{split}
\sum_{i=1}^d\lambda_i(Y+\|\tilde{\bmz}\|_2^2\bar{\bmx}_0\bar{\bmx}_0^T) &= tr(Y+\|\tilde{\bmz}\|_2^2\bar{\bmx}_0\bar{\bmx}_0^T) \\
&= tr(Y+\|\bar{z}\mathbf{1}\|_2^2\bar{\bmx}_0\bar{\bmx}_0^T) + tr((\|\tilde{\bmz}\|_2^2 -\|\bar{z}\mathbf{1}\|_2^2)\bar{\bmx}_0\bar{\bmx}_0^T) \\
&= \sum_{i=1}^d\lambda_i(Y+\|\bar{z}\mathbf{1}\|_2^2\bar{\bmx}_0\bar{\bmx}_0^T) + (\|\tilde{\bmz}\|_2^2 -\|\bar{z}\mathbf{1}\|_2^2)\|\bar{\bmx}_0\|_2^2 \\
\end{split}
\end{equation}
\end{scriptsize}
Since $\mathbf{1}^T\tilde{\bmz} = q\bar{z}$,
we get $\|\tilde{\bmz}\|_2^2 \geq \|\bar{z}\mathbf{1}\|_2^2$
and $\lambda_i(Y+\|\tilde{\bmz}\|_2^2\bar{\bmx}_0\bar{\bmx}_0^T) \geq \lambda_i(Y+\|\bar{z}\mathbf{1}\|_2^2\bar{\bmx}_0\bar{\bmx}_0^T) \geq 0 $.
Moreover,
$\| \tilde{\bmz}\|_2^2 - \| \bar{z}\mathbf{1}\|_2^2 = \| \tilde{\bmz} - \bar{z}\mathbf{1}\|_2^2$.
Based on Eq.~\ref{eq:lambda} and Lemma~\ref{lemma:lambda_mu}, we get
\begin{equation}
\label{eq:nuclear1}
\begin{split}
\|[\hat{X}_{\hat{\bmz}}\ \bar{\bmx}_0\tilde{\bmz}^T]\|_* &= \sum_{i=1}^d\sqrt{\lambda_i(Y+\|\tilde{\bmz}\|_2^2\bar{\bmx}_0\bar{\bmx}_0^T)}\\
& \geq \sum_{i=1}^d\sqrt{\lambda_i(Y+\|\bar{z}\mathbf{1}\|_2^2\bar{\bmx}_0\bar{\bmx}_0^T)} + \frac{\|\tilde{\bmz} - \bar{z}\mathbf{1}\|_2^2\|\bar{\bmx}_0\|_2^2}{2\sqrt{\lambda_1(Y+\|\tilde{z}\|_2^2\bar{\bmx}_0\bar{\bmx}_0^T)}}\\
& \geq \|[\hat{X}_{\hat{\bmz}}\ \bar{z}\bar{\bmx}_0\mathbf{1}^T]\|_* + \frac{\|\bar{\bmx}_0\|_2^2}{2\|[\hat{X}_{\hat{\bmz}}\ \bar{\bmx}_0\tilde{\bmz}^T]\|_2}\delta^2\\
\end{split}
\end{equation}
Moreover,
\begin{equation}
\label{eq:nuclear2}
\begin{split}
\|[\hat{X}_{\hat{\bmz}}\ \bar{z}\bar{\bmx}_0\mathbf{1}^T]\|_* &= \|[\hat{X}_{\hat{\bmz}}\ \tilde{X}\text{Diag}(\bar{z}\mathbf{1})] + [0\ (\bar{z}\bar{\bmx}_0\mathbf{1}^T - \tilde{X}\text{Diag}(\bar{z}\mathbf{1}))]\|_* \\
& \geq \|[\hat{X}_{\hat{\bmz}}\ \tilde{X}\text{Diag}(\bar{z}\mathbf{1})]\|_* - |\bar{z}|\|\bar{\bmx}_0 \mathbf{1}^T - \tilde{X}\|_* \\
& \geq \|[\hat{X}_{\hat{\bmz}}\ \tilde{X}\text{Diag}(\bar{z}\mathbf{1})]\|_* - |\bar{z}|\epsilon. \\
\end{split}
\end{equation}
Substituting Eq.~\ref{eq:nuclear1} and~\ref{eq:nuclear2} into Eq.~\ref{eq:init_extend}, we have 
\begin{small}
\begin{equation}
\label{eq:final}
\begin{split}
f([\hat{\bmz};\tilde{\bmz}]) & \geq \Omega_5 + \alpha_1(\|[\hat{X}_{\hat{\bmz}}\ \tilde{X}\text{Diag}(\bar{z}\mathbf{1})]\|_* - |\bar{z}|\epsilon + \frac{\|\bar{\bmx}_0\|_2^2}{2\|[\hat{X}_{\hat{\bmz}}\ \bar{\bmx}_0\tilde{\bmz}^T]\|_2}\delta^2) +  \Omega_6\\
& = f([\hat{\bmz}; \bar{z}\mathbf{1}]) + (\frac{\alpha_1\|\bar{\bmx}_0\|_2^2}{2\|[\hat{X}_{\hat{\bmz}}\ \bar{\bmx}_0\tilde{\bmz}^T]\|_2})\delta^2 + \frac{\alpha_2}{2}\Omega_4 
 - \Omega_7\\
\end{split}
\end{equation}
\end{small}
where $\Omega_7 = ((\alpha_1 + \|\bmx - \hat{X}\hat{\bmz} - \tilde{X} (\bar{z} \mathbf{1})\|_2) \|\tilde{\bmz}\|_2 + \alpha_1 |\bar{z}|)\epsilon$.
\comment{
\begin{equation}
\label{eq:final}
\begin{split}
f([\hat{\bmz};\tilde{\bmz}]) & \geq \frac{1}{2}\|\bmx - \hat{X}\hat{\bmz} - \tilde{X} (\bar{z} \mathbf{1})\|_2^2  - (\alpha_1 + \|\bmx - \hat{X}\hat{\bmz} - \tilde{X} (\bar{z} \mathbf{1})\|_2) \|\tilde{\bmz}\|_2 \epsilon \\
& + \alpha_1(\|[\hat{X}_{\hat{\bmz}}\ \tilde{X}\text{Diag}(\bar{z}\mathbf{1})]\|_* - |\bar{z}|\epsilon + \frac{\|\bar{\bmx}_0\|_2^2}{2\|[\hat{X}_{\hat{\bmz}}\ \bar{\bmx}_0\tilde{\bmz}^T]\|_2}\delta^2) \\
&+  \frac{\alpha_2}{2}||[\hat{\bmz}; \bar{z}\mathbf{1}] - \bmw||_2^2  + \frac{\alpha_2}{2} \sum_{j=1}^q [(y_j^* - \bar{z})(y_j^* + \bar{z} - 2\tilde{w}_j)]\\
& = f([\hat{\bmz}; \bar{z}\mathbf{1}]) + (\frac{\alpha_1\|\bar{\bmx}_0\|_2^2}{2\|[\hat{X}_{\hat{\bmz}}\ \bar{\bmx}_0\tilde{\bmz}^T]\|_2})\delta^2 + \frac{\alpha_2}{2} \sum_{j=1}^q [(y_j^* - \bar{z})(y_j^* + \bar{z} - 2\tilde{w}_j)]\\
& - ((\alpha_1 + \|\bmx - \hat{X}\hat{\bmz} - \tilde{X} (\bar{z} \mathbf{1})\|_2) \|\tilde{\bmz}\|_2 + \alpha_1 |\bar{z}|)\epsilon\\
\end{split}
\end{equation}
}
From Eq.~\ref{eq:final},
we see that
if $\left\| \tilde{\bmz} - \bar{z}\mathbf{1} \right\|_2 > \delta$, $f([\hat{\bmz};\tilde{\bmz}]) > f([\hat{\bmz};\bar{z}\mathbf{1}])$, 
where 
\begin{equation}
\nonumber
\delta = \sqrt{\frac{(2\Omega_7 - \alpha_2 \Omega_4)(\|[\hat{X}_{\hat{\bmz}}\ \bar{\bmx}_0 \tilde{\bmz}^T]\|_2)}{\alpha_1\|\bar{\bmx}_0\|_2^2}}.
\end{equation}

%From Lemma.~\ref{lemma:epsilon1}, we know 
%$\epsilon \leq y_j^* - \bar{z} \leq \epsilon$.
%When $\epsilon \rightarrow 0$,
%we have $\Omega_4 \rightarrow 0$, $\Omega_7 \rightarrow 0$
%and thus $\delta \rightarrow 0$.
%In this case,
%$\tilde{\bmz}$ thus has to be very close to $\bar{z}\bm{1}$.
%The grouping effect of $z^*$ is thus guaranteed.

%\end{proof}

\end{document}